\ifdefined\useorstyle

\documentclass[opre,nonblindrev]{informs3}

\usepackage{./macros/packages-or}
\usepackage{./macros/statistics-macros-or}
\usepackage{./macros/formatting}

\DoubleSpacedXI 

\usepackage{endnotes}
\let\footnote=\endnote

\usepackage{natbib}
 \bibpunct[, ]{(}{)}{,}{a}{}{,}%
 \def\bibsep{\smallskipamount}%
\TheoremsNumberedThrough     
\ECRepeatTheorems

\EquationsNumberedThrough    

\usepackage{hyperref}

\usepackage{pgfplotstable}
\usepackage{graphicx}
\usepackage{subcaption}
\usepackage{float} 
\usepackage{tabularx}
\usepackage{overpic}
\usepackage{tikz}
\usepackage{rotating}
\usepackage{psfrag}
\usepackage{enumitem}
  
\begin{document}



\RUNAUTHOR{Che and Namkoong}
\RUNTITLE{Adaptive Experimentation at Scale}

\TITLE{Adaptive Experimentation at Scale: A Computational Framework for Flexible
  Batches}

\ARTICLEAUTHORS{%

\AUTHOR{Ethan Che}
\AFF{Decision, Risk, and Operations Division, Columbia Business School, New York, NY 10027, \EMAIL{namkoong@gsb.columbia.edu}} 

\AUTHOR{Hongseok Namkoong}
\AFF{Decision, Risk, and Operations Division, Columbia Business School, New York, NY 10027, \EMAIL{namkoong@gsb.columbia.edu}} 

} 

\ABSTRACT{

Standard bandit algorithms that assume continual reallocation of measurement
effort are challenging to implement due to delayed feedback and
infrastructural/organizational difficulties. Motivated by practical instances
involving a handful of reallocation opportunities in which outcomes are
measured in batches, we develop a computation-driven adaptive experimentation
framework that can flexibly handle batching.  Our main observation is that
normal approximations, which are universal in statistical inference, can also
guide the design of adaptive algorithms. By deriving a Gaussian sequential
experiment, we formulate a dynamic program that can leverage prior information
on \emph{average rewards}.  Instead of the typical theory-driven paradigm, we
leverage computational tools and empirical benchmarking for algorithm
development. Our empirical analysis highlights a simple yet effective
algorithm, $\algofull$, which iteratively solves a planning problem using
stochastic gradient descent.  Our approach significantly improves power over
standard methods, even when compared to Bayesian algorithms (e.g., Thompson
sampling) that require full distributional knowledge of \emph{individual
rewards}.  Overall, we expand the scope of adaptive experimentation to
settings standard methods struggle with, involving limited adaptivity, low
signal-to-noise ratio, and unknown reward distributions.

}



\maketitle

%


\else

\documentclass[11pt]{article}
\usepackage[numbers]{natbib}
\usepackage{./macros/packages}
\usepackage{./macros/formatting}
\usepackage{./macros/statistics-macros}
\usepackage{multirow}


\begin{document}

\abovedisplayskip=8pt plus0pt minus3pt
\belowdisplayskip=8pt plus0pt minus3pt


\begin{center}
  {\huge Adaptive Experimentation at Scale: \\ A Computational Framework for Flexible
  Batches } \\
  \vspace{.5cm} {\Large Ethan Che ~~~ Hongseok Namkoong} \\
  \vspace{.2cm}
  {\large Decision, Risk, and Operations Division, Columbia Business School} \\
  \vspace{.2cm}
  \texttt{\{eche25, namkoong\}@gsb.columbia.edu}
\end{center}


\begin{abstract}%
  
\end{abstract}

\fi

\section{Introduction}
\label{section:introduction}

Experimentation is the basis of scientific decision-making for medical
treatments, engineering solutions, policy-making, and business products
alike. Since experimenting is typically expensive or risky (e.g., clinical
trials), the cost of collecting data poses a central operational
constraint. Simultaneously, as policy interventions and engineering solutions
become more sophisticated, modern experiments increasingly involve many
treatment arms. When the differences in average rewards across arms (average
treatment effects) are small relative to the sample size, statistical power is
of fundamental concern~\citep{KohaviLoSoHe09, ButtonIoMoNoFlRoMu13,
  CziborJiLi19}. Even for online platforms that can automatically deploy
experiments to millions to billions of users, typical A/B tests are
underpowered as they involve incremental changes to a product that have a
small relative impact on key business metrics such as revenue or user
satisfaction~\citep{KohaviLoSoHe09, KohaviDeFrLoWaXu12,
  KohaviDeFrWaXuPo13}. When there is interference across individuals,
treatments may be randomized over entire markets or geographic regions
severely limiting statistical power~\citep{ImbensRu15}.

Adaptive allocation of measurement effort can significantly improve
statistical power and allow reliable identification of the optimal
decision/treatment.  Accordingly, adaptive methods---dubbed pure-exploration
multi-armed bandit (MAB) algorithms---have received tremendous attention since
the foundational works of Thompson, Chernoff, Robbins, and
Lai~\citep{Thompson33, Chernoff59, Robbins52, LaiRo85}. Most of these
algorithms are specifically designed to enjoy strong theoretical guarantees as
the number of reallocation epochs grows to infinity~\citep{BubeckCe12,
  LattimoreSz19, Slivkins19}. However, standard frameworks cannot model
typical experimentation paradigms where adaptive reallocation incurs high
operational cost. Although a universal assumption in the MAB literature,
unit-level continual reallocation of sampling effort is often expensive or
infeasible due to organizational cost and delayed feedback. Even in online
platforms with advanced experimentation infrastructures, engineering
difficulties and lack of organizational incentives deter continual
reallocation at the unit level~\citep{SculleyEtAl15, AgarwalEtAl16,
  BakshyEtAl18, NamkoongDaBa20}.

Due to the challenges associated with reallocating measurement effort, typical
real-world adaptive experiments employ a few reallocation epochs in which
outcomes are measured for many units in parallel (``batches'')
\citep{BakshyEtAl18, OfferWestortCoGr20, JobjornssonScMuFr22, ChowCh08,
  KasySa21, EspositoSa22}.
Motivated by these operational considerations, we develop and analyze batched
adaptive experimentation methods tailored to a handful of reallocation
opportunities. Our main conceptual contribution is the formulation of a
dynamic program that allows designing adaptive methods that are near-optimal
for the fixed number of reallocation epochs. Algorithms designed from our
framework can flexibly handle any batch size, and are automatically tailored
to the instance-specific measurement noise and statistical
power. Specifically, we use a normal approximation for aggregate rewards to
formulate a \emph{Gaussian sequential experiment} where each experiment epoch
consists of draws from a Gaussian distribution for each arm. The dynamic
program solves for the best adaptive allocation where noisy Gaussian draws
become more accurate with increased sampling allocation (see
Figure~\ref{fig:arm-diagram}).

\begin{figure}[t]

  \centering
  \includegraphics[height=5cm]{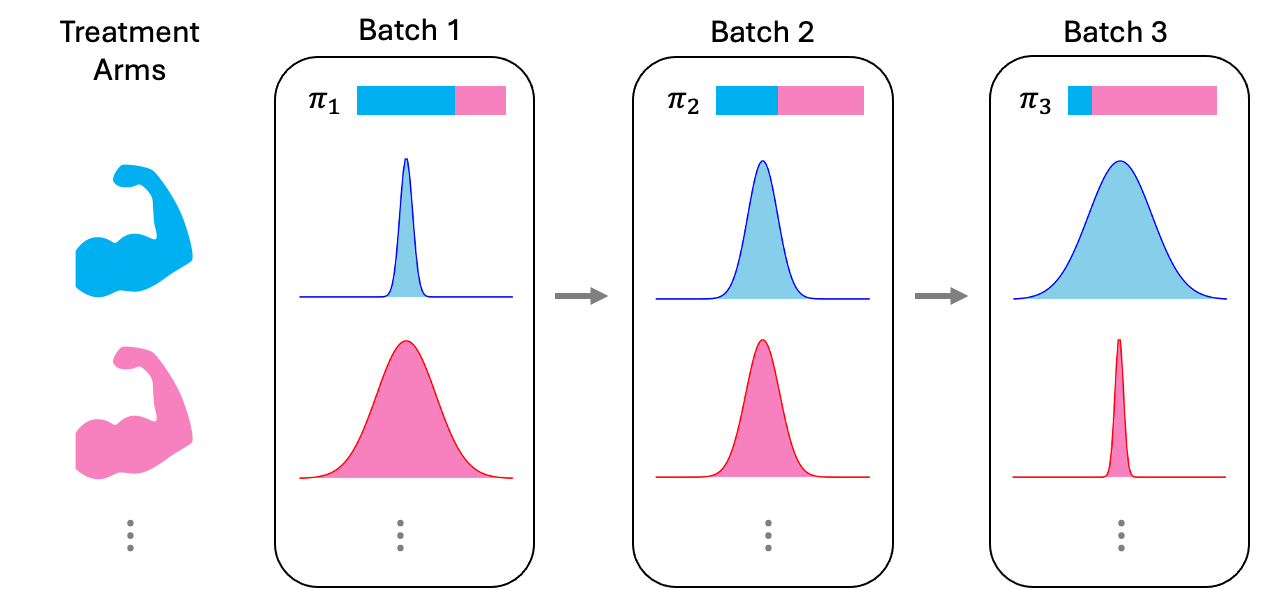}%
  \vspace{0.4cm}
  \caption{\label{fig:arm-diagram} Gaussian sequential experiment with $T=3$
    reallocation epochs. Bar visualizes sampling allocations at each epoch
    and bell curves depict normal approximations of the aggregated reward (sample mean)
    distribution.
  }
\end{figure}

Although we use (frequentist) central limit-based normal approximations to
derive the Gaussian sequential experiment, our proposed dynamic programming
framework solves for adaptive allocations using Bayesian approaches. Unlike
standard Bayesian bandit formulations (e.g.,~\citep{FrazierPoDa08, KasySa21,
  Russo20}) that require distributional knowledge of individual rewards, we
only use a prior over the \emph{average rewards} and our likelihood functions
over average rewards are Gaussian from the CLT.  Our
formulation thus allows leveraging prior information on average rewards
constructed using the rich reservoir of previous experiments, but retains some
of the advantages of model-free frequentist methods. Computationally, sampling
allocations derived from our Gaussian sequential experiment can be computed
offline after each epoch, in contrast to typical Bayesian bandit algorithms
that require real-time posterior inference (e.g., top-two Thompson
sampling~\citep{Russo20}).

Our formulation provides a computation-driven framework for algorithm
design. Despite continuous and high-dimensional state/action spaces, our
limiting dynamic program offers a novel computational advantage: the Gaussian 
sequential experiment provides a \emph{smoothed} model of partial feedback
for which sample paths are fully \emph{differentiable} with respect to the sampling
probabilities chosen by the experimenter.  By simulating a trajectories of the
experiment, we can calculate performance metrics (e.g., cumulative or simple
regret) and compute policy gradients through modern auto-differentiation
softwares such as Tensorflow~\citep{TensorFlow15} or PyTorch~\citep{PyTorch19}.
These gradients can be used to directly optimize a planning problem over
allocations, e.g., policy gradient or policy iteration. We demonstrate our
framework by providing efficient implementations of approximate dynamic
programming (ADP) and reinforcement learning (RL) methods, and evaluating them
through \emph{empirical benchmarking}. Our algorithm development and
evaluation approach is in stark contrast to the bandit literature that largely
operates on mathematical insight and researcher ingenuity~\citep{MinMoRu2020}.


Our empirical analysis highlights a simple yet effective heuristic,
$\algofull$, which iteratively solves for the optimal static allocation that
only uses currently available information. Similar to model predictive control
(MPC), the allocation is used in the corresponding period and newly collected
observations are used to update the planning problem in the next
period. Through extensive empirical validation, we demonstrate $\algo$
consistently provides major gains in decision-making performance over a rich
set of benchmark methods. Out of 640 problem instances we consider---varying
across number of reallocation opportunities, number of arms, reward
distributions, priors, etc.---$\algo$ achieves the best performance in 493
(77\%) of them despite its simplicity.  As a summary of our findings,
Figure~\ref{fig:bern-bsr-batch} compares the performance of $\algo$ against
standard batch policies such as oracle Thompson sampling-based policies that
have full distributional knowledge of individual rewards.  Despite relying on
Gaussian approximations, $\algo$ provides significant power gains in hard
instances with high measurement noise.  Overall, our approach expands the
scope of adaptive experimentation to settings standard adaptive policies
struggle with, involving few reallocation epochs, low signal-to-noise ratio,
and unknown reward distributions.

\begin{figure}[t]
  
  \centering \hspace{-.9cm} \subfloat[\centering Batch size $= 100$
samples.]{\label{fig:bern_bsr_100}{\includegraphics[height=5cm]{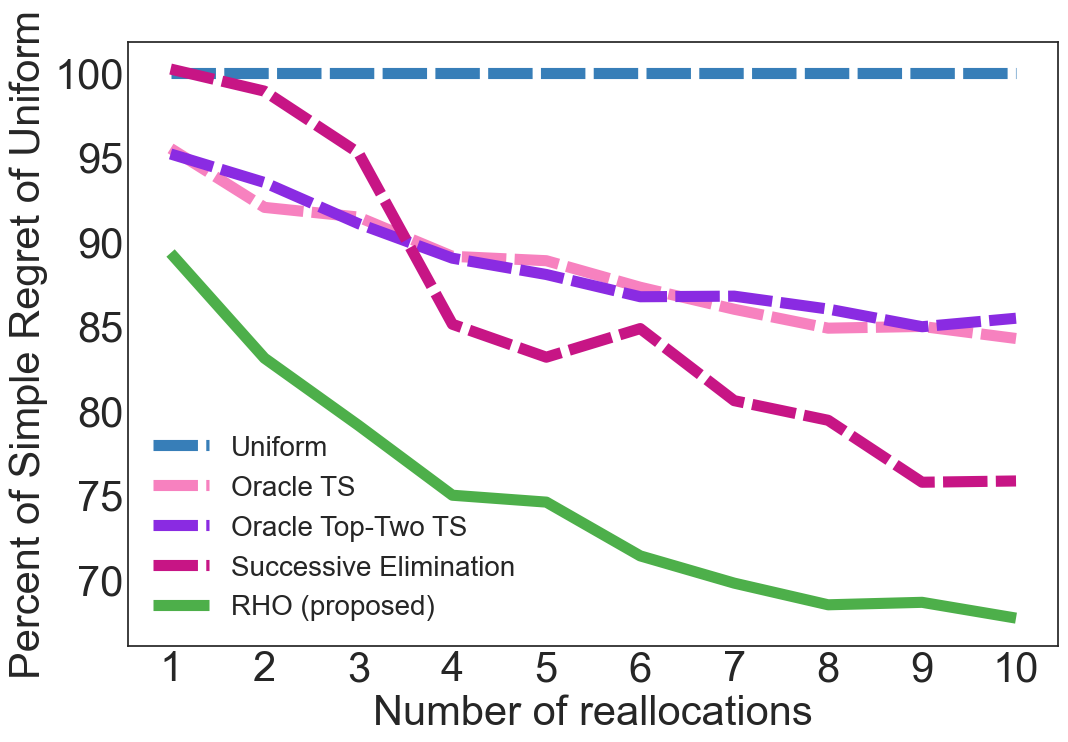}}
} \hspace{.3cm} \subfloat[\centering Batch size $= 10,000$
samples.]{\label{fig:bern_bsr_10000}{\includegraphics[height=5cm]{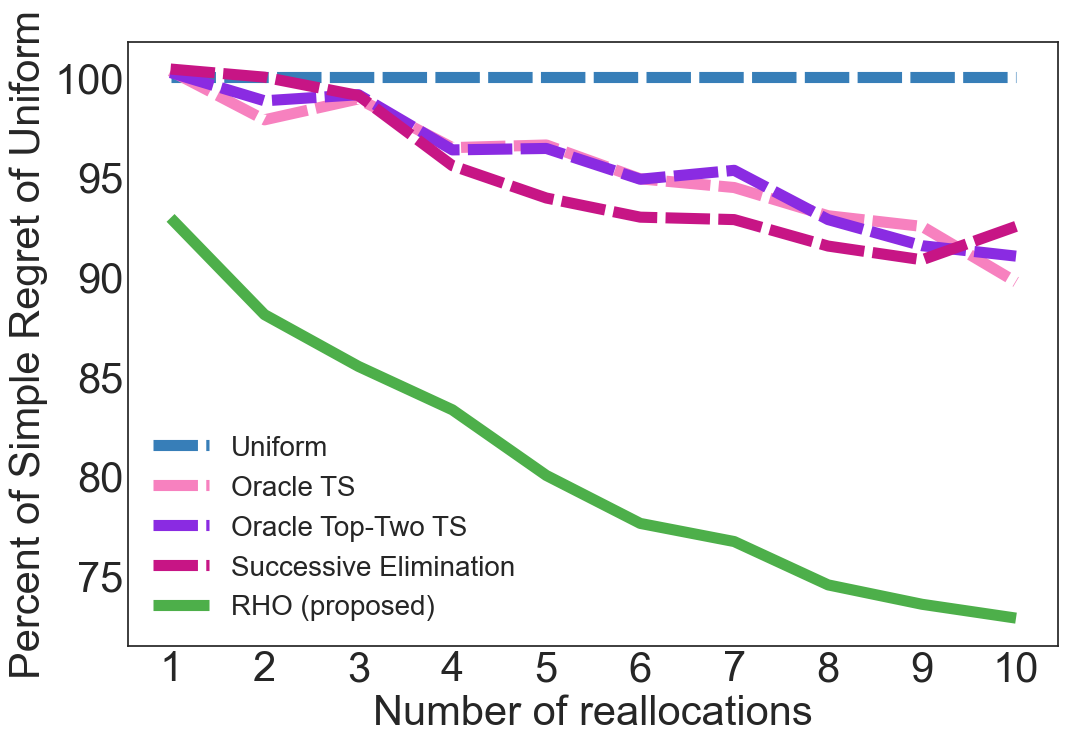}
}}

  \caption{\label{fig:bern-bsr-batch} Relative gains over the uniform
allocation as measured by the Bayes simple regret for the finite batch problem
with $K = 100$ treatment arms.  Individual rewards are Bernoulli with a
$\mbox{Beta}$ prior. Despite relying on normal approximation of aggregate
rewards over a batch, $\algo$ delivers substantial performance gains even over
oracle algorithms that know the true reward model.  These gains persist even
for small batch sizes in Figure~\ref{fig:bern_bsr_100} where batch size is
equal to the number of arms.}
\end{figure}

\paragraph{Paper outline} We begin our discussion by showing that the Gaussian
sequential experiment is a valid approximation as the batch size becomes large
(Section~\ref{section:formulation}). We study the admissible regime where
differences in average rewards scale as $1/\sqrt{n}$, where $n$ is the batch
size at each epoch.  The scaling is standard in formalizing statistical power
in inferential scenarios~\citep{VanDerVaart98, LeCamYa00}. By extending this
conceptual framework to sequential experiments, we show that normal
approximations universal in statistical inference is also useful for designing
adaptive experimentation methods. In Section~\ref{section:algorithms}, we
observe the Gaussian sequential experiment can be represented by a Markov
decision process over the experimenter's beliefs on average rewards.  At every
reallocation epoch, state transitions in the MDP are governed by posterior
updates based on new observations.

We illustrate our computation-driven paradigm to algorithm development in
Section~\ref{section:algorithms-adp} and~\ref{section:experiments}.  Our MDP
formulation gives rise to several ADP and RL-based policies, which we detail
in Section~\ref{section:adp}.  Through empirical benchmarking, we find that a
simple MPC-style policy, $\algo$, consistently outperforms even carefully
tuned RL methods (Section~\ref{section:benchmarking}).  Motivated by its
strong empirical performance, we provide basic theoretical insights on $\algo$
in Section~\ref{section:rho-analysis}. We show $\algo$ is optimal among those
that do not optimize for potential future feedback and in particular, is
always better than the uniform allocation---the de facto standard in practice.
By noting that $\algo$ can be seen as a dynamic extension of the single-batch
problem analyzed by~\citet{FrazierPo10}, we also characterize how $\algo$
calibrates exploration when the remaining sampling budget is large.  In
Section~\ref{section:experiments}, we perform a thorough empirical comparison
with standard batched bandit policies under diverse settings, e.g., low vs
high measurement noise, large vs small batches, flat vs concentrated priors.
We summarize our empirical benchmarking in the interactive
app~\eqref{eqn:streamlit}.

We defer a full discussion of the literature to the end of the paper to
maximize clarity.  Our asymptotic formulation and resulting algorithms are
new, but intimately connected to the vast body of work on adaptive allocation
of measurement effort. In Section~\ref{section:discussion}, we situate the
current paper across several fields, such as pure-exploration MABs, batched
MABs, ranking \& selection in simulation optimization, diffusion limits for
MABs, as well as various works studying Gaussian environments related to the
one we outline in Figure~\ref{fig:arm-diagram}.



\section{Gaussian Sequential Experiment}
\label{section:formulation}


Our goal is to select the best treatment arm out of $\numarm$ alternatives,
using a small (known) number of reallocation epochs ($\horizon$).  Each
experimentation epoch $t=0,...,\horizon - 1$ involves a batch of $b_{t}n$
samples, where $b_{t}>0$ is a fixed and known constant and $n$ is a scaling
parameter. The constants $b_t$ can be thought of as the length of time between
experiment epochs, which may vary across epochs $t$. Solving for the optimal
adaptive allocation is challenging as it involves a dynamic program over a
combinatorially large action space that depends on unknown reward
distributions. Any solution to such ``finite batch'' problem will depend
heavily on the batch size and need to be re-solved if the batch size changes,
that is if more/less samples are collected than expected.

To circumvent such complexity, we formulate an asymptotic approximation of the
adaptive experimentation problem using normal approximation of aggregate
rewards. The \emph{sequential Gaussian experiment} we derive in this section
provides tractable policies that can deal with any \emph{flexible batch
  size}. To derive this CLT-based large batch limit, we must first choose a
natural scaling for the average rewards. If the gap in the average reward of
each arm is $\gg1/\sqrt{n}$, adaptive experimentation is unnecessary as the
best arm can be found after only a single epoch. Conversely, if the gaps are
$\ll1/\sqrt{n}$, then we cannot reliably learn even after many experimentation
epochs; one cannot improve upon the uniform allocation (a.k.a. randomized
design or A/B testing). We thus focus on the admissible regime where the gap
between average rewards are $\Theta(1/\sqrt{n})$.

Using this scaling for average rewards, we observe i.i.d. rewards for each
unit allocated to a treatment arm $a$
\begin{equation}
  \label{eqn:rewards}
  R_a = \frac{h_a}{\sqrt{n}} + \epsilon_a,
\end{equation}
where $h_a$ is an unknown ``local parameter'' that determines the difference
in average rewards across arms. Without loss of generality, we set the
baseline reward to zero; since we assume the noise $\epsilon_a$ has mean zero,
henceforth we abuse notation and refer to $h$ as both the average rewards and
the gaps in average rewards.  We assume that the variance of the noise
$\text{Var}(\epsilon_{a})=s_a^{2}$ is known and constant. In particular,
$s_a^2$ does not scale with $n$, so it is crucial to sample arms many times to
discern differences between their means $h_a/\sqrt{n}$. Although reward
variances are typically unknown, they can be estimated from a small initial
batch in practice; empirically, the policies we consider are robust to
estimation error in $s_a^2$ and a rough estimate suffices (see Section~\ref{section:unknown-s2}).

\subsection{Gaussian sequential experiment as a limit adaptive experiment}
\label{section:formulation-gse}

Our goal is to use the information collected until the beginning of epoch $t$
to select an adaptive allocation (policy) $\pi_{t} \in\Delta_{\numarm}$, the
fraction of samples allocated to each of the $\numarm$ treatment arms. Let
$\{R_{a, j}^t\}_{a=1}^\numarm$ denote the \emph{potential rewards} for unit
$j \in [b_t n]$ at time $t$.  We use $\xi_{a, j}^t$ to denote an indicator for
whether arm $a$ was pulled for unit $j$ at time $t$.  Our starting point is
the observation that the (scaled) estimator for average rewards converges in
distribution from the central limit theorem
\begin{equation}
  \label{eqn:limit}
  \sqrt{n} \bar{R}_{t, a}^{n} \defeq
  \frac{1}{b_{t}\sqrt{n}} \sum_{j=1}^{b_{t}n}\xi_{a,j}^{t} R^{t}_{a,j}
  \cd N\left(\pi_{t,a}h_{a},\pi_{t,a}s_a^{2}\right).
\end{equation}
From the normal approximation~\eqref{eqn:limit},
$\sqrt{n} \bar{R}_{t, a}^n / \pi_{t, a}$ can be seen as an approximate draw
from the distribution $N(h_a, s_a^2/\pi_{t, a})$, giving a noisy observation
of the average reward $h_a$. The allocation $\pi_{t}$ controls the effective
sample size and the ability to distinguish signal from noise,
a.k.a. statistical power.

Using successive normal approximations~\eqref{eqn:limit} at each epoch, we
arrive at a \emph{Gaussian sequential experiment} that provides an asymptotic
approximation to the adaptive experimentation problem.
\begin{definition}
  \label{def:gse}
  A Gaussian sequential experiment is characterized by observations
  $G_0, \ldots, G_{T-1}$ with conditional distributions
  \begin{equation*}
    G_t | G_0, \ldots, G_{t-1} \sim
    N\left(\left\{\pi_{t,a}h_{a}\right\}_{a=1}^\numarm,
    \diag\left(\left\{\frac{\pi_{t, a}s_a^{2}}{b_t }\right\}_{a=1}^\numarm\right)\right).
  \end{equation*}
\end{definition}
\noindent In this asymptotic experiment, the experimenter chooses $\pi_t$ at
each epoch $t$ and observes an independent Gaussian measurement distributed as
$N\left(\pi_{t,a}h_a,\pi_{t, a} s_a^{2} \right)$ for each arm $a$.  
We use the asymptotic Gaussian sequential experiment as an approximation to the
original batched adaptive epochs and derive near-optimal adaptive
experimentation methods for this asymptotic problem.

Building on the observation~\eqref{eqn:limit}, our first theoretical result
shows that the Gaussian sequential experiment (Definition~\ref{def:gse})
provides an accurate approximation in the large batch limit $n \to \infty$.
Our limit sequential experiment extends classical local asymptotic normality
results in statistics~\cite{VanDerVaart98, LeCamYa00}. Our result relies on
the following basic conditions on the reward and policy $\pi_t$'s.
\begin{assumption}
  \label{assumption:reward}
\begin{enumerate}

\item  \label{item:ignorability}
  \emph{Ignorability}: rewards are independent from sampling decisions,
  conditional on past observations: \ifdefined\msom
$\xi^{t+1}_j \perp R^{t+1}_j \mid \bar{R}^n_0, \ldots, \bar{R}^n_t.$
\else
\begin{equation*}
  \xi^{t+1}_j \perp R^{t+1}_j \mid \bar{R}^n_0, \ldots, \bar{R}^n_t.
\end{equation*}
\fi

\item \label{item:moment} \emph{Moment restriction}: there exists $C>0$
  such that $\E\norm{\epsilon}_{2}^{4}\leq C$.
  \end{enumerate}
\end{assumption}
\noindent For the sequential Gaussian experiment to provide a valid
approximation of the objective~\eqref{eqn:bsr}, we further need regularity conditions
on the policy.
\begin{assumption}
  \label{assumption:policy}
  \begin{enumerate}
  \item \label{item:sufficiency} \emph{Sufficiency}: the policy depends on the
    data only through aggregate rewards: \ifdefined\msom
$\pi_{t}=\pi_{t}\left(\sqrt{n} \bar{R}^{n}_{0},\ldots,\sqrt{n}
\bar{R}^{n}_{t-1}\right)$ for all $t= 0, \cdots, T-1$.
\else
\begin{equation*}
  \pi_{t}=\pi_{t}\left(\sqrt{n} \bar{R}^{n}_{0},\ldots,\sqrt{n}
  \bar{R}^{n}_{t-1}\right)~~~\mbox{for all}~t= 0, \cdots, T-1
  \end{equation*}
\fi
  \item \label{item:continuity} \emph{Continuity}: discontinuity points of
    $\pi_{t}$ are measure zero with respect to $(G_{0},...,G_{t-1})$ for all
    $t$
  \end{enumerate}
\end{assumption}

\noindent Assumption~\ref{assumption:policy}\ref{item:sufficiency} is
unavoidable in our setup as we restrict attention to coarser information
represented by the aggregate rewards~\eqref{eqn:limit}.  Note however, that we
do not require sampling probabilities to be bounded away from zero, as is commonly
assumed for inferential procedures.

We are now ready to give our main asymptotic result, which we prove in
Section~\ref{section:proof-limit}. The key difficulty of this result is
to show convergence of the batch sample means to their Gaussian limit under
sampling probabilities that are themselves stochastic, 
as they are selected by the policy which is influenced by previous measurements.
The sampling probabilities are arbitrary and are allowed to even be zero, so we 
require showing uniform convergence across all possible sampling probabilities.
\begin{theorem}
  \label{theorem:limit}
  Let Assumptions~\ref{assumption:reward},~\ref{assumption:policy} hold. Then,
  the Gaussian sequential experiment in Definition~\ref{def:gse} provides a
  valid asymptotic approximation as $n \to \infty$
  \begin{equation}
    \label{eqn:weak-convergence}
    (\sqrt{n} \bar{R}^{n}_{0},\ldots,\sqrt{n} \bar{R}^{n}_{T-1}) \cd
    (G_{0},...,G_{T-1}).
  \end{equation}
\end{theorem}

\ifdefined\msom
\else
\noindent Our proof also provides a bound for the rate of convergence in
Theorem~\ref{theorem:limit}.  To metrize weak convergence, consider the
bounded Lipschitz distance $d_{\text{BL}}$
\begin{equation*}
  d_{\text{BL}}(\mu, \nu) \defeq \sup\left\{
    |\mathbb{E}_{R \sim \mu}\bar{f}(R)-\mathbb{E}_{R \sim \nu}\bar{f}(R)|:
    \bar{f}\in\text{Lip}(\mathbb{R}^{\numarm}),
    \sup_{x,y\in\mathbb{R}^{\numarm}}|\bar{f}(x)-\bar{f}(y)|\leq1\right\}.
\end{equation*}
Since we measure convergence through bounded Lipschitz test functions, we require smoothness of
the policy with respect to observations
\begin{equation}
  \label{eqn:lip-policy}
  \bar{L} := \max_{0\leq t \leq T -1} \max_{a} \max \{ \norm{\pi_{t,a}}_{\textup{Lip}}, \norm{\sqrt{\pi_{t,a}}}_{\textup{Lip}} \}<\infty.
\end{equation}
\vspace{-.4cm}
\begin{corollary}
  \label{cor:proof-limit-rate}
  Let Assumptions~\ref{assumption:reward},~\ref{assumption:policy} and the bound~\eqref{eqn:lip-policy} hold.
  Let $M := \left(1 + \bar{L} ( s_*^2/b_* \numarm^{1/2} + \max_{a} |h_{a}|) \right)$ with
  $s_{*}^{2} := \max_{a} s_{a}^{2}$, $b_{*} := \min_{0\leq t\leq T-1} b_{t}$.
 Then,
 \begin{equation}
   \label{eqn:bl-rate}
   d_{\textup{BL}}((\sqrt{n} \bar{R}^{n}_{0},\ldots,\sqrt{n} \bar{R}^{n}_{T-1}), (G_{0},\ldots,G_{T-1})) 
   \leq C M^{T} n^{-1/6}
  \end{equation}
  where $C\in (0,\infty)$ is a constant that depends polynomially 
  on $K$, $M$, $h$, $s_{*}^{2}$, $\E \norm{\epsilon}_{2}^{3}$, $\E \norm{\epsilon}_{2}^{4}$
  and $b_{*}^{-1}$.
\end{corollary}
Unlike other asymptotic normality results for adaptive sampling policies, we
do not assume anything beyond continuity of the sampling probabilities. In particular, 
we do not assume that sampling probabilities are lower bounded, as is typically done
for proving central limit theorems under adaptive sampling. As a
result, we obtain a slower rate of convergence than the standard $n^{-1/2}$
rate expected for the CLT (e.g., Berry-Esseen bound) as the batch size grows
large; a minor modification to our proof gives the usual $O(n^{-1/2})$ rate if
the sampling probabilities are bounded from below. 
We obtain an $O(n^{-1/6})$ rate through a multivariate Stein's method; since 
we do not assume a lower bound on sampling probabilities the corresponding Stein's operator
has weaker smoothing properties than in standard cases
and we require an additional Gaussian smoothing argument
to apply the result to Lipschitz test functions.

If sampling probabilities are bounded away from zero, we can obtain the usual
$n^{-1/2}$ rate of weak convergence.  We do not require that the sampling
probabilities are almost surely greater than some threshold, but a milder
condition on integrability of the inverse sampling probabilities.
\begin{assumption}
  \label{assumption:overlap}
  Overlap: there exists a constant $C_{o} >0$ such that
  \begin{align}
    \sup_{n,t} \E\left[ \max_{a} \frac{1}{\pi_{t, a}(\Pstate_{0},\ldots, \Pstate_{t-1})^{3}} \right] \leq C_{o}
  \end{align}
\end{assumption}
\noindent We prove the following convergence rate in
Section~\ref{section:proof-root-n-rate}.
\begin{corollary}
  \label{cor:root-n-rate}
  Let
  Assumptions~\ref{assumption:reward},~\ref{assumption:policy},~\ref{assumption:overlap}
  and the bound~\eqref{eqn:lip-policy} hold.  Let
  $M := \left(1 + \bar{L} ( s_*^2/b_* \numarm^{1/2} + \max_{a} |h_{a}|)
  \right)$ with $s_{*}^{2} := \max_{a} s_{a}^{2}$,
  $b_{*} := \min_{0\leq t\leq T-1} b_{t}$.  Then,
 \begin{equation}
   \label{eqn:root-n-rate}
   d_{\textup{BL}}((\sqrt{n} \bar{R}^{n}_{0},\ldots,\sqrt{n} \bar{R}^{n}_{T-1}), (G_{0},\ldots,G_{T-1})) 
   \leq \tilde{C} M^{T} n^{-1/2}
  \end{equation}
  where $\tilde{C}\in (0,\infty)$ depends on $K$, $C_{o}$, $M$, $h$,
  $s_{*}^{2}$, $\E \norm{\epsilon}_{2}^{3}$, $\E \norm{\epsilon}_{2}^{4}$, and
  $b_{*}^{-1}$ polynomially.
\end{corollary}

The bounds~\eqref{eqn:bl-rate} and~\eqref{eqn:root-n-rate} suffer an
exponential dependence on the time horizon $T$, which is perhaps unsurprising
as the result involves \emph{joint} weak convergence of correlated random
variables.  The smoothness of the policy also enters, which suggests that
convergence may be slower for policies that are very sensitive to changes in
the measurement.  Nevertheless, $T$ is very small compared to the batch size
for the settings we consider. Empirically, we observe below that the
asymptotic limit offer an accurate approximation even when the batch sizes are
small.

\begin{figure}[t]
  \vspace{-.4cm}
  \centering
  \vspace{-.4cm}
  \hspace{-.9cm}

  \ifdefined\msom
  \subfloat[\centering Simple regret for Gumbel rewards (Gamma prior)]{\label{fig:scaling_gumbel}{\includegraphics[height=5cm]{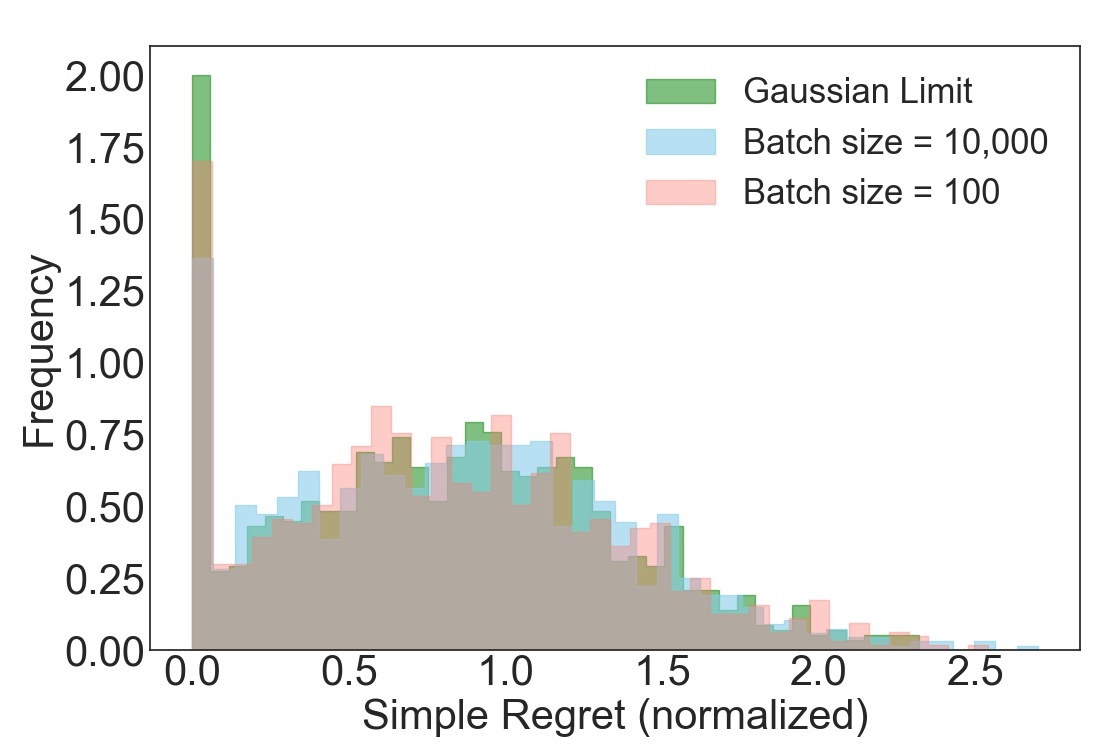}} }%
  \subfloat[\centering Simple regret for Bernoulli rewards (Beta
  prior)]{\label{fig:scaling_bern}{\includegraphics[height=5cm]{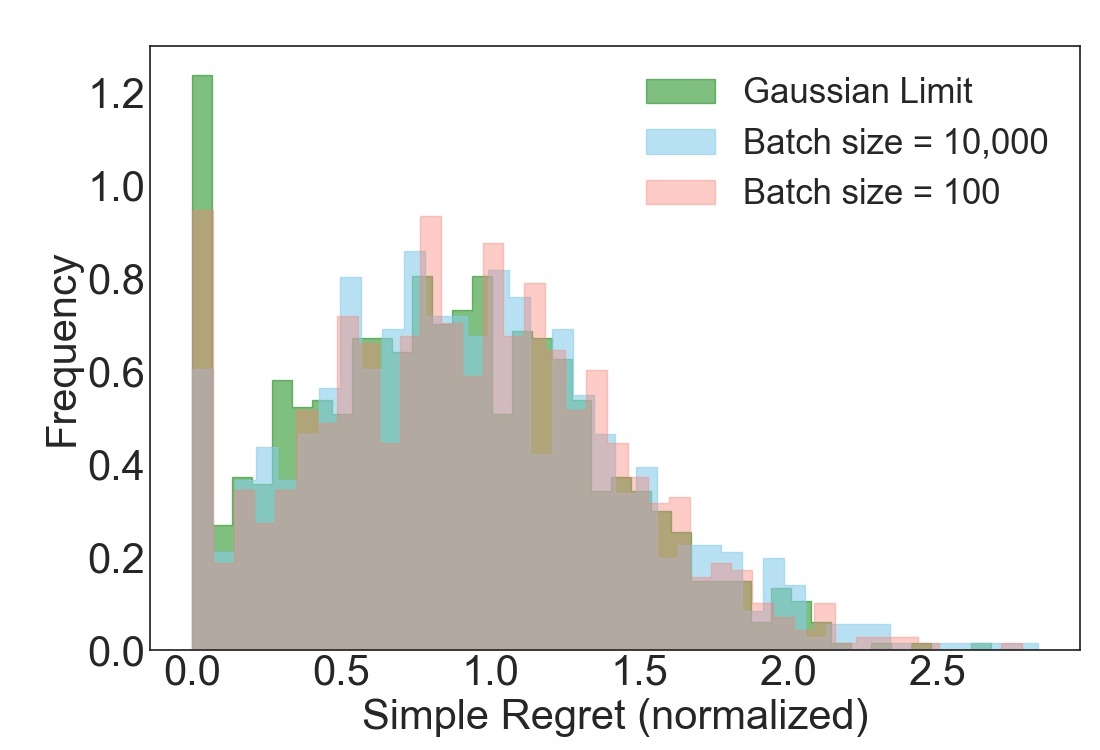}
    }}%
  \else
  \subfloat[\centering Simple regret for Gumbel rewards (Gamma prior)]{\label{fig:scaling_gumbel}{\includegraphics[height=5.7cm]{./fig/scaling_gumbel_ts.jpg}} }%
  \subfloat[\centering Simple regret for Bernoulli rewards (Beta
  prior)]{\label{fig:scaling_bern}{\includegraphics[height=5.7cm]{./fig/scaling_bern_ts.jpg}
    }}%
  \fi

  \vspace{.2cm}
  \caption{\label{fig:scaling-bsr} Histograms of simple regret after $T=10$
    epochs involving $\numarm = 100$ arms, across instances $h$ drawn from the
    prior $\nu$. Each histogram corresponds to a different batch size, with
    the green histogram corresponding to the Gaussian sequential experiment
    ($n=\infty$).  Even when the batch size is small ($b_t n = 100$), the
    performance of the policy in these non-Gaussian environments closely
    matches the performance of the policy in the Gaussian sequential
    experiment. 
      }
\end{figure}

\subsection{Validity of normal approximations even for small batch sizes}
\label{section:formulation-validity}

Empirically, we observe that the Gaussian sequential experiment is a good
approximation even when the batch sizes are exceedingly small. While
Theorem~\ref{theorem:limit} and Corollary~\ref{cor:bsr-limit} only guarantees
validity of our approximation in the large batch limit, this observation
impels us to apply policies derived from the problem~\eqref{eqn:gse-bsr} to
problems with small batches.  Overall, our observations are consistent with
the longstanding folklore in statistics that the central limit theorem often
provides a practical approximation even for small samples.

We illustrate that the sequential Gaussian experiment provides an accurate
prediction of performance in finite batch systems. Our primary metric of performance
is simple regret, the difference in the average reward between the best arm
and the arm selected by the policy at the end of experiment (according to a policy denoted as $\pi_{T}$). 
Since modern experimental platforms typically run many experiments, it is natural for the experimenter
to have a prior $\nu$ over the average rewards $h$ of the treatment arms, so we focus on simple regret among
instances drawn from this prior. For different batch
scalings $n$, we compare the \emph{distribution} of the simple regret
\begin{equation}
  \label{eqn:sr}
  \E_{\what{a}} [h_{a\opt} - h_{\what{a}}]~~~\mbox{where}~~
  \what{a} \sim \pi_T\left( \sqrt{n}\bar{R}_0,\ldots, \sqrt{n}\bar{R}_{T-1} \right)
\end{equation}
over the prior $h \sim \nu$, to its limiting object
$\E_{\what{a} \sim \pi_T(G_0, \ldots, G_{T-1})} [h_{a\opt} - h_{\what{a}}]$.

For simplicity, we consider a large batch limit version of the Thompson
sampling policy; we observe similar trends for different policies $\pi$.  As
we detail in Section~\ref{section:algorithms}, this Bayesian policy maintains
a Gaussian belief over the average rewards $h$ and updates posteriors using
observed aggregate rewards $\sqrt{n}\bar{R}^{n}$. Figure~\ref{fig:scaling-bsr}
displays histograms of the simple regret~\eqref{eqn:sr} incurred by this
policy across experiments with $T = 10$ reallocation epochs and
$\numarm = 100$ treatment arms.  Each histogram depicts the distribution of
the simple regret~\eqref{eqn:sr} over the prior $h \sim \nu$ for a particular
batch scaling $n$, including under the Gaussian sequential experiment
corresponding to $n=\infty$. Even for exceptionally small batch sizes relative
to the number of arms (batch size of $b_t n = 100$ across $K=100$ arms), the
Bayes simple regret closely matches that predicted by the Gaussian sequential
experiment.

\section{Bayesian Adaptive Experimentation}
\label{section:algorithms}

It is natural for the experimenter to have a prior distribution $h \sim \nu$
over the relative gap between average rewards. For example, modern online
platforms run thousands of experiments from which a rich reservoir of
previously collected data is available~\citep{KohaviLoSoHe09,
  KohaviDeFrLoWaXu12, KohaviDeFrWaXuPo13}. In this work, we focus on
minimizing the \emph{Bayes simple regret} at the end of the experiment
\begin{equation}
  \label{eqn:bsr}
  \bsr_{T}(\pi,\nu,\bar{R})
  \defeq \E_{h \sim \nu} \E[h_{a\opt} - h_{\what{a}}]
  ~~~\mbox{where}~~\what{a} \sim \pi_T~\mbox{and}~a\opt \in \argmax_a h_a,
\end{equation}
the scaled optimality gap between the final selection $\what{a}$ and the
optimal arm $a\opt$ averaged over the prior.  The notation
$\bsr_{T}(\pi,\nu,\bar{R})$ considers adaptive policies
$\pi = \{\pi_t\}_{t=0}^T$, prior $\nu$ over $h$, and observation process
$\bar{R} = (\bar{R}_0,\ldots, \bar{R}_{T-1} )$, which is the set of aggregate
rewards used by the policy to determine the sampling allocations.

Instead of optimizing the Bayes simple regret for each finite batch size, we
use Theorem~\ref{theorem:limit} to derive an asymptotic approximation under
the Gaussian sequential experiment.  
\begin{corollary}
  \label{cor:bsr-limit}
  Let $\nu$ be a prior over average rewards $h$ satisfying
  $\E_{h \sim \nu} \lone{h} < \infty$ and let
  Assumptions~\ref{assumption:reward},~\ref{assumption:policy} hold. Then, the
  Bayes simple regret under $\bar{R}^n$ can be approximated by that under the
  Gaussian sequential experiment $G$ in Definition~\ref{def:gse}:
  \begin{equation}
    \label{eqn:bsr-limit}
    \bsr_{T}(\pi,\nu, \sqrt{n} \bar{R}^{n}) \to \bsr_{T}(\pi,\nu, G).
  \end{equation}
\end{corollary}
\noindent See Section~\ref{section:proof-bsr-limit}
for a proof of the corollary.

Using the approximation of Bayes simple regret, we optimize the asymptotic
Bayesian objective
\begin{equation}
  \label{eqn:gse-bsr}
  \minimize_{\pi}~~\left\{
    \bsr_{T}(\pi,\nu, G) = \E_{h \sim \nu} \E[h_{a\opt} - h_{\what{a}}]
    \right\}
\end{equation}
over policies $\pi = \{\pi_t(G_0, \ldots, G_{t-1})\}_{t=0}^T$ adapted to the
sequential observations $G_0, \ldots, G_{T-1}$ of the Gaussian sequential
experiment.  Policies derived from the optimization
problem~\eqref{eqn:gse-bsr} has marked modeling flexibility and computational
advantages compared to typical Bayesian sequential sampling algorithms (e.g.,
variants of Thompson sampling~\citep{Russo20}). Although reward distributions
are unknown in general, Bayesian sequential sampling approaches require a
distributional model of individual rewards comprising of prior and likelihood
function~\citep{RussoVaKaOsWe18}.  In contrast, our framework does not assume
restrictive distributional assumptions on rewards and allows naturally
incorporating prior information over average rewards $h_a$. 

Computationally, policies derived from the optimization
problem~\eqref{eqn:gse-bsr} can be efficiently updated \emph{offline}. As
these offline updates give a fixed sampling probability over the $\numarm$
arms at each epoch, the policies we propose in subsequent sections can be
deployed to millions of units at ease. This is in stark contrast to Bayesian
sequential sampling algorithms designed for unit-level feedback. While their
policy updates are also often performed offline in batches in
practice~\citep{OfferWestortCoGr20, EspositoSa22}, when deployed these policies
require \emph{real-time} posterior inference and action optimization to
generate treatment assignments for each unit. Implementing such methods at
scale is highly challenging even for the largest online platforms with mature
engineering infrastructure~\citep{AgarwalBiCoHoLaLeLiMeOsRi16,
  NamkoongDaBa20}. Methods derived from our formulation provides a scalable
alternative as deploying the resulting adaptive policies only involves
sampling from a fixed sampling probability $\pi_t$ for every unit in a batch,
\emph{regardless} of the batch size.

\subsection{Markov decision process}
\label{section:algorithms-mdp}

Using the Gaussian sequential experiment as an asymptotic approximation, we
derive a dynamic program that solves for the adaptive allocation minimizing
the Bayes simple regret~\eqref{eqn:gse-bsr}. Policies derived from this
dynamic program---which we call the (limiting) Bayesian adaptive
experiment---are tailored to the signal-to-noise ratio in each problem
instance and the number of reallocation opportunities $T$. In the Bayesian
adaptive experiment, we observe a Gaussian draw $G_t$ at each epoch and use it
to perform posterior updates over the experimenter's belief on the average
rewards $h$. We formulate the sequential updates using a Markov decision
process where state transitions are governed by changes in the posterior mean
and variance.
In Section~\ref{section:algorithms-adp}, we introduce a range of adaptive
sampling algorithms that (approximately) solve the Bayesian dynamic program,
and benchmark them empirically. Our empirical analysis highlights a
particularly effective policy, $\algofull$, which solves an open-loop problem
over future allocations that only use currently available information.

The normal approximations in the previous section gives Gaussian likelihood
functions in our Bayesian adaptive experiment
\begin{equation*}
  \mbox{Likelihood function:}~~~~
  G \mid h \sim N\left(\pi h, \diag\left( \frac{\pi s^2}{b} \right)\right),
\end{equation*}
where we abuse notation to describe elementwise algebraic operations over
$\numarm$-dimensional vectors. (Recall that $s^2 \in \R^{\numarm}$ is the
measurement variance on the raw rewards~\eqref{eqn:rewards}.) To achieve
conjugacy, we assume that the experimenter has a Gaussian prior over the
average reward
\begin{equation}
  \label{eqn:prior}
 \mbox{Prior:}~~~~ h \sim N\left(\mu_0, \diag(\sigma_0^2)\right) \eqdef \nu,
\end{equation}
where $\mu_0, \sigma_0^2 \in \R^{\numarm}$ are the prior mean and
variance. Standard posterior updates for Gaussian conjugate priors give the
following recursive formula for the posterior mean and variance
\begin{subequations}
  \label{eqn:posterior}
  \begin{align}
    \mbox{Posterior variance:}~~~~ \sigma_{t+1, a}^{-2}
    & \defeq \sigma_{t, a}^{-2} + s_a^{-2} b_t \pi_{t, a}
      \label{eqn:posterior-var} \\
    \mbox{Posterior mean:}~~~~\mu_{t+1, a}
    & \defeq   \sigma_{t+1, a}^{2}
      \left( \sigma_{t, a}^{-2} \mu_{t, a} + s_a^{-2}b_t G_{t, a} \right).
      \label{eqn:posterior-mean}
  \end{align}
\end{subequations}
The posterior variance decreases as a deterministic function of the sampling
allocation $\pi_t$, and in particular does not depend on the observation
$G_t$.

Under the posterior updates~\eqref{eqn:posterior}, our goal is to optimize the
Bayes simple regret~\eqref{eqn:gse-bsr} at the end of the experiment (time
$t = T$). We consider the final allocation that simply selects the arm with
the highest posterior mean; formally, if the $\argmax_a \mu_{T, a}$ is unique,
this is equivalent to 
\begin{equation*}
  \pi_{T, a} = \begin{cases}
    1 ~~&~\mbox{if}~a = \argmax_{a} \mu_{T, a} \\
    0 ~~&~\mbox{if}~a \neq\argmax_{a} \mu_{T, a}. 
  \end{cases}
\end{equation*}
(More generally, we choose an arm randomly from $\argmax_a \mu_{T, a}$; this
does not change any of the results below.)  Then, minimizing the Bayes simple
regret under the asymptotic Gaussian sequential experiment~\eqref{eqn:gse-bsr}
is equivalent to the following reward maximization problem
\begin{equation}
  \label{eqn:bayesian-dp}
  \maximize_{\pi_0, \ldots, \pi_{T-1}}~~
  \left\{V_{0}^{\pi}(\mu_{0},\sigma_{0}) = \E^{\pi} \left[ \max_a \mu_{T, a} \right]\right\}.
\end{equation}
Here, we write $\E^\pi$ to denote the expectation under the stochastic
transitions~\eqref{eqn:posterior} induced by the policy
$\pi_0, \ldots, \pi_{T-1}$ adapted to the observation sequence
$G_0, \ldots, G_{T-1}$. Although we focus on the Bayes simple regret in this
work, our MDP formulation can accomodate any objective function or constraint
that depends on the posterior states, such as the cumulative regret,
probability of correct selection, simple regret among the top-$k$ arms, and
constraints on the sampling allocations.


To simplify things further, we use a change of variable to reparameterize
state transitions~\eqref{eqn:posterior} as a random walk.
\begin{lemma}
  \label{lemma:mdp}
  Let $h \sim N(\mu_0, \sigma_0^2)$ and
  $Z_0, \ldots, Z_{T-1} \simiid N(0, I_{\numarm})$ be standard normal
  variables. The system governed by the posterior
  updates~\eqref{eqn:posterior} has the same joint distribution as that with
  the following state transitions
  \begin{subequations}
    \label{eqn:dynamics}
    \begin{align}
      \sigma_{t+1, a}^{-2}
      & \defeq \sigma_{t, a}^{-2} + s_a^{-2} b_t \pi_{t, a}(\mu_t, \sigma_t)
        \label{eqn:dynamics-var} \\
      \mu_{t+1, a}
      &:= \mu_{t, a} + \sigma_{t, a}\sqrt{\frac{b_t \pi_{t, a}(\mu_t, \sigma_t)
        \sigma_{t, a}^{2}}{s_a^{2}+b_t\pi_{t, a}(\mu_t, \sigma_t)\sigma_{t, a}^{2}}} Z_{t, a}.
        \label{eqn:dynamics-mean}
    \end{align}
  \end{subequations}
\end{lemma}
\noindent We defer derivation details to Section~\ref{section:proof-mdp}.

For the Markov decision process defined by the reparameterized state
transitions~\eqref{eqn:dynamics}, the value function for the Bayes simple regret
$\E^{\pi} \left[ \max_a \mu_{T, a} \right]$ has the following succinct
representation
\begin{align}
  V_{t}^{\pi}(\mu_{t},\sigma_{t})
  & = \E^{\pi} \left[\max_a \mu_{T,a} \mid \mu_t, \sigma_t \right] \nonumber \\
  & = \E^{\pi} \left[ \max_{a} \left\{ \mu_{t,a} + \sum_{v=t}^{T-1} \sigma_{v,a}\sqrt{\frac{b_v \pi_{v, a}(\mu_{v},\sigma_{v})\sigma_{v, a}^{2}}{s_a^{2}+b_v\pi_{v, a}(\mu_{v},\sigma_{v})\sigma_{v, a}^{2}}}Z_{v, a}
    \right\} ~\Bigg|~ \mu_t, \sigma_t \right].   \label{eqn:q_fn}
\end{align}
In what follows, we use the notational shorthand
$\E_t^{\pi}[\cdot] = \E^{\pi}[\cdot \mid \mu_t, \sigma_t^2]$ so that
$V_{t}^{\pi}(\mu_{t},\sigma_{t}) = \E^{\pi}_t \left[\max_a \mu_{T,a} \right]$.

\subsection{Asymptotic validity}

The state transitions~\eqref{eqn:posterior} are derived under the idealized
the Gaussian sequential experiment. In practice, any policy $\pi$ derived in
this asymptotic regime will be applied to a finite batch problem.  Recalling
the asymptotic approximation~\eqref{eqn:weak-convergence}, the finite batch
problem will involve states $(\mu_{n,t}, \sigma_{n,t})$ updated according to
the same dynamics~\eqref{eqn:posterior}, but using the sample mean estimator
$\sqrt{n}\bar{R}^{n}_{t}$ instead of the Gaussian approximation $G_t$
\begin{subequations}
  \label{eqn:pre-limit-posterior}
  \begin{align}
    \mbox{Pre-limit posterior variance:}~~~~ \sigma_{n,t+1, a}^{-2}
    & \defeq \sigma_{n,t, a}^{-2} + s_a^{-2} b_t \pi_{t, a}(\mu_{n,t}, \sigma_{n,t})
      \label{eqn:pre-limit-posterior-var} \\
    \mbox{Pre-limit posterior mean:}~~~~\mu_{n,t+1, a}
    & \defeq   \sigma_{n,t+1, a}^{2}
      \left( \sigma_{n,t, a}^{-2} \mu_{n,t, a} + s_a^{-2}b_t \sqrt{n}\bar{R}_{t,a}^{n} \right).
      \label{eqn:pre-limit-posterior-mean}
  \end{align}
\end{subequations}
Given any policy $\pi_t(\mu_t, \sigma_t)$ derived from the Bayesian dynamic
program~\eqref{eqn:bayesian-dp} and average rewards $h$, the pre-limit
posterior state $(\mu_{n,t}, \sigma_{n,t})$ evolves as a Markov chain and each
adaptive allocation to be used in the original finite batch problem is given
by $\pi_{t}(\mu_{n,t}, \sigma_{n,t})$.

For policies $\pi$ that is continuous in the states---satisfying conditions in
Assumption~\ref{assumption:policy}---Corollary~\ref{cor:bsr-limit} implies
that the Bayesian dynamic program~\eqref{eqn:bayesian-dp} is an accurate
approximation for measuring and optimizing the performance of the policy as
$n$ grows large. Moreover, we can show that the trajectories of pre-limit
posterior beliefs $(\mu_{n,t}, \sigma_{n,t})$~\eqref{eqn:pre-limit-posterior}
can be approximated by trajectories of the Markov decision process derived
using the Gaussian sequential experiment~\eqref{eqn:dynamics}.
\begin{corollary}
  \label{cor:bayes-limit}
  Let $\pi_t(\mu_t, \sigma_t)$ be a policy derived under the limit Bayesian
  dynamic program~\eqref{eqn:bayesian-dp} that is continuous as a function of
  $(\mu_t, \sigma_t)$.  Let Assumption~\ref{assumption:reward} hold and
  consider any fixed average reward $h$ and prior $(\mu_{0}, \sigma_{0})$. The
  posterior states~\eqref{eqn:pre-limit-posterior} converge to the
  states~\eqref{eqn:posterior} as $n\to\infty$
  \begin{equation*}
    (\mu_{n,t}, \sigma_{n,t},\ldots,\mu_{n,T-1}, \sigma_{n,T-1})
    \cd (\mu_{t}, \sigma_{t},\ldots,\mu_{T-1}, \sigma_{T-1})
  \end{equation*}
\end{corollary}
\noindent See Section~\ref{section:proof-bayes-limit} for the proof.



\section{Algorithm design through empirical benchmarking}
\label{section:algorithms-adp}

We now derive algorithms for solving the limit Bayesian dynamic
program~\eqref{eqn:bayesian-dp}, which arise naturally from reinforcement
learning (RL) and approximate dynamic programming (ADP).  Breaking from the
typical theory-driven paradigm in bandit algorithms that compare regret
bounds, we empirically benchmark adaptive experimentation methods to
assess their performance. To ensure \emph{empirical rigor}, we present a
comprehensive set of experiments in the following Streamlit app:
\begin{equation}
  \label{eqn:streamlit}
    \mbox{\url{https://aes-batch.streamlit.app/}}.
\end{equation}
Our empirical approach allows us to study performance among a large class of
algorithms, and analyze factors that impact performance (e.g.  measurement
noise $s_a^2$, horizons length $T$). Our focus on empirical analysis reveals
practical aspects of implementation that are critical for performance, but are
difficult to observe from a purely theoretical approach.

\subsection{Adaptive experimentation as approximate dynamic programming}
\label{section:adp}

 Since the states
$(\mu_t, \sigma_t)$ and actions $\pi_t$ are both continuous, using dynamic
programming to directly solve the policy optimization
problem~\eqref{eqn:bayesian-dp} is computationally intractable even for a
moderate number of arms and reallocation epochs. However, a key property of
this MDP is that the state transitions are differentiable with respect to the
sampling allocations along any sample path $Z_0,...,Z_{T-1}$, which is enabled
by the reparameterization trick for Gaussian random
variables~\citep{KingmaBa15}. As a result, standard Monte Carlo approximations
of the value function are differentiable with respect to the sampling
allocations, allowing the use of gradient-based methods for planning and
(approximately) solving the DP. In this section, we explore algorithms that
utilize this auto-differentiability~\citep{TensorFlow15, PyTorch19}.

\paragraph{Residual Horizon Optimization (RHO)} We propose a simple yet
effective method that solves a computationally cheaper approximation of the
dynamic program based on model predictive control (MPC). $\algo$ iteratively solves an open-loop planning
problem, optimizing over a sequence of future sampling allocations based on currently
available information $(\mu_t, \sigma_t)$ at time $t$. At each epoch $t$, the
allocation $\rho_t(\mu_t, \sigma_t)$ is derived by assuming that a fixed
sequence of allocations will be used for the remaining periods horizon
regardless of new information obtained in the future
\begin{align}
  \label{eqn:rho-planning}
  \rho_{t}(\mu_t,\sigma_t)
  \in  \argmax_{\bar{\rho}_{t},\ldots,\bar{\rho}_{T-1} \in \Delta_\numarm}
    ~\E_{t} \left[ \max_{a} \left\{ \mu_{t,a}
    + \sum_{v=t}^{T-1} \sigma_{v,a}\sqrt{\frac{b_v \bar{\rho}_{v,a}\sigma_{v, a}^{2}}
    {s_a^{2}+b_v \bar{\rho}_{v,a}\sigma_{v, a}^{2}}}Z_{v, a}
    \right\}~~\Bigg|~~ \mu_{t},\sigma_{t} \right],
\end{align}
where $\rho_{t}(\mu_{t}, \sigma_{t}) = \rho_{t}^{*}$ for the sequence $\rho_{t}^{*},\ldots,\rho_{T-1}^{*}$
that maximizes the planning objective.

\begin{algorithm}[t]
  \caption{\label{alg:rho} $\algofull$}
  \begin{algorithmic}[1]
    \State \textsc{Input: prior mean and variance on average rewards $(\mu_0, \sigma_0)$}
    \State Initialize pre-limit states $(\mu_{n, 0}, \sigma_{n, 0}) = (\mu_0, \sigma_0)$
    \For{each epoch $t \in 0, \ldots, T-1$}
    \State Letting $\bar{b}_t \defeq \sum_{v=t}^{T-1} b_v$, solve the following (e.g., using stochastic gradient methods)
    \begin{equation}
      \label{eqn:rho}
      \rho_{t}(\mu_t,\sigma_t) \in
      \argmax_{\bar{\rho} \in \Delta_\numarm}
      \left\{ V^{\bar{\rho}}_{t}(\mu_{t},\sigma_{t})
        \defeq \E_{t}  \left[ \max_{a} \left\{ \mu_{t,a}
            + \sqrt{\frac{\sigma_{t, a}^4 \bar{\rho}_{a} \bar{b}_{t}}
              {s_a^2 + \sigma_{t, a}^2 \bar{\rho}_{a} \bar{b}_{t}}} Z_{t, a}
          \right\} \right] \right\}
    \end{equation}
    \State For each unit $j =1, \ldots, b_t n$,
    sample arms according to $\rho_{t}(\mu_t,\sigma_t)$ 
    and observe reward $R^t_{a, j}$ if  arm $a$ was sampled ($\xi_{a, j}^t = 1$)
    \State Use aggregate rewards 
    $\sqrt{n} \bar{R}_{t, a}^{n} = \frac{1}{b_{t}\sqrt{n}}
    \sum_{j=1}^{b_{t}n}\xi_{a,j}^{t} R^{t}_{a,j}$ and  the
    formula~\eqref{eqn:pre-limit-posterior} to compute the next
    state transition $(\mu_{n, t+1}, \sigma_{n, t+1})$ 
    \EndFor 
    \State \Return $\argmax_{a \in [\numarm]} \mu_{n, T, a}$
  \end{algorithmic}
\end{algorithm}

The planning problem~\eqref{eqn:rho-planning} can be simplified to consider a
constant allocation $\bar{\rho}\in \Delta_{K}$ to be deployed for every
remaining period.  
Intuitively,
since the allocation does not change as the experiment progresses, we can
accumulate updates to the posterior beliefs at each epoch and think of it as a
single update to the current belief, resulting in the terminal posterior
belief.  In this sense, the policy can be seen as a dynamic extension
of the single-batch problem studied in~\cite{FrazierPo10}. We summarize how our procedure will be applied to finite batch
problems in Algorithm~\ref{alg:rho}, and formalize this insight in the below
result. See Section~\ref{section:proof-rho-reduction} for its proof. 
\begin{lemma}
  \label{lemma:rho-reduction}
  Let $\bar{b}_t \defeq \sum_{v=t}^{T-1} b_v$. For any sequence of future
  allocations $\bar{\rho}_{t},\cdots,\bar{\rho}_{T-1} \in \Delta_\numarm$ that only
  depends on $(\mu_t,\sigma_t)$, there is a constant allocation
  $\bar{\rho}(\mu_t, \sigma_t)$ achieving the same Bayes simple regret 
  \begin{equation*}
    V^{\bar{\rho}_{t:T-1}}_{t}(\mu_{t},\sigma_{t}) = V^{\bar{\rho}}_{t}(\mu_{t},\sigma_{t})
    ~~\mbox{where}~~V^{\bar{\rho}}_{t}(\mu_{t},\sigma_{t})~\mbox{is defined in Eq.~}\eqref{eqn:rho}.
  \end{equation*}
  Thus, it is sufficient to solve the constant allocation planning
  problem~\eqref{eqn:rho} to achieve the same optimal value as the original
  problem~\eqref{eqn:rho-planning}.
\end{lemma}

In the empirical benchmarks to come, we find that among Bayesian policies that
utilize batch Gaussian approximations, $\algofull$ achieves the largest
performance gain across a wide variety of settings. When the number of
reallocation epochs is small, $\algo$ calibrates the level of exploration to
the time horizon by iteratively planning with the Gaussian sequential
experiment~\eqref{eqn:rho}.

\paragraph{Pathwise Policy Iteration} We also consider reinforcement
learning/approximate DP methods for solving the full horizon
problem~\eqref{eqn:bayesian-dp}. 
Our framework enables the use of policy iteration~\citep{SuttonBa18}
to improve upon a standard sampling policy, such as Thompson Sampling or Top-Two Thompson Sampling.
Policy iteration (also referred to as `Rollout') takes a base policy $\pi$
and returns an improved policy $\pi'$ that selects the sampling allocation at each epoch $t$ that maximizes
the $Q$-function of $\pi$, which
gives the expected reward of selecting sampling allocation $\bar{\rho}$ at state $(\mu_{t}, \sigma_{t})$
and epoch $t$ given that all future sampling allocations are determined by the base policy,
\[
  Q_{t}^{\pi}(\bar{\rho}; \mu_{t}, \sigma_{t})
  = \E_{t} \left[ V_{t+1}^{\pi}\left(\mu_{t} + \sigma_{t}\sqrt{\frac{b_t \bar{\rho}
  \sigma_{t}^{2}}{s^{2}+b_t\bar{\rho}\sigma_{t}^{2}}} Z_{t}, 
  (\sigma_{t}^{-2} + s^{-2} b_t \bar{\rho})^{-1} \right) \right].
\]
Optimizing $Q_{t}^{\pi}$ would typically require discretizing the state and
action spaces or using function approximation for the Q-function. However,
since the dynamics of the MDP are differentiable, one can simply draw a Monte
Carlo sample of the $Q$-function and directly compute the gradient with
respect to $\bar{\rho}$ through auto-differentiation~\citep{TensorFlow15, PyTorch19}, as long
as the base policy $\pi$ is differentiable with respect to the state
$(\mu, \sigma)$.  One can then use stochastic gradient ascent to optimize the $Q$-function. 
Although differentiability is not guaranteed
for all policies, we observe there is often reliable differentiable
surrogates. For example, for Thompson Sampling we have the following
approximation
\[
  \begin{aligned}
\pi^{\text{TS}}(\mu, \sigma) 
= \E\left[\textsf{onehot}(\argmax_{a} \mu_{a} + \sigma_{a}Z_{a})\right]
\approx \E\left[\textsf{softmax}(\mu_{a} + \sigma_{a}Z_{a})\right] =: \hat{\pi}^{\text{TS}}(\mu, \sigma)
\end{aligned}
\]
where $Z_{a} \simiid N(0, 1)$, $\textsf{onehot}$ maps an index to its
corresponding one-hot vector, and
$\textsf{softmax}(v) \defeq \exp(v_{a})/\sum_{a'} \exp(v_{a'})$. We define the
policy TS$+$ to be policy iteration on the approximate Thompson sampling
policy, in which at every state $(\mu_{t}, \sigma_{t})$ the allocation is
determined by solving
\begin{equation}
  \label{eqn:TS-plus}
  \pi^{\text{TS}+}_{t}(\mu_{t}, \sigma_{t}) \in \argmax_{\bar{\rho}\in \Delta_{K}} Q_{t}^{\hat{\pi}^{\text{TS}}}(\bar{\rho}; \mu_{t}, \sigma_{t})
\end{equation}
with stochastic gradient ascent.

\paragraph{Pathwise Policy Gradient} We apply policy gradient (PG) computed through
black-box ML models, as well as their limited lookahead
variants~\citep{SuttonBa18}. We consider an differentiable parameterized
policy $\pi_\theta = \{\pi_{t}^{\theta} \}_{t=0}^{T-1}$ (e.g., neural
networks) and aim to directly optimize the value function~\eqref{eqn:q_fn} using
stochastic approximation methods. We use stochastic gradient ascent over
sample paths $Z_0, \ldots, Z_{T-1}$ to update policy parameters $\theta$:
\begin{equation}
  \label{eqn:policy-gradient}
  \theta \leftarrow \theta + \alpha \nabla_{\theta} V_{0}^{\pi_{\theta}}(\mu_{0},\sigma_{0}),
\end{equation}
for step-size $\alpha$ and prior $(\mu_{0}, \sigma_{0})$.
As for policy iteration, that the gradients of the value function can be computed \emph{exactly}
through auto-differentiation along every sample path.
This is in contrast with standard RL methods 
for continuous state and action spaces,
such as Deep Deterministic Policy Gradient (DDPG) \citep{LillicrapEtAl16},
that require function approximation for the Q-functions, a step
which can introduce significant approximation errors.

For long horizons, training a policy with gradient descent is more
difficult due to vanishing
gradients. Policies are evaluated solely on the simple regret incurred at the
end of the experiment, which makes it difficult to do credit assignment for
sampling allocations near the beginning of the experiment.  For policy gradient, we consider
$m$-lookahead policies (denoted PG-$m$) trained to optimize the shorter time
horizon objective $V_{T-m}^{\pi_\theta}(\mu_{0},\sigma_{0})$. 


\subsection{Empirical comparison of Gaussian batch policies}
\label{section:benchmarking}

There are many policies that use the Gaussian sequential experiment
(Definition~\ref{def:gse}) as an approximation for batch rewards.  To identify
performant policies among them, we turn to empirical benchmarking on realistic
examples.  In total, we consider 640 problem instances across different number
of reallocation opportunities, number of arms, reward distributions, and
priors. 

\paragraph{Setup}
Since we are interested in the performance of these policies in the pre-limit
problem, we simulate state transitions~\eqref{eqn:pre-limit-posterior} based
on observed aggregate rewards $\sqrt{n} \bar{R}_{t, a}^{n}$.  As a concrete
illustration of our main findings, consider a setting with $K = \{10, 100 \}$
arms with up to $T=10$ batches. All batches have the same number of
observations in each epoch and we consider two batch sizes: $b_{t}n = 100$
samples and $b_{t}n = 10,000$ observations per epoch. We evaluate each policy
using the Bayes simple regret~\eqref{eqn:bsr} under the true prior. For
policies that use the Gaussian sequential experiment, we train them by
approximating the true prior with a Gaussian prior with the same mean and
standard deviation in order to preserve conjugacy and utilize the MDP
formulation~\eqref{eqn:dynamics}. All policies are evaluated according to the
pre-limit finite batch problem under the true prior.

We consider two data-generating distributions. We use the
$\textsf{Beta-Bernoulli experiment}$ as our primary vehicle, where rewards for
each arm are drawn from independent $\textsf{Bernoulli}(\theta_{a})$
distributions. Here, the parameters $\theta_{a}$ are drawn independently from
a known $\textsf{Beta}(\alpha,\beta)$ prior.  For each batch size
$b_{t}n = 100$ or $10,000$, we scale the prior parameters
$\alpha = \beta = b_{t}n$ to preserve the difficulty. Concretely, when
$b_{t}n = 100$, the prior mean for each parameter $\theta_{a}$ is 0.5 and the
prior standard deviation is $\approx 0.03$, so typically each treatment arm
has an average reward of $0.5 \pm 0.03$. We focus on small differences in
average rewards as this is often the case in real-world
experiments~\citep{KohaviLoSoHe09, KohaviDeFrLoWaXu12, KohaviDeFrWaXuPo13}.

Since we cannot vary the measurement noise
$s^{2}_a = \var(\epsilon_a)$~\eqref{eqn:rewards} under the
$\textsf{Beta-Bernoulli experiment}$, we also consider an alternative setting
which we call the $\textsf{Gamma-Gumbel experiment}$.  Here, rewards for each
arm are drawn from independent $\textsf{Gumbel}(\mu_{a}, \beta)$
distributions, with a known, fixed scale parameter $\beta$ that determines the
measurement variance $s^2_a = \frac{\pi^2}{6}\beta$.  The location parameters
$\mu_{a}$ are drawn independently from a known $\textsf{Gamma}(\alpha,\beta)$
prior, which determines the differences in the average reward between the
arms.  For each batch size $b_{t}n$, we set $\alpha = b_{t}n$ and
$\beta = 1/b_{t}n$ to preserve the difficulty of the problem as measured by
the gap between average rewards.

\begin{table}[t]
  \hspace{-1cm}
  \begin{tabular}{|c|c|c|c|c|c|c|c|c|}
  \hline 
  \multirow{3}{*}{Policy} & \multirow{3}{*}{Baseline} & \multicolumn{1}{c|}{Reward } & \multicolumn{1}{c|}{Number } & \multirow{2}{*}{Batch size} & \multicolumn{2}{c|}{Measurement } & \multicolumn{2}{c|}{Prior}\tabularnewline
   &  & Distribution & of Arms &  & \multicolumn{2}{c|}{Noise} & \multicolumn{1}{c}{} & \tabularnewline
  \cline{3-9} \cline{4-9} \cline{5-9} \cline{6-9} \cline{7-9} \cline{8-9} \cline{9-9} 
   &  & Bernoulli & 100 & 10,000 & 0.2 & 5 & Top-One & Descending\tabularnewline
  \hline 
  \hline 
  TS & 70.1 & 80.9 & 73.4 & 69.1 & 51.6 & 89.1 & 60.0 & 62.8\tabularnewline
  \hline 
  Top-Two TS & 68.7 & 82.1 & 74.1 & 71.8 & 46.8 & 88.9 & 59.3 & 65.0\tabularnewline
  \hline 
  TS$+$ & 70.9 & 82.4 & 76.6 & 81.4 & 63.3 & 83.3 & 63.0 & 67.9\tabularnewline
  \hline 
  Myopic & 66.5 & 76.0 & 64.3 & 69.1 & 49.7 & 80.2 & 60.0 & 63.6\tabularnewline
  \hline 
  Policy Gradient & 63.9 & 79.1 & 61.4 & \textbf{63.6} & 49.3 & 79.8 & 55.3 & 59.6\tabularnewline
  \hline 
  RHO & \textbf{58.8} & \textbf{73.1} & \textbf{55.8} & 64.2 & \textbf{41.5} & \textbf{78.6} & \textbf{52.6} & \textbf{56.6}\tabularnewline
  \hline 
  \end{tabular}

  \vspace{1em}

  \caption{\label{table:gse_compare} Comparison of empirical performance of
    Gaussian batch policies.  Simple regret of various policies as a
    percentage of the simple regret of the uniform allocation policy (lower is
    better).  Baseline column displays results for the
    $\textsf{Gamma-Gumbel experiment}$ with $K = 10$ arms, batch size
    $b_{t}n = 100$, $s^{2} = 1$ with a flat prior. The other columns display
    results for alternative settings, where one aspect of the baseline setting
    is changed (e.g. $K = 100$ arms instead of $K = 10$ arms) with all else
    remaining the same. Out of 640 problem instances we study, $\algo$
    achieves the best performance in 493 (77.0\%) settings and policy gradient
    in 76 (11.8\%) settings, among a selection of 10 policies including Gaussian policies in
    Section~\ref{section:algorithms-adp} and standard batch experimentation
    policies in Section~\ref{section:experiments}.}
  \end{table}

\paragraph{Algorithms} In addition to the policies mentioned in
Section~\ref{section:algorithms-mdp} that use the dynamic programming
structure of the model, we also implement policies such as Gaussian Thompson
Sampling which are natural adaptations of standard bandit policies to the
Gaussian sequential experiment. We summarize the the list of methods we
compare below.
\begin{itemize}[itemsep=0pt]
\item $\algofull$: Solves the planning problem \eqref{eqn:rho} as in
  Algorithm \ref{alg:rho}
\item Gaussian Limit Thompson Sampling: TS policy for the Gaussian sequential
  experiment
\item Gaussian Limit Top-Two Thompson Sampling: Top-Two TS \citep{Russo20} for the Gaussian
  sequential experiment.
\item Myopic: Maximizes one-step lookahead value function for the
  problem~\eqref{eqn:bayesian-dp}; a randomized version of the Knowledge
  Gradient method
\item TS$+$: Policy iteration on the approximate TS
  policy~\eqref{eqn:TS-plus}
\item Policy Gradient: Heuristically solves the dynamic
  program~\eqref{eqn:bayesian-dp} using a policy parameterized by a
  feed-forward neural network; trained through policy
  gradient~\eqref{eqn:policy-gradient} with episode length 5
\end{itemize}

\begin{figure}[t]
  \vspace{-1.6cm}
  \centering
  \hspace{-1.6cm}
  \subfloat[\centering Measurement variance $s_{a}^{2} = 1$.]{\label{fig:gse_bsr}{\includegraphics[height=5.2cm]{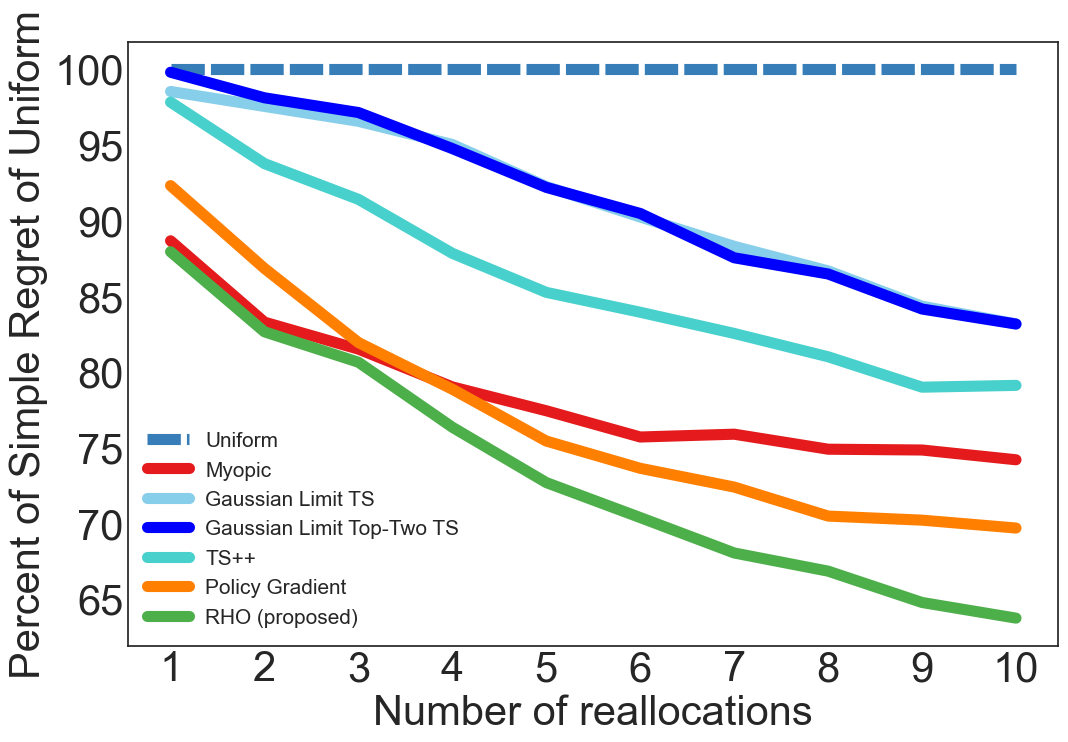} }}%
  \subfloat[\centering Measurement variance $s_{a}^{2} \in \{0.2, 1, 5\}$ for $T = 10$]{\label{fig:gse_var}{\includegraphics[height=5.2cm]{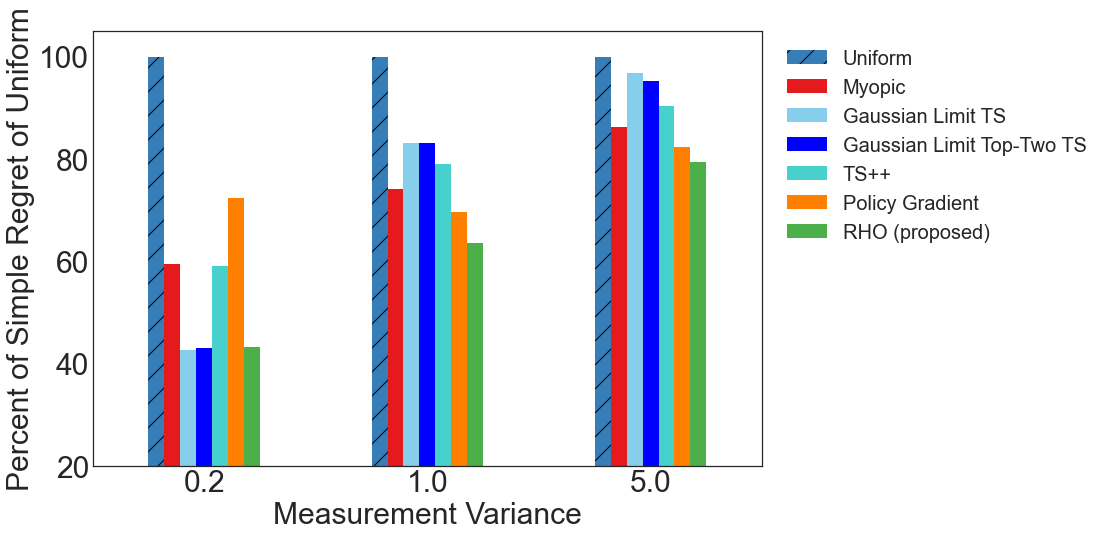}} }%
  \vspace{.4cm}
  \caption{\label{fig:gse_compare} Comparison of Gaussian batch policies.
    Relative gains over the uniform allocation as measured by the Bayes simple
    regret for the finite batch problem with $K = 100$ treatment arms and
    batches of size $b_{t} n = 10,000$.  $\textsf{Gamma-Gumbel experiment}$
    where individual rewards are Gumbel with a $\mbox{Gamma}$ prior. $\algo$
    maintains strong performance for small and long horizon experiments, as
    well as for low and high noise levels.  Myopic performs well in short
    horizon experiments, but worsens in longer horizon experiments due to
    insufficient exploration. The Thompson sampling policies are effective for
    low measurement noise but their performance degrades in underpowered
    settings. 
  }
\end{figure}

\paragraph{Overview of empirical results} We refer the reader to the
interactive plot~\eqref{eqn:streamlit} for a comprehensive presentation on our
benchmarking results.  In Figure~\ref{fig:gse_compare}, we provide
representative results where we consider $\numarm = 100$ arms and batches of
size $b_{t} n = 10,000$.  Although all of these policies use Gaussian batch
approximations, we evaluate them in a pre-limit problem with Gumbel
distributed rewards and a Gamma prior on the mean rewards, which we refer to
as the $\textsf{Gamma-Gumbel experiment}$.  Figure~\ref{fig:gse_bsr} focuses
on a fixed measurement variance $(s_{a}^{2} = 1)$ across reallocation epochs,
whereas Figure~\ref{fig:gse_var} considers different measurement noises for a
fixed number of reallocations $T = 10$. Finally, Table~\ref{table:gse_compare}
presents ablation studies over different reward distributions, number of arms,
batch size, measurement noise, and prior specifications. 

$\algo$ exhibits consistently strong performance across a wide array of
settings; other policies are effective in some instances and less effective in
others.  For example, the myopic policy performs well for short horizon
experiments as expected, but worsens in longer horizon settings due to
insufficient exploration.  Thompson sampling, which tends to explore more than
other policies, is an effective heuristic when the noise level is low, but
suffers in more underpowered experiments.
The policy gradient method achieves equivalent performance to $\algo$ when
there are a small number of alternatives ($\numarm = 10$) and measurement
noise is high.  TS$+$ improves upon Thompson Sampling under similar
conditions, but in general suffers from instabilities related to
optimization/training.  Compared to $\algo$, policy gradient/iteration methods
(heuristically) solve more sophisticated planning problems accounting for
future adaptivity. However, we observe empirically that they perform worse in
settings with more arms, lower measurement noise, and longer horizon
experiments.

\subsection{Discussion of $\algofull$}
\label{section:rho-analysis}

\begin{figure}[t]
  \vspace{-.4cm}
  \centering
  \includegraphics[height=3.8cm]{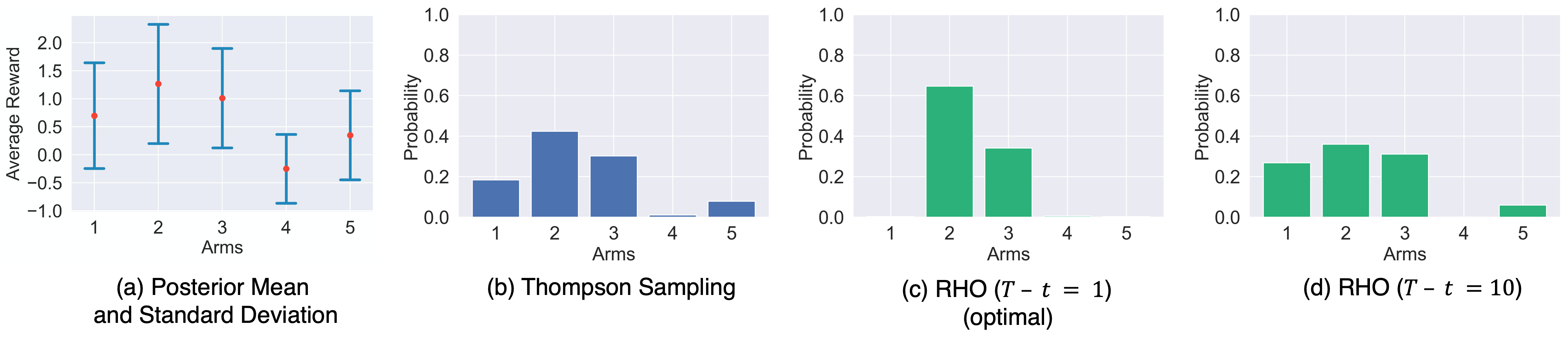}

  \vspace{.2cm}
  \caption{\label{fig:rho_explore} Comparison of sampling allocations.  
  (a) displays posterior means $\mu$ and standard deviations
    $\sigma$ (the length of each whisker) for $K = 5$ arms. The other graphs
    illustrate the sampling allocation computed by different policies given
    posterior belief $(\mu, \sigma)$. (b) shows the sampling
    allocation produced by the Gaussian Thompson sampling policy, (c)
    is $\algo$ when there is only 1 epoch left, and (d) is $\algo$ when
    the residual horizon is $T - t = 10$.  $\algo$ calibrates exploration to
    the length of the remaining horizon; when there is only 1 epoch left, the
    policy gives the Bayes optimal allocation focusing solely on the top two
    arms. When there are many reallocation epochs left, the policy explores
    other arms more and resembles the Thompson sampling allocation.  }
\end{figure}

The strong empirical performance of $\algofull$ impels us to carefully study
the advantages of solving the planning problem~\eqref{eqn:rho}.  In this
subsection, we show that $\algo$ enjoys several desirable properties. First
and foremost, implementing this policy only requires solving the above
optimization problem~\eqref{eqn:rho} at the currently observed state, allowing
it to remain adaptive while being computationally efficient; we use stochastic
gradient methods for this purpose. Second, the objective~\eqref{eqn:rho}
encourages the policy to explore more when the remaining horizon is long,
which leads it to explore aggressively early on in the experiment while
focusing on strong alternatives as the experiment winds down (see
Figure~\ref{fig:rho_explore}).

At each epoch $\algo$ iteratively computes the optimal allocation among
those that assume future allocations will only use currently available
information. It is thus guaranteed to outperform any open-loop allocation
including static designs; in particular, $\algo$ achieves smaller Bayes
simple regret than the uniform allocation. This gives a practical
performance guarantee even for small $T$, the regime which is relevant for
most real-world experimental settings, for which the uniform allocation is a
highly competitive policy.
\begin{proposition}
  \label{prop:rho-vs-static}
  Let $\rho_t$ be given by Algorithm~\ref{alg:rho}.  For any policy
  $\bar{\pi}_{t:T} = (\bar{\pi}_t(\mu_t, \sigma_t), \ldots,
  \bar{\pi}_{T-1}(\mu_t, \sigma_t))$ that only depends on currently available
  information $(\mu_t,\sigma_t)$, we have
  $V_{t}^{\rho}(\mu_t,\sigma_t) \geq V_{t}^{\bar{\pi}}(\mu_t,\sigma_t)$.
\end{proposition}
\noindent See Section~\ref{section:proof-rho-vs-static} for the proof.
In cases where it is computationally expensive to compute the allocation 
$\rho_t(\mu_t, \sigma_t)$ as the experiment progresses, one can learn the policy
offline via standard ML methodologies (see Section~\ref{section:distilled}).


Although we are primarily interested in short horizons $T$ in practice, we can
theoretically characterize the behavior of $\algo$ in the infinite horizon
limit $T - t\to \infty$. We show that the optimization problem~\eqref{eqn:rho}
becomes strongly concave for a large enough residual horizon $T-t$, which was
observed for the single-batch problem in~\cite{FrazierPo10}.  Then, by using
the KKT conditions of the optimization problem~\eqref{eqn:rho}, we also
characterize the asymptotic sampling policy $\rho_t$ converges to as $T-t$
grows large. We find that it converges to a novel posterior sampling policy we
denote as Density Thompson Sampling (DTS), as it samples arms proportionally
to the partial derivatives of the Gaussian Thompson Sampling policy.
\begin{theorem}
  \label{theorem:asymptotic-rho}
  Consider any fixed epoch $t$ and state $(\mu, \sigma)$.
  Let $\Delta_K^\epsilon = \Delta_K \cap \{p: p \ge \epsilon\}$ be the truncated simplex. 
  Suppose that as $T-t$ grows large, 
  the residual sample budget $\bar{b}_{t} = \sum_{s=t}^{T-1} b_{s}\to \infty$.
  \begin{enumerate}
  \item As $\bar{b}_{t}\to \infty$, the gradient scales as $\nabla_{\bar{\rho}} V^{\bar{\rho}}_{t}(\mu,\sigma)
     \sim c\frac{s_{a}^{2}}{\bar{b}_{t}\bar{\rho}_{a}^{2}}$, where $c$ depends only on $(\mu, \sigma)$.
  \item There exists $T_0$ such that $\forall T - t > T_0$,
  $\bar{\rho} \mapsto \bar{b}_{t}V^{\bar{\rho}}_{t}(\mu,\sigma)$ is strongly
  concave on $\Delta_K^\epsilon$.
  \item Suppose there exists
  $T_1$ such that for $\forall T - t > T_1$, the $\algo$ allocation satisfies
  $\rho_{t,a}(\mu,\sigma) > \epsilon$. Then as $T-t\to \infty$
  \begin{equation*}
    \rho_{t,a}(\mu,\sigma) \to \pi_{a}^{\textup{DTS}}(\mu, \sigma),~\mbox{where}~
    \pi^{\rm DTS}_a(\mu, \sigma) \propto s_a \left[ \frac{\partial}{\partial
        \mu_a}\pi_{a}^{\textup{TS}}(\mu,\sigma) \right]^{1/2}
  \end{equation*}
  and $\pi_{a}^{\text{TS}}(\mu, \sigma):= \P(a^{*} = a|\mu, \sigma)$ is the Thompson Sampling probability.
\end{enumerate}
\end{theorem}
\noindent See Section~\ref{section:proof-asymptotic-rho} for the proof. 

Theorem~\ref{theorem:asymptotic-rho} shows that the planning problem
becomes more amenable to optimization as the residual horizon increases,
although the gradient of the objective decays at rate $1/\bar{b}_{t}$ as
$T-t\to \infty$.  Moreover, although we would expect the open-loop planning
problem to be better calibrated under short time horizons, even in the
infinite horizon limit it gives rise to a natural, stationary posterior
sampling policy, establishing a novel connection between posterior sampling
and model predictive control.  The limiting policy DTS is highly exploring
and does not over-sample the best arm, which makes it better suited for
best-arm identification tasks than standard Gaussian Thompson Sampling (see
Section~\ref{section:proof-log-dts} for more details).

In contrast, we empirically observe that the more sophisticated approaches
, i.e. policy gradient,  suffer
optimization difficulties due to vanishing gradients. This is alluded to in
Theorem~\ref{theorem:asymptotic-rho}: the gradient of the static allocation
value function scales as $(s_{a}^{2}/\bar{b}_{t})$ and thus vanishes under
small measurement variance and long residual horizon.

\section{Comparison with standard bandit approaches}
\label{section:experiments}


Our empirical and theoretical analysis in the previous section shows that
$\algo$ is an exceedingly simple yet effective heuristic. Impelled by these
benefits, we now provide a careful empirical comparison between $\algo$ and
other standard multi-armed bandit algorithms. Overall, we find that although
$\algo$ relies on Gaussian approximations for the rewards and the prior
distribution, performance gains from planning with the Gaussian sequential
experiment greatly outweigh approximation errors. In total, we find that
across 640 settings, $\algo$ outperforms Gaussian policies discussed in Section~\ref{section:algorithms-adp} and
standard batch bandit policies in 493 (77.0\%) of them.
Our empirical results
suggest that $\algo$ particularly excels in settings that are difficult for standard
adaptive policies, including those a limited number of reallocation epochs,
unknown reward distributions, low signal-to-noise ratios, and a large number
of treatment arms.  Its versatile performance across different horizon and
noise levels shows the utility of calibrating exploration using the Gaussian
sequential experiment.  In addition to the key dimensions we highlight in
this section, we provide a comprehensive set of experiments in the Streamlit
app~\eqref{eqn:streamlit}.

\begin{figure}[t]

    \centering
    \hspace{-.9cm}
      \subfloat[\centering Number of arms $\numarm = 10$.]{\label{fig:num_arm_bsr_10}{\includegraphics[height=5cm]{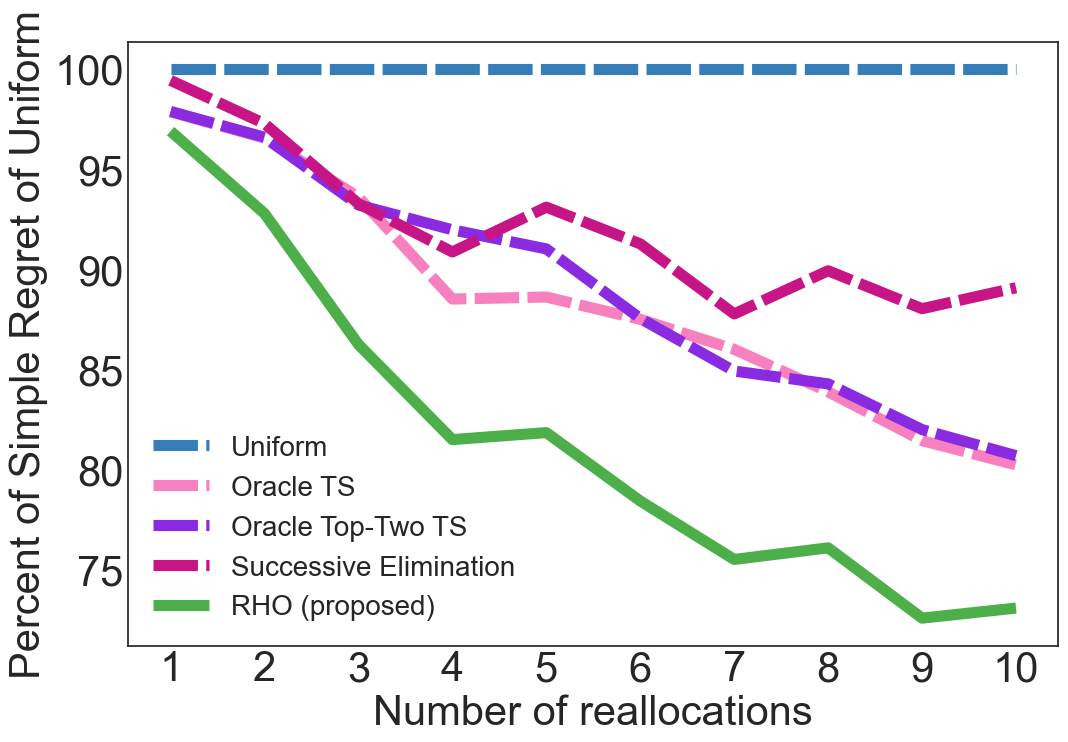}} }%
      \hspace{.5cm}
      \subfloat[\centering Number of arms $\numarm = 100$.]{\label{fig:num_arm_bsr_100}{\includegraphics[height=5cm]{./fig/reg_Bernoulli_100_100_1_Flat.png} }}%
      \vspace{.4cm}
      \caption{\label{fig:num_arm_bsr} Comparison of performance across
        different number of treatment arms.  Relative gains over the uniform
        allocation as measured by the Bayes simple regret for the finite batch
        problem with a batch size of $b_{t}n = 100$.
        $\textsf{Beta-Bernoulli experiment}$ where individual rewards are
        Bernoulli with a $\mbox{Beta}$ prior.  
          }
    
  \end{figure}

\paragraph{Algorithms} To benchmark performance, we consider a suite of standard batch bandit
policies proposed in the literature.
\begin{itemize}[itemsep=-2pt, leftmargin=.7cm]
  \item Uniform: For each batch, assign samples uniformly across treatment arms.
  \item Successive Elimination \citep{EvenDarMaMa06, GaoHaReZh19,
      PerchetRiChSn16}: For each batch, assign samples uniformly across
    treatment arms that are not eliminated. Eliminate all arms whose upper
    confidence bound is less than the lower confidence bound of some other
    arm. Given $\numarm$ arms, $n_{a}$ samples drawn for arm $a$, measurement
    variance $s_{a}^{2}$, the confidence interval is
    $\hat{\mu}_{a} \pm \beta_{a}(n_{a})$, where $\hat{\mu}_{a}$ is the
    empirical mean reward of arm $a$ and the width of the confidence bound
    $\beta_{a}(n_{a})$ is
  $$ \beta_{a}(n_{a}) = cs_{a} \sqrt{\frac{\log(n_{a}^{2}K/\delta)}{n_{a}}},$$
  where $c,\delta$ are constants that are chosen via grid search to minimize regret in each instance.
  \item Oracle Thompson Sampling \citep{KalkanhOz21, KandasamyKrScPo18}: Beta-Bernoulli TS with batch updates.
  \item Oracle Top-Two Thompson Sampling \citep{Russo20}: Beta-Bernoulli Top-Two TS with batch updates.
  \item $\algofull$: Selects the allocation by solving the planning problem \eqref{eqn:rho}, as in Algorithm \ref{alg:rho}.
\end{itemize}

For the $\textsf{Gamma-Gumbel experiment}$, the Gumbel distribution does not
have a known conjugate prior distribution and updating the posterior after
each batch is more involved. For this reason, when in this environment we
restrict our focus to successive elimination as the main baseline. This
captures settings in which the reward model is unknown to the experimenter,
where it is difficult to use Thompson sampling or other Bayesian policies that
require the exact reward distribution.

\begin{figure}[t]
  \centering
  \includegraphics[height=5cm]{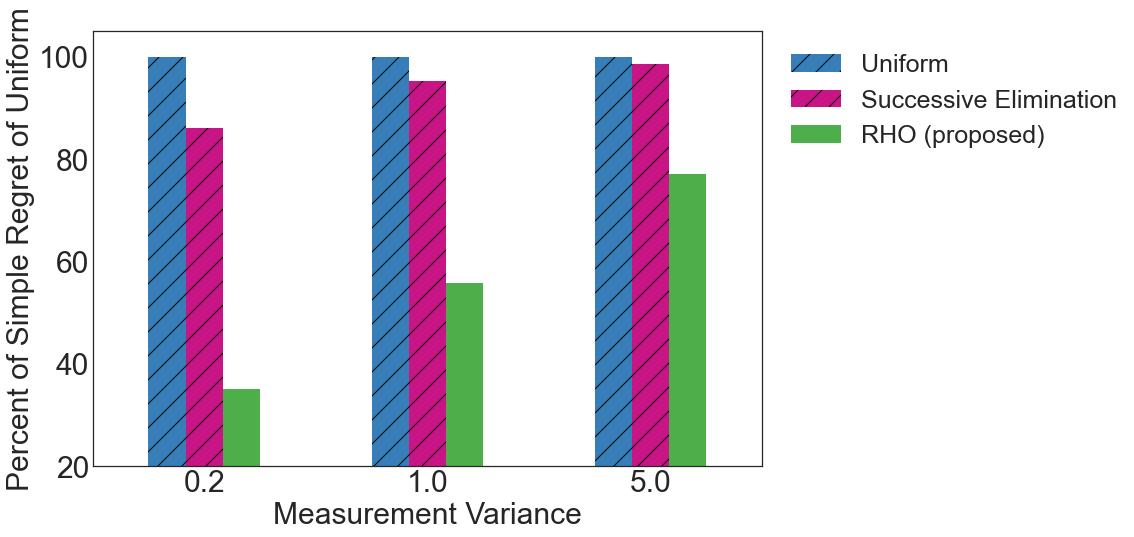}
    \caption{\label{fig:bar_var} Comparison of performance across different
      measurement noise levels.  Relative gains over the uniform allocation as
      measured by the Bayes simple regret. $\textsf{Gamma-Gumbel experiment}$
      where rewards follow Gumbel distributions with measurement variances
      $s_{a}^{2} \in \{0.2, 1, 5 \}$.  There are $K = 100$ treatment arms with
      $T = 10$ batches of size $b_{t}n = 100$.  
    }
\end{figure}

\paragraph{Number of alternative $\numarm$ (treatment arms)}
In order to study how performance changes with the number of alternatives, we
fix the batch size to be $b_{t}n = 100$ and evaluate our policies for
$\numarm = 10$ and $\numarm = 100$ alternatives. Figure \ref{fig:num_arm_bsr}
shows that $\algo$ achieves strong performance for large and small number of
alternatives alike, and the performance gains are larger when there is a large
number of arms.  
In these experiments and others presented in the
interactive app~\eqref{eqn:streamlit}, we observe that the performance gains
of $\algo$ persist across all time horizons, particularly for short
horizons. 

\paragraph{Measurement variance $s_a^2$}
Next, we study how the signal-to-noise of the problem instance affects the
performance of adaptive policies. To study a scenario where we can flexibly
control different measurement variances
$s^{2}_a = \var \epsilon_a$~\eqref{eqn:rewards}, we move away from the
$\textsf{Beta-Bernoulli experiment}$ and instead consider the
$\textsf{Gamma-Gumbel experiment}$. In Figure \ref{fig:bar_var}, we observe
that $\algo$ outperforms uniform allocation and successive elimination in both
high and low signal-to-noise regimes. When the signal-to-noise ratio is high
(i.e., measurement noise $s_a^2$ is small), $\algo$ is able to rapidly hone in
on the highest reward treatment arms. Even when the signal-to-noise is low,
$\algo$ is able to make substantial progress over uniform while other adaptive
policies struggle to learn enough to improve sampling efficiency.  


\begin{figure}[t]
  \centering
  \hspace{-.9cm}
  \subfloat[\centering Batch size $b_{t}n = 100$.]{\label{fig:gumbel_batch_bsr_10}{\includegraphics[height=5cm]{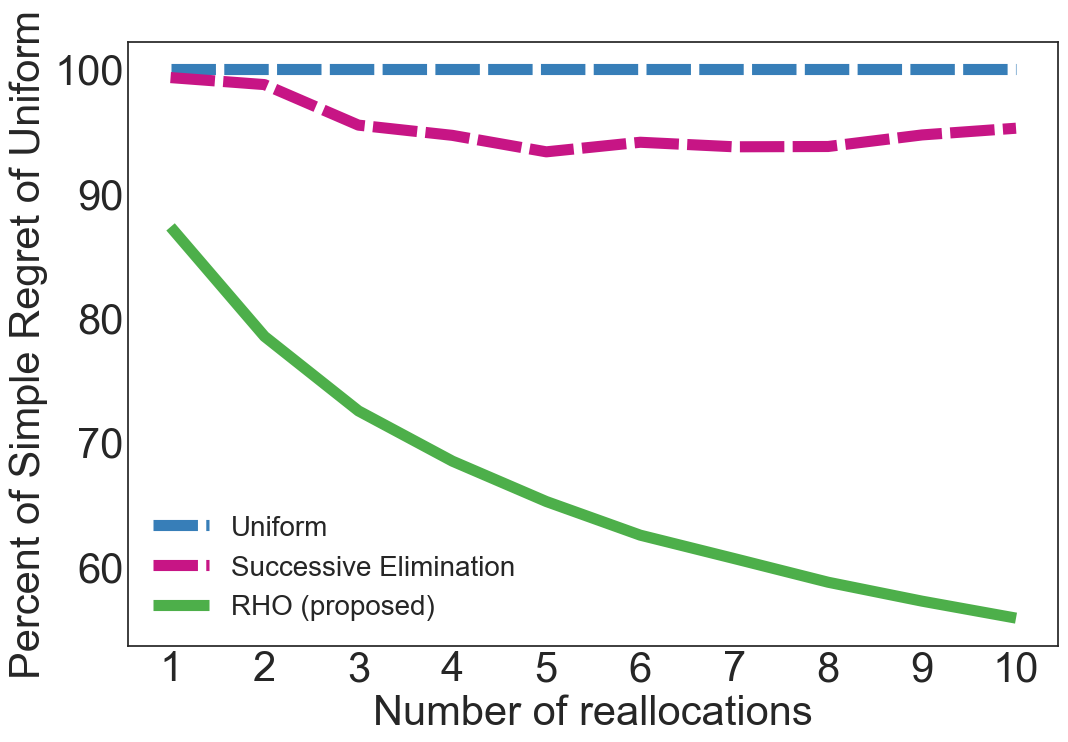}} }%
  \hspace{.5cm}
  \subfloat[\centering Batch size $b_{t}n = 10,000$.]{\label{fig:gumbel_batch_bsr_100}{\includegraphics[height=5cm]{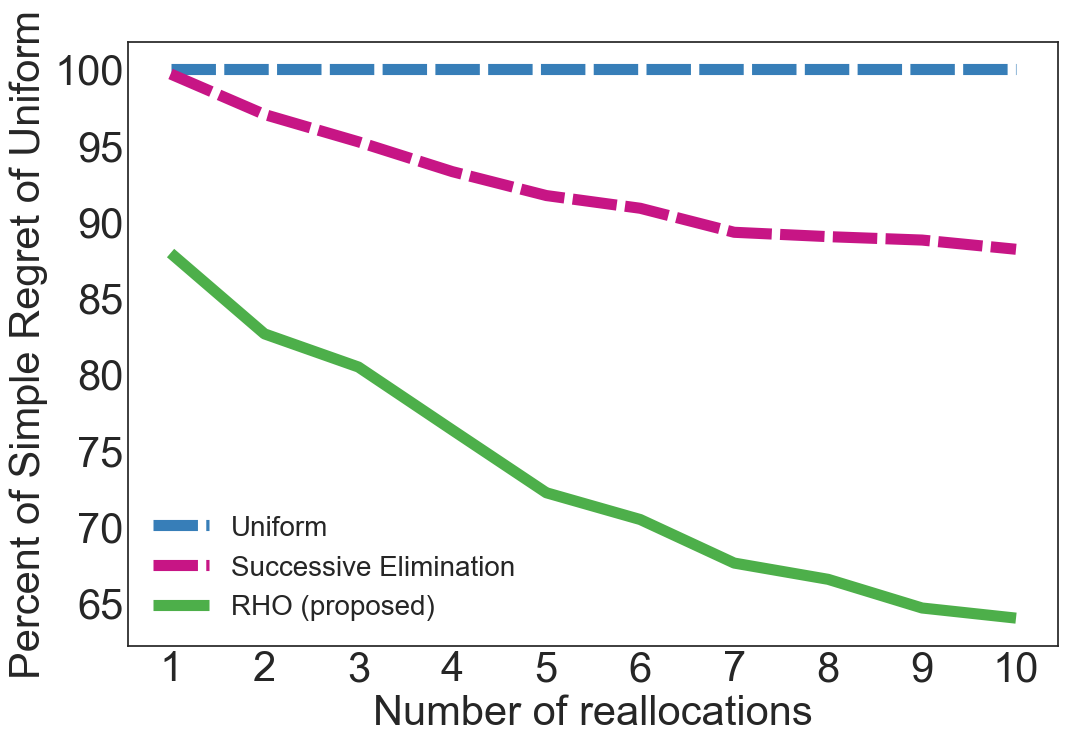} }}%
  \vspace{.4cm}
  \caption{\label{fig:gumbel_batch_bsr} Comparison of performance across
    batch sizes.  Relative gains over the uniform allocation as measured
    by the Bayes simple regret for the finite batch problem with $K = 100$
    treatment arms. $\textsf{Gamma-Gumbel experiment}$ where individual
    rewards are Gumbel with a $\mbox{Gamma}$ prior. 
    }    
\end{figure}

\paragraph{Batch size $b_t n$}
In Figure~\ref{fig:bern-bsr-batch} we presented in
Section~\ref{section:introduction}, we compared the regret incurred by
different policies for large and small batches in the
$\textsf{Beta-Bernoulli experiment}$.  We use the
$\textsf{Gamma-Gumbel experiment}$ as an additional test of whether the
Gaussian approximations hold for batched rewards. Although the Gumbel
distribution has a high excess kurtosis, we still observe that the Gaussian
sequential experiment serves as a useful approximation even for small
batches. In Figure \ref{fig:gumbel_batch_bsr}, we see that $\algo$ greatly
outperforms uniform allocation and successive elimination for small and large
batches alike.  In particular, successive elimination especially struggles
when the batch size is small ($b_{t}n = 100$).

\ifdefined\msom
\else

\begin{figure}[t]
  \centering
  \includegraphics[height=5cm]{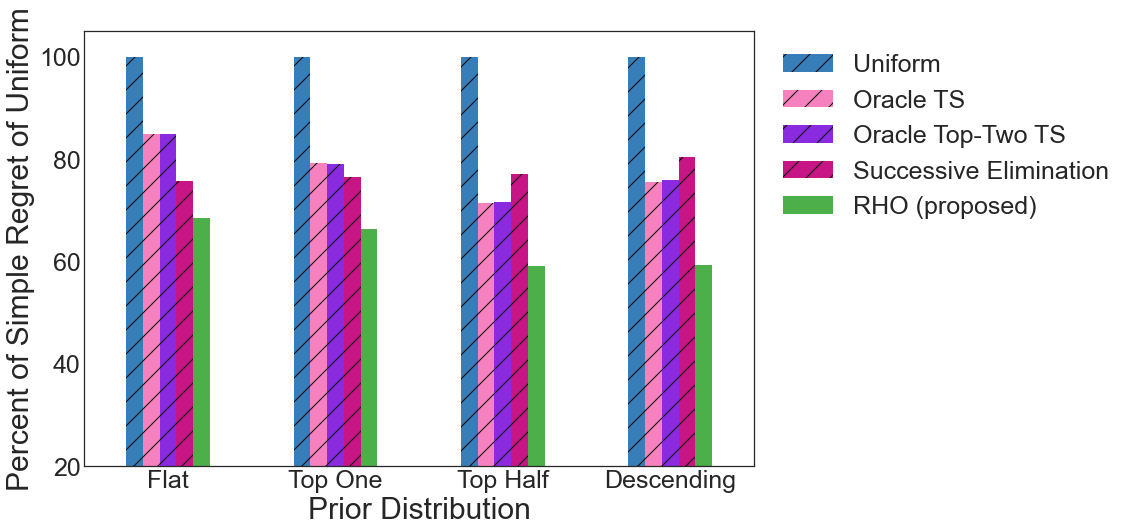}
    \vspace{.4cm}
    \caption{\label{fig:bar_prior} Comparison of performance across  different
      priors.  Relative gains over the uniform allocation as measured by the
      Bayes simple regret for the finite batch problem for $\numarm = 100$
      arms and $T = 10$ batches of size $b_{t}n = 100$ samples.
      $\textsf{Beta-Bernoulli experiment}$ where individual rewards are
      Bernoulli with various $\mbox{Beta}$ priors.  
    }
  
\end{figure}

\paragraph{Non-uniform prior distributions}

The previous numerical experiments focused on the case in which the
experimenter's prior distribution is identical across treatment arms. Yet,
modern experimental platforms typically run many experiments with similar
treatments, so the experimenter may have different prior beliefs across
treatments. These beliefs can be incorporated into the experimental design to
improve sampling efficiency.  To study the effect of non-uniform priors, we
consider the $\textsf{Beta-Bernoulli experiment}$ under different Beta priors
for each arm. Fixing the batch size to be $b_{t}n = 100$ and the number of
arms to be $ \numarm = 100$, we consider the following prior distributions.
\begin{itemize}[itemsep=-2pt, leftmargin=.7cm]
  \item Flat: All arms have an identical prior of Beta$(100,100)$.
  \item Top One: There is a single treatment with a Beta$(110,100)$ prior such
    that it has a higher prior mean than other arms. All other arms have
    Beta$(100,100)$ priors.
  \item Top Half: Half of the arms have Beta$(110,100)$ priors so that they
    have a higher prior mean than the rest which have Beta$(100,100)$ priors.
  \item Descending: The first arm has the highest prior mean (Beta$(100,100)$)
    and the prior means decrease for each arm $i \in [\numarm]$. Each arm has
    a Beta$(100 - (i - 1),100)$ prior.
\end{itemize}

Figure \ref{fig:bar_prior} compares the performance of the sampling policies
across different prior distributions. Unsurprisingly, policies which use prior
information outperform non-Bayesian policies when the prior is more
informative. We observe that despite using a Gaussian prior to approximate the
true prior distribution, $\algo$ obtains significantly larger performance
gains compared to Thompson sampling policies that use the true prior.  
\section{Discussion}
\label{section:discussion}

In this work, we use Gaussian approximations to derive an MDP framework that
describes how posterior beliefs evolve under the experimenter's sampling
policy in batch experiments.  With tools from optimal control and
reinforcement learning, this framework guides the development of policies that
can plan ahead and calibrate exploration to the size of the remaining sampling
budget.

\subsection{Related work}
\label{section:related-work}

This paper is situated in the extensive body of work studying adaptive methods
for instance-dependent pure-exploration problems. Since this literature spans
multiple communities including operations research, statistics, and machine
learning, we provide a necessarily abridged review centered around the
following characteristic aspects of our approach.
\begin{enumerate}[itemsep=-4pt, leftmargin=.7cm]
\item Our proposed adaptive framework focuses on a small number of
  reallocation epochs (horizon $T$). The algorithms we derive optimize
  instance-dependent factors that do not grow with $T$, which are often
  ignored as ``constants'' in the literature.
  \label{item:small-horizon}
\item Using sequential Gaussian approximations, our algorithms can handle
  flexible batch sizes.
  \label{item:batch}
\item The Gaussian sequential experiment we derive is based on the admissible
  regime where the gaps in average rewards scale as $\Theta(1/\sqrt{n})$.
  \label{item:diffusion}
\item When endowed with a Gaussian prior over average rewards, the Gaussian
  sequential experiment gives rise to a Markov decision process with
  fractional measurement allocations. Our MDP is connected to previous works
  in simulation optimization that study problems with Gaussian observations,
  as well as (adaptive) Bayesian experimental design methods.
  \label{item:gaussian}
\end{enumerate}

\paragraph{Pure-exploration MABs and ranking \& selection}

The problem of identifying the top-$k$ arms out of a set of alternatives has
been extensively studied in the simulation optimization and multi-armed bandit
literatures. As a detailed review of this vast literature is outside the scope
of the work, we refer the reader
to~\citet{HongNeXu15},~\citet{ChenChLePu15},~\citet[Ch
33.5]{LattimoreSz19},~\citet[Section 1.2]{Russo20}, and references therein for
a complete review. In light of the small-horizon perspective
(Point~\ref{item:small-horizon}), we discuss these works by dividing them into
two broad categories.

The first category studies fully sequential procedures and study the behavior
of the problem when the number of reallocation epochs (horizon $T$) is
expected to be large. Several authors explicitly study the limiting regime
$T \to \infty$~\cite{Chernoff59, Chernoff73, ChenLiYuCh00, GlynnJu04,
  Russo20}, though the multi-armed bandit (MAB) literature primarily focuses
on finite horizon guarantees. These works are often classified into two
regimes: in the fixed-confidence setting \cite{MannorTs04,EvenDarMaMa06,
  KarninKoSo13, JamiesonMaNoBu14, JamiesonNo14, KaufmannKa13, GarivierKa16,
  KaufmannCaGa16}, the objective is to find the best arm with a fixed level of
confidence using the smallest number of samples, where as the fixed budget
setting \cite{BubeckMuSt09, AudibertBuMu10,
  GabillonGhLa12,KarninKoSo13,CarpentierLo16} aims to find the best arm under
a fixed sampling budget.  Our setting can be viewed as a Bayesian fixed-budget
problem where the sampling budget is split into fixed batches. The main
conceptual difference between this paper and the above body of work is that we
explicitly optimize performance for each problem instance and a fixed horizon,
where ``constant terms'' that do not depend on the horizon $T$ play an outsize
role in determining statistical performance.

From an algorithmic viewpoint, the majority of the works in this category
study sequential elimination schemes or modifications of the upper confidence
bound algorithm (UCB). Our Bayesian adaptive algorithms are mostly closely
related to Bayesian algorithms for best-arm
identification~\cite{QinKlRu17,RussoVaKaOsWe18, ShangHeMeKaVa20,
  Russo20, KasySa21}. When viewed as algorithms for the original pre-limit problem,
these algorithms require a full distributional model of individual rewards, in
contrast to our framework that only requires priors over average rewards as we
belabored in Section~\ref{section:algorithms}. Empirically, we observe in
Figure~\ref{fig:bern-bsr-batch} that a simple algorithm derived from the
\emph{limiting Gaussian sequential experiment} significantly outperforms
these oracle Bayesian sampling methods and these gains hold even when the batch size
is small.

Most of the aforementioned works study the probability of correct selection,
with notable recent exceptions that consider the Bayes simple
regret~\cite{KomiyamaArKaQi21, AtsidakouKaSaKv22}. In this paper, we focus on
the Bayes simple regret for ease of exposition, and as we believe this is the
more relevant objective in practice. However, our asymptotic formulation can
incorporate alternative objectives such as the probability of correct
selection, and the adaptive algorithms we develop can be appropriately
modified.

The second category of works develops one-step or two-step heuristics for
selecting the next alternative, such as Knowledge Gradient or expected
opportunity cost~\cite{GuptaMi96, FrazierPoDa08, ChickBrSc10, ChickIn01,
  HeChCh07} or probability of correct selection~\cite{ChickBrSc10}. Of
particular relevance to our work is \citet{FrazierPo10}, who consider the
problem of allocating a \emph{single} batch of samples across several
alternatives with normally distributed rewards to minimize the Bayes simple
regret. Their setting is equivalent to a single-stage version of the Gaussian
sequential experiment we consider. Empirically, we observe in
Section~\ref{section:experiments} that our proposed planning-based policies
can perform much better than these single-batch or one-step lookahead
procedures.

\paragraph{Batching and delayed feedback}

From the perspective of handling flexible batch sizes
(Point~\ref{item:batch}), several authors have adapted bandit algorithms to
problems with batch evaluations and delayed feedback. In the standard
exploration-exploitation MAB setting with fixed batch sizes,
\citet{PerchetRiChSn16, GaoHaReZh19} find that even when the number of
reallocation epochs scales logarithmically or even sub-lograthmically in the
total sample size, one can recover the same rates for cumulative regret as in
the fully-sequential setting.~\citet{EsfandiariKaMeMi21} show these rates can
be made tighter under adaptive batch sizes, and~\citet{KalkanhOz21, KarbasiMiSh21} show that
Thompson sampling with an adaptive batch scheme also achieves rates equivalent
to the fully sequential case up to similar logarithmic terms.  The growing
literature on bandits with delayed feedback can also be seen as a bandit
problem with stochastic batch sizes \cite{JoulaniGySz13, GroverEtAl18,
  Pike-BurkeAgSzGr18, VernadeCaLaZaErBr20}.  

We focus on pure-exploration when the number of reallocation epochs is small
compared to the total sample size.  In this regard, \citet{JunJaNoZh16} propose
fixed confidence and fixed budget policies that obtain the top-k arms and find
that the batch complexity scales logarithmically in the total sample size.
\citet{AgarwalEtAl17} show that a sequential elimination policy obtains the
top-k arms with fixed confidence with a number of reallocation epochs that
only grows iterated logarithmically in the total sample size. Most of the
policies proposed in the batched bandit literature are either sequential
elimination based policies, or batch variants of Thompson Sampling that know
the true reward distribution. Unlike these works, we study a limiting regime
and propose Bayesian algorithms that use Gaussian approximations for the
likelihoods.  As we further detail in the next discussion point, the policies
we derive for the Gaussian sequential experiment can be directly applied to
the original pre-limit problem, and can thus handle any batch size
flexibly. Empirically, our methods outperform standard batch methods across a
range of settings, as we outline in Figure~\ref{fig:bern-bsr-batch} and
further expand in Section~\ref{section:experiments}.


\paragraph{Gaussian approximations and asymptotics}

\citet{BhatFaMoSi20} study optimal experimental design for binary treatments,
under a linear specification for treatments and covariates. Their approach has
philosophical similarities to ours, which optimizes a final stage outcome by
using DP methods to sequentially assign treatments to subjects.  They analyze
an fully online version of the problem where subjects arrive stochastically
and the experimenter sequentially allocates them to the treatment or control
group. They balance covariates in the two groups to minimize the variance of
the OLS estimator at the end of the experiment; when covariates are drawn
i.i.d.  from an elliptical distribution, they show the dimensionality of the
DP collapses. Recently,~\citet{XiongAtBaImb23} consider a similar problem for
panel data, maximizing the precision of a generalized least squares (GLS)
estimator of the treatment effect.
On the other hand, our DP crystallizes how the adaptive experimentation
problem simplifies in the large-batch approximation; our formulation relies on
Bayesian posterior beliefs over \emph{average rewards}, rather than an
estimate of the uncertainty. We consider multiple treatment arms where an
effective adaptive design must sample strategically to reduce uncertainty for
treatments which are likely to be the best one.

Our work is also connected to recent works that use Gaussian approximations
for statistical inference on data produced by bandit
policies~\cite{HadadHiZhAt21, LuedtkeVa16, LuedtkeVa18}. Our work is most
related to~\citet{ZhangJaMu20} who derive inferential procedures based on
Gaussian approximations in a batched bandit setting.  They construct a
modified ordinary least squares (OLS) estimator and show asymptotic normality
as the batch size grows large. Although our setting does not feature any
contextual information, we similarly consider a fixed horizon experiment with
batches and our main asymptotic result considers the regime where the batch
size grows large while the horizon remains fixed. In this regard, we show that
asymptotic approximations used for inference are also highly effective in
guiding adaptive experimentation procedures. However, unlike their work, we
are able to derive our result without assuming that the propensity scores are
bounded from below or clipped.

Concurrent and independent to the present paper,~\citet{HiranoPo23} consider
data produced from a parametric model and derive Gaussian approximations for
batched data produced by adaptive sampling policies. Their work extends the
classical framework of local asymptotic normality of parametric
models~\cite{VanDerVaart98} to adaptive data.  On the other hand, we do not
assume a parametric model over the reward distribution and instead derive
Gaussian limits over \emph{average rewards}. The two approaches are
complementary: we focus on deriving a formulation on which modern
computational tools can be used to derive new Bayesian batch adaptive sampling
policies, whereas \citet{HiranoPo23} study power calculations and inferential
questions.

There is a nascent literature on deriving diffusion limits for bandit
policies~\cite{KalvitZe21, WagerXu21, FanGl21, Adusumilli21,
  AramanCa22}. These works typically let the number of reallocation epochs
grow to infinity, while scaling down the gaps in average rewards like in our
setting (Point~\ref{item:diffusion}). \citet{WagerXu21} consider any Markovian
policy, and~\citet{FanGl21} study Thompson sampling with extensions to batched
settings where the batch size is small compared to the number of reallocation
epochs. \citet{KalvitZe21} find a sharp discrepancy between the behavior of
UCB and Thompson Sampling in a related regime. \citet{Adusumilli21} considers
an arbitrary Bayesian policy and derive an HJB equation for the optimal policy
in the limit, which can be solved by PDE methods. \citet{AramanCa22} consider
a sequential testing environment with two hypotheses involving experiments
that arrive according to an exogenous Poisson process.  They derive diffusion
process approximations in the limit as intensity of arrivals grows large and
the informativeness of experiments shrinks.

In contrast, our diffusion limit is motivated by the practical difficulty of
reallocating sampling effort and we consider a \emph{fixed} number of
reallocation epochs. Our formulation is derived by letting the batch size $n$
grow large, while scaling down the difference in average rewards as
$\Theta(1/\sqrt{n})$.  This setting is particularly appropriate for modeling
experimentation on online platforms that deal with many units and
interventions with small treatment effects.  Unlike the above line of work
that study limiting diffusion processes, our limiting process is a discrete
time MDP which allows us to use approximate DP methods to develop new adaptive
experimentation policies.

\paragraph{Gaussian observations and priors}

Several research communities have taken interest in sequential decision-making
problems with Gaussian observations and priors due to its tractability. While
the algorithmic connections are salient as we discuss shortly, these
literatures do not provide formal justifications of normality unlike the
present work, although we have found some nice heuristic discussions (e.g.,
see~\citet[Section 3.1]{KimNe07}).  Another key difference between our
asymptotic framework and others is that we consider normal distributions over
the gaps between \emph{average rewards}, rather than individual rewards
(Point~\ref{item:gaussian}).

Taking a decision-theoretic framework in problems with Gaussian observations,
a number of authors~\cite{GuptaMi96, ChickBrSc10, FrazierPo10} have studied
simple settings that yield an exact solution, e.g., when there is a single
measurement. While the formal settings are different, the spirit of our
algorithmic development is similar to these approaches which heuristically
extend the solution derived in the simple setting to build adaptive policies.
Our use of the Gaussian sequential experiment to improve statistical power
appears new; a similar Gaussian MDP problem has been recently studied in the
context of robust control~\cite{MullerVaPrRo17} and attention
allocation~\cite{LiangMuSy22}.  Most recently, \citet{LiuDeVaXu22} provide
bounds on the cumulative regret of a Bayesian agent that maintains a
misspecified Gaussian posterior state in a (fully sequential) Bernoulli reward
environment.

The Bayesian optimization literature uses Gaussian processes to model the
world. Many authors have proposed acquisition functions for batch
evaluations~\cite{GinsbourgerRiCa08, ShahGh15, WangGeKoJe18, GonzalezDaHeLa16,
  WuFr19} and some derive regret bounds for batched versions UCB and Thompson
sampling in Gaussian process environments\cite{ContalBuRoVa13,
  DesautelsKrBu14, KandasamyKrScPo18}. The primary focus of this literature is
problems with a continuous set of arms. In contrast, we focus on the
allocation of measurement effort over a finite number of arms under limited
statistical power, and our setting is characterized by limited extrapolation
between arms and a fixed, finite exploration horizon. Our approach of using a
MDP for posterior beliefs to design non-myopic adaptive experimentation
policies conceptually builds on works that design non-myopic acquisition
functions for Bayesian optimization~\cite{GonzalezOsLa16, JiangChGoGa20,
  LamWiWo16, JiangJiBaKaGaGa19,AstudilloJiBaBaFr21}.

As our methods maximize an expected utility function, it is conceptually
related to Bayesian experimental design methods~\cite{ChalonerVe95,
  RyanDrMcPe16, FosterIvMaRa21}. Instead of optimizing expected information
gain (EIG), we minimize expected simple regret at the terminal period, a more
tractable objective. In particular, we consider policy gradient based methods
to guide the experimental design similar to the works~\cite{FosterIvMaRa21,
  JorkeLeBr22}. Our work is also similar to~\citep{MinMoRu2020}, who propose a policy gradient algorithm
  for shaping the posterior used by Thompson Sampling as a computational tool to improve its performance. But unlike these works, the methods we consider at the end of
Section~\ref{section:algorithms-mdp} use \emph{pathwise} policy gradients of the
value function enabled by the 
reparameterization trick for Gaussian random vectors~\cite{KingmaWe14}.
This is in contrast with the score function or REINFORCE gradient estimator~\cite{Williams92}
commonly used in reinforcement learning, which is known to have higher variance
and is difficult to apply in continuous state and action settings.
Our method allows us to use standard auto-differentiation frameworks to compute exact 
gradients of the value function without having to
fit the value function or the model separately, as is required for
other RL methods such as Deep Deterministic Policy Gradient~\cite{LillicrapEtAl16}. 
While the reparameterization trick
has been used before to maximize one-step or two-step 
acquisition functions in Bayesian optimization
(e.g. ~\cite{WilsonHuDe18, WuFr19, BalandatEtAl20}), our work differs in that we compute pathwise gradients of entire
sample paths with respect to sampling \emph{allocations},
and use these gradients to optimize sampling policies.



\paragraph{Acknowledgement} We are indebted to Omar Besbes, Lin Fan and Daniel
Russo for constructive feedback and helpful discussions. This research was
partially supported by the Digital Futures Initiative.



\bibliographystyle{abbrvnat}

\ifdefined\useorstyle
\setlength{\bibsep}{.0em}
\else
\setlength{\bibsep}{.7em}
\fi

\bibliography{../bib}

\ifdefined\useorstyle

\ECSwitch


\ECHead{Appendix}

\else
\newpage
\appendix

\fi

\section{Further Experimental Results}

\subsection{Unknown measurement variance}
\label{section:unknown-s2}

We consider the case where the measurement variance is unknown to the experimenter. 
We fix $K = 100$ and $n = 10,000$ in the Gamma-Gumbel experiment. The experimenter believes
the measurement variance for all arms is identically equal to $s^{2} = 1$, but the 
true measurement variance is $s^{2} = Y_{i}$ where $Y_{i} \sim \mathsf{Lognormal}(0, \varsigma^{2})$
for $\varsigma \in \{0.25, 0.5, 1.0\}$. The 50\% confidence intervals for $Y_{i}$ under these distributions are
\begin{align*}
    \varsigma = 0.25 &: 50\% \text{ CI } [0.84, 1.18] \\
    \varsigma = 0.50 &: 50\% \text{ CI }[0.75, 1.40] \\
    \varsigma = 1.00 &: 50\% \text{ CI }[0.50, 1.96]
\end{align*}
Intuitively, when $\varsigma = 1.0$, the actual measurement variance can easily range from 50\% to 196\% of the variance 
assumed by the experimenter. We observe in Figure \ref{fig:var_perturb} that despite large mismatch between
the experimenter's belief of the measurement variance and the true measurement variance, 
$\algo$ is still able to retain much of its performance benefits over other policies. This illustrates that the method
is rather robust to estimation errors of the measurement variance.

\begin{figure}[t]
     \centering
     \subfloat{\includegraphics[height=5.7cm]{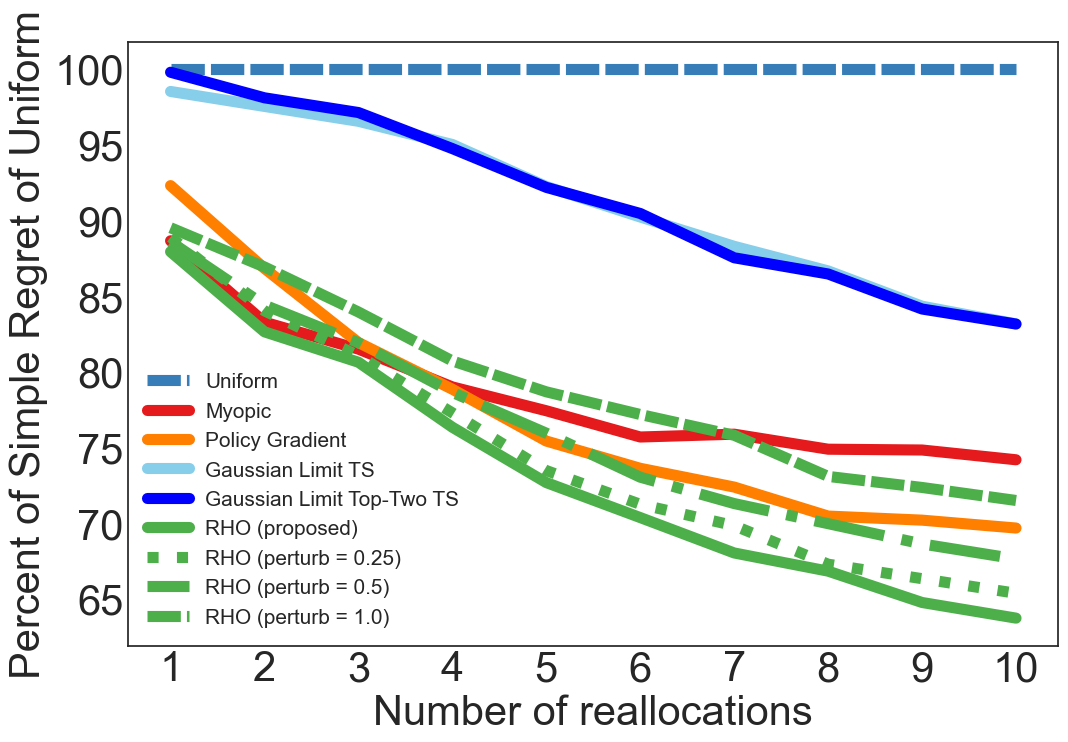}} %
     \caption{\label{fig:var_perturb} Comparison of performance across
       variance perturbations.  Relative gains over the uniform allocation as measured
       by the Bayes simple regret for the finite batch problem with $K = 100$
       treatment arms, batch size $n = 10,000$ in the $\textsf{Gamma-Gumbel experiment}$.
       Despite drastic variance mis-measurement ($\varsigma = 1$), $\algo$ maintains performance
       benefits over other Bayesian policies (except Policy Gradient).}    
   \end{figure}
\section{Derivations of Gaussian sequential experiment}

\subsection{Proof of Theorem \ref{theorem:limit}}
\label{section:proof-limit}

Let $s_{t, a}^{2} \defeq s_a^{2}/b_{t}$ be the rescaled measurement variance.
Recall the Gaussian sequential experiment given in Definition~\ref{def:gse}:
for all $0< t \le T-1$
\begin{equation*}
  G_{0}\sim N(h\pi_{0},\Sigma_{0}),
  ~~G_{t}|G_{t-1},\ldots,G_{0} \sim N(h\pi_{t}(G_{1},...,G_{t-1}),\Sigma_{t})
\end{equation*}
where $\Sigma_{t}=\text{diag}(s_{t,a}^{2}\pi_{t,a}(G_{1},...,G_{t-1}))$ and
$\Sigma_{0}=\text{diag}(s_{t,a}^{2}\pi_{0,a})$.

\paragraph{Induction} We use an inductive argument to prove the weak
convergence~\eqref{eqn:weak-convergence} for a policy $\pi$ and reward process
$R_{a,t}$ satisfying
Assumptions~\ref{assumption:reward},~\ref{assumption:policy}.  For the base
case, let $I_{0} \defeq \{ a \in [K]: \pi_{0,a} > 0 \}$ denote the arms with a
positive sampling probability (this is a deterministic set).  For the arms not in
$I_{0}$,
$\sqrt{n}\bar{R}^{n}_{0} = 0 = N(h_{a}\pi_{0,a}, s_{t,a}^{2}\pi_{0,a})$.  For
the remaining arms, by the Lindeberg CLT, we have that
$\Pstate_{0} \cd N(\pi_{0}h, \Sigma_{0})$.

Next, suppose
\begin{equation}
  \label{eqn:inductive-hypothesis}
  (\sqrt{n}\bar{R}^{n}_{0},\ldots,\sqrt{n} \bar{R}^{n}_{t})
  \cd (G_{0},...,G_{t}).
\end{equation}
To show weak convergence to $G_{t+1}$, it is sufficient to show 
\begin{equation}
  \label{eqn:inductive-weak-convergence}
  \E[f(\sqrt{n}\bar{R}^{n}_{0},\ldots,\sqrt{n} \bar{R}^{n}_{t+1})] \to
  \E[f(G_{0},...,G_{t+1})]~~\mbox{for any
bounded Lipschitz function}~f
\end{equation}
Fixing a bounded Lipschitz $f$, assume without loss of generality that
\begin{equation*}
  \sup_{x,y\in\mathbb{R}^{(t+1)\times k}}|f(x)-f(y)|\leq 1
  ~~~\mbox{and}~~~f\in\text{Lip}(\R^{(t+1)\times K}).
\end{equation*}

We first set up some basic notation. For any
$r_{0:t} \in \R^{(t+1) \times \numarm}$, define the conditional expectation
operator on a random variable $W$
\begin{equation*}
  \E_{r_{0:t}}[W] \defeq \E\left[ W \Bigg| \sqrt{n} \left\{\bar{R}^n_s\right\}_{s=0}^t = r_{0:t}\right].
\end{equation*}
Then, conditional on realizations of previous estimators up to time $t$, we
define a shorthand for the conditional expectation of $f$ and its limiting
counterpart.  Suppressing the dependence on $f$, let
$g_{n},g:\R^{t\times K} \to \R$ be
\begin{equation}
  \begin{aligned}
    \label{eqn:gs}
  g_{n}(r_{0:t}) &\defeq
             \E_{r_{0:t}}\left[f\left(\sqrt{n}\bar{R}_{0}^{n}, \ldots,
             \sqrt{n}\bar{R}_{t}^{n},\sqrt{n}\bar{R}_{t+1}^{n}\right)\right]
             = \E\left[f\left(r_{0:t}, \sqrt{n}\bar{R}_{t+1}^{n}\right)\right] \\
  g(r_{0:t}) & \defeq \E_{r_{0:t}}\left[f\left(\sqrt{n}\bar{R}_{0}^{n}, \ldots,
         \sqrt{n}\bar{R}_{t}^{n},\Sigma_{t+1}^{1/2}(r_{0:t})Z + h\pi_{t+1}(r_{0:t})\right)\right] \\
         & = \E\left[f\left(r_{0:t}, \Sigma_{t+1}^{1/2}(r_{0:t})Z + h\pi_{t+1}(r_{0:t})\right)\right] 
\end{aligned}
\end{equation}
where $Z\sim N(0,I)$ and the conditional covariance is determined by the
allocation $\pi_{t+1,a}(r_{0:t})$
\begin{equation}
  \label{eqn:cov}
  \Sigma_{t+1}(r_{0:t}) \defeq \text{diag}\left(\pi_{t+1,a}(r_{0:t})s_{t+1,a}^{2}\right).
\end{equation}
Conditional on the history $r_{0:t}$, $\sqrt{n}\bar{R}_{t+1}^{n}$ depends on
$r_{0:t}$ only through the sampling probabilities $\pi_{t+1,a}(r_{0:t})$.

To show the weak convergence~\eqref{eqn:inductive-weak-convergence}, decompose the
difference between
$\E[f(\sqrt{n}\bar{R}^{n}_{0},\ldots,\sqrt{n} \bar{R}^{n}_{t})]$ and
$\E[f(G_{0},\ldots,G_{t})]$ in terms of $g_{n}$ and $g$
\begin{align}
  &|\mathbb{E}[f(\sqrt{n}\bar{R}_{0}^{n},...,\sqrt{n}\bar{R}_{t+1}^{n})]
     - \mathbb{E}[f(G_{0},...,G_{t+1})]| \label{eqn:start} \\ 
  & =|\mathbb{E}[g_{n}(\sqrt{n}\bar{R}_{0}^{n},...,\sqrt{n}\bar{R}_{t}^{n})]
    -\mathbb{E}[g(G_{0},...,G_{t})]|  \nonumber  \\
  & =|\mathbb{E}[g_{n}(\sqrt{n}\bar{R}_{0}^{n},...,\sqrt{n}\bar{R}_{t}^{n})]
    -\mathbb{E}[g(\sqrt{n}\bar{R}_{0}^{n},...,\sqrt{n}\bar{R}_{t}^{n})]|
    +|\E[g(\sqrt{n}\bar{R}_{0}^{n},...,\sqrt{n}\bar{R}_{t}^{n})]-\mathbb{E}[g(G_{0},...,G_{t})]| \nonumber
\end{align}
By Assumption~\ref{assumption:policy}\ref{item:continuity} and dominated
convergence, $g$ is continuous almost surely under $(G_0, \ldots, G_{t})$.
From the inductive hypothesis~\eqref{eqn:inductive-hypothesis}, the continuous mapping theorem implies
\begin{equation*}
  \E[g(\sqrt{n}\bar{R}_{0}^{n},...,\sqrt{n}\bar{R}_{t}^{n})] \to \mathbb{E}[g(G_{1},...,G_{t})].
\end{equation*}

\paragraph{Uniform convergence of $g_n \to g$}
It remains to show the convergence
$\E[g_{n}(\sqrt{n}\bar{R}_{0}^{n},...,\sqrt{n}\bar{R}_{t}^{n})] \to
\E[g(\sqrt{n}\bar{R}_{0}^{n},...,\sqrt{n}\bar{R}_{t}^{n})]$.  While one would
expect pointwise convergence of $g_n(r_{0:t})$ to $g(r_{0:t})$ by the CLT,
proving the above requires controlling the convergence across random
realizations of the sampling probabilities
$\pi_{t+1,a}(\sqrt{n}\bar{R}_{0}^{n},...,\sqrt{n}\bar{R}_{t}^{n})$.  This is
complicated by the fact that we allow the sampling probabilities to be zero.  To
this end, we use the Stein's method for multivariate distributions to provide
rates of convergence for the CLT that are uniform across realizations
$r_{0:t}$. We use the bounded Lipschitz distance $d_{\text{BL}}$ to metrize
weak convergence.
                               
To bound $|\E[g_{n}(\Pstate_{0}, \ldots, \Pstate_{t})] - g(\Pstate_{0}, \ldots, \Pstate_{t})|$,
decompose
\begin{align}
  & |\E[g_{n}(\Pstate_{0}, \ldots, \Pstate_{t}) - g(\Pstate_{0}, \ldots, \Pstate_{t})]| \nonumber  \\
  & = |\E[\E_{\pstate_{0:t}}[f(\pstate_{0:t}, \Pstate_{n,t+1})
    - f(\pstate_{0:t},
    \Sigma_{t+1}(\pstate_{0:t})^{1/2}Z + \pi_{t+1}(\pstate_{0:t})h)]]|
    \nonumber \\
    & \leq \E[\E_{\pstate_{0:t}}[d_{\text{BL}}(Q_{n,t+1}(\pstate_{0:t}),
    \Sigma_{t+1}^{1/2}(\pstate_{0:t})Z)]] \nonumber \\
  & \leq \E[\E_{\pstate_{0:t}}[d_{\text{BL}}(Q_{n,t+1}(\pstate_{0:t}),
    \Sigma_{n,t+1}^{1/2}(\pstate_{0:t})Z)]
    + \E_{\pstate_{0:t}}[d_{\text{BL}}(
    \Sigma_{n,t+1}^{1/2}(\pstate_{0:t})Z,
    \Sigma_{t+1}^{1/2}(\pstate_{0:t})Z)]]
\label{eqn:bl-triangle}
\end{align}
where $Q_{n,t+1}(\pstate_{0:t})$ is the demeaned estimator
$$Q_{n,t+1}(\pstate_{0:t}) := \Pstate_{t+1}(\pstate_{0:t}) - \pi_{t+1}(\pstate_{0:t})h 
= \frac{1}{\sqrt{b_{t+1}n}}\sum_{j=1}^{b_{t+1}n} \left( \frac{\xi_{j}^{t+1}R_{j}^{t+1}}{\sqrt{b_{t+1}}} - \frac{\pi_{t+1}(\pstate_{0:t})h}{\sqrt{b_{t+1}n}} \right)$$
and the covariance $\Sigma_{n,t+1}(\pstate_{0:t})$ is
\begin{align}
  \label{eqn:cov-n}
  \Sigma_{n,t+1}(\pstate_{0:t})
  & \defeq \text{Cov}
  \left(\frac{\xi_{j}^{t+1}R_{j}^{t+1}}{\sqrt{b_{t+1}}}\right) \\
  & = \Sigma_{t+1}(\pstate_{0:t}) 
  + \frac{1}{b_{t+1}n}\text{diag}(\pi_{t+1,a}(\pstate_{0:t})h_{a}^{2})
  - \frac{1}{b_{t+1}n}(\pi_{t+1}(\pstate_{0:t})h)
  (\pi_{t+1}(\pstate_{0:t})h)^{\top}. \nonumber
\end{align}
To ease notation, we often omit the dependence on $\pstate_{0:t}$. Note that
$\Sigma_{n,t+1}$ is not quite equal to the covariance matrix $\Sigma_{t+1}$
appearing in the limiting Gaussian sequential experiment~\eqref{eqn:cov}.
This is because for any finite $n$, the covariance of this Gaussian
$\Sigma_{n,t+1}(r_{0:t})$ not only includes the variance of the arm rewards
but also correlations that emerge from sampling arms according to a random
multinomial vector (rather than a deterministic allocation).  This gap
vanishes as $n\to\infty$ and the covariance matrix converges to the diagonal
covariance matrix $\Sigma_{t+1}(r_{0:t})$.

\paragraph{Bounding the first term in inequality~\eqref{eqn:bl-triangle}}
The first term in inequality~\eqref{eqn:bl-triangle} measures the distance
between the sample mean $Q_{n,t+1}$ and its Gaussian limit.  Before we
proceed, it is helpful to define the following quantities, which describe
smoothness of the derivatives of any function $f\in\mathcal{C}^{3}$
\[
  M_{1}(f) = \sup_{x} \norm{\nabla f}_{2} 
  ~~~~M_{2}(f) = \sup_{x} \norm{\nabla^{2}f}_{op} 
  ~~~~M_{3}(f) = \sup_{x} \norm{\nabla^{3}f}_{op}
\]
The following bound, which we prove in
Section~\ref{section:proof-clt-sample-mean}, quantifies the rate of
convergence for the CLT using the multivariate Stein's
method~\citet{Meckes2009stein}. Recall that $\epsilon$ is the noise in the
rewards~\eqref{eqn:rewards}.
\begin{proposition}
  \label{prop:clt-sample-mean}
  For any $f \in \mathcal{C}^{3}$,
  \[
    |\E f(Q_{n,t+1}) - \E f(\Sigma_{n,t+1}^{1/2} Z) |
    \leq 
    C_{1} n^{-1/2} M_{2}(f) + C_{2}n^{-1/2} M_{3}(f)
  \]
  where $C_{1}$ and $C_{2}$ depend polynomially on $K,b_{t+1},h,s^{2},\norm{\epsilon}_{2}^{3}, \norm{\epsilon}_{2}^{4}$.
\end{proposition}

It remains to show convergence of $Q_{n,t+1}$ to $\Sigma_{n,t+1}Z$ for
Lipschitz test functions, which is required to control the bounded Lipschitz
distance.  We use standard Gaussian smoothing arguments found in
\citet{Meckes2009gauss}: by convolving the test function with a Gaussian
density, one obtains a smoother function for which the result of
Proposition~\ref{prop:clt-sample-mean} is applicable. At the same time, the
amount of smoothing is controlled to ensure the bias with the original test
function is small.

\begin{lemma}[{\citet[Corollary 3.5]{Meckes2009gauss}}]
\label{lemma:convolution}
For any 1-Lipschitz function $f$, consider the Gaussian convolution 
$(f * \phi_{\delta})(x) := \E[f(x + \delta Z)]$, where $Z\sim N(0,I)$.
\[
  \begin{aligned}
  M_{2}(f * \phi_{\delta}) 
  &\leq M_{1}(f) \sup_{\theta:\norm{\theta}_{2} = 1} \norm{\nabla \phi_{\delta}^{\top} \theta}
  \leq \sqrt{\frac{2}{\pi}} \frac{1}{\delta} \\
  M_{3}(f * \phi_{\delta}) 
  &\leq M_{1}(f) \sup_{\theta:\norm{\theta}_{2} = 1} \norm{\theta^{\top} \nabla^{2} \phi_{\delta} \theta}
  \leq \frac{\sqrt{2}}{\delta^{2}}
  \end{aligned}
\]
Moreover, for any random vector $X$,
$$ \E[(f * \phi_{\delta})(X) - f(X)] \leq \E[\delta\norm{Z}_{2}] \leq \delta \sqrt{\numarm} $$
\end{lemma}

Thus, for any 1-Lipschitz function $f$, we have the bound
\[
  \begin{aligned}
    |\E f(Q_{n,t+1}) - \E f(\Sigma_{n,t+1}^{1/2} Z) |
    &\leq |\E f(Q_{n,t+1}) - \E (f * \phi_{\delta})(Q_{n,t+1}) |
    + |\E (f * \phi_{\delta})(Q_{n,t+1}) - \E (f * \phi_{\delta})(\Sigma_{n,t+1}^{1/2} Z) | \\
    &\qquad + |\E (f * \phi_{\delta})(\Sigma_{n,t+1}^{1/2} Z)
    - \E f(\Sigma_{n,t+1}^{1/2} Z)| \\
    &\leq 2\delta \sqrt{\numarm}
       +  \sqrt{\frac{2}{\pi}} \frac{1}{\delta} C_{1} n^{-1/2}
        + \frac{\sqrt{2}}{\delta^{2}} C_{2} n^{-1/2}
  \end{aligned}
\]
Optimizing over $\delta$, we obtain
\[
  d_{\text{BL}}(Q_{n,t+1}, \Sigma_{n,t+1}^{1/2} Z)
  \leq C_{t+1} n^{-1/6}
\]
for some constant $C_{t+1}$ that depends only on $K, b_{t+1}, \E \norm{\epsilon}_{2}^{3}, \E \norm{\epsilon}_{2}^{4}$
but not on $n$ or $\pstate_{0:t}$. 

\paragraph{Bounding the second term in inequality~\eqref{eqn:bl-triangle}}
Finally, it remains to show uniform convergence of the second term in the
bound~\eqref{eqn:bl-triangle}. That is, the limiting Gaussian
$N(0, \Sigma_{n,t+1}^{1/2}(r_{0:t}))$ converges uniformly to
$N(0, \Sigma_{t+1}^{1/2}(r_{0:t}))$ as $n\to \infty$:
\begin{equation*}
  \E\left[\E_{r_{0:t}}\left[d_{\text{BL}}\left(\Sigma_{n,t+1}^{1/2}(r_{0:t})Z,
        \Sigma_{t+1}^{1/2}(r_{0:t})Z\right) \right]\right] \to 0.
\end{equation*}
First, for any two measures $(\mu,\nu)$, $d_{\text{BL}}(\mu,\nu)$ is upper
bounded by the total variation distnace $2d_{\text{TV}}(\mu,\nu)$, which is in
turn bounded above by KL divergence $\sqrt{\frac{1}{2}d_{\text{KL}}(\mu,\nu)}$
by Pinsker's inequality. We derive the following concrete bound on the KL
divergence, whose proof we defer to Section~\ref{section:proof-kl-cov}.
\begin{lemma}\label{lemma:kl-cov} The KL-divergence between 
  $N(0, \Sigma_{n,t+1}(\pstate_{0:t}))$ and
  $N(0, \Sigma_{t+1}(\pstate_{0:t}) )$
  is bounded above as follows
\[
 \dkl{N \left( 0,\Sigma_{n,t+1}(\pstate_{0:t}) \right)}{N \left( 0, \Sigma_{t+1}(\pstate_{0:t}) \right)}
  \leq \frac{1}{b_{t+1}n} \left[\frac{\max_{a} h_{a}^{2}/s_{t+1,a}^{2}}{1 - \max_{a} h_{a}^{2}/(s_{a}^{2}+h_{a}^{2})}+\sum_{a}\left(\frac{h_{a}^{2}}{s_{t+1,a}^{2}}\right)\right]
\]
\end{lemma}
\noindent This guarantees that
\[
  \begin{split}
  &\E\left[\E_{r_{0:t}}\left[d_{\text{BL}}\left(\Sigma_{n,t+1}^{1/2}(r_{0:t})Z,
  \Sigma_{t+1}^{1/2}(r_{0:t})Z\right) \right]\right] \\
  &\leq (b_{t+1}n)^{-1/2} \sqrt{2}\left(\frac{\max_{a} h_{a}^{2}/s_{t+1,a}^{2}}{1 - \max_{a} h_{a}^{2}/(s_{a}^{2}+h_{a}^{2})}+\sum_{a}\left(\frac{h_{a}^{2}}{s_{t+1,a}^{2}}\right)\right)^{1/2}
  \end{split}
\]

\paragraph{Conclusion} Altogether, we have shown
\[
  |\E[g_{n}(\Pstate_{0},\ldots,\Pstate_{t}) -
  g(\Pstate_{0},\ldots,\Pstate_{t})]| \leq C_{t+1}n^{-1/6} + D_{t+1}n^{-1/2}
\]
for constants $C_{t+1}, D_{t+1}$ that depend polynomially on $h_{a},s_{a}^{2}$
and higher order moments of $\epsilon_{a}$ as well as on $K$ and $b_{t+1}$.
Thus, this shows that
$\E[f(\Pstate_{0},\ldots,\Pstate_{t+1})] \to \E[f(G_{0},\ldots,G_{t+1})]$ as
$n\to \infty$ for any bounded Lipschitz function $f$, which implies the
desired weak convergence.

\subsubsection{Proof of Proposition~\ref{prop:clt-sample-mean}}
\label{section:proof-clt-sample-mean}

In order to quantify the rate of the CLT, we use the following result by
\citet{Meckes2009stein} which provides a characterization of Stein's method
for random vectors with arbitrary covariance matrices.
\begin{lemma}[{\citet[Theorem 3]{Meckes2009stein}}]
  \label{theorem:meckes}
  Let $(W, W')$ be an exchangeable pair of random vectors in $\R^{K}$. 
  Suppose that there exists $\lambda > 0$, a positive semi-definite matrix $\Sigma$,
  a random matrix $E$ such that
  \[
    \begin{aligned}
      \E[W' - W | W] &= -\lambda W \\
      \E[(W' - W)(W' - W)^{\top}|W] &= 2\lambda \Sigma + \E[E|W]
    \end{aligned}
  \]
  Then for any $f \in \mathcal{C}^{3}$,
  \[
    |\E f(W) - \E f(\Sigma^{1/2} Z) |
    \leq \frac{1}{\lambda}
    \left[ 
        \frac{\sqrt{K}}{4} M_{2}(f) \E \norm{E}_{H.S.}
        + \frac{1}{9} M_{3}(f) \E \norm{W' - W}^{3}
    \right]
  \]
where $\norm{\cdot}_{H.S.}$ is the Hilbert-Schmidt norm.
\end{lemma}

The rest of the proof is similar to that of~\citet[Theorem 7]{ChatterjeeMe08},
with slight modifications due to the fact that we have a non-identity
covariance.  For simplicity, define
$X_{j} := \frac{\xi_{j}^{t+1}R_{j}^{t+1}}{\sqrt{b_{t+1}}} -
\frac{\pi h}{\sqrt{b_{t+1}n}}$.
For any index $j$, we construct an independent copy $Y_j$ of $X_j$.  We
construct an exchangeable pair $(Q_{n,t+1},Q'_{n,t+1})$ by selecting an random
index $I\in \{1,...,b_{t+1}n\}$ chosen uniformly and independently from
$Q_{n,t+1}$ and letting
$$ Q'_{n,t+1} = Q_{n,t+1} - \frac{X_{I}}{\sqrt{b_{t+1}n}} + \frac{Y_{I}}{\sqrt{b_{t+1}n}}. $$

We can observe then that
\[
  \begin{aligned}
    \E [Q'_{n,t+1} - Q_{n,t+1}|Q_{n,t+1}]
    &= \frac{1}{\sqrt{b_{t+1}n}} \E[Y_{I} - X_{I}|Q_{n,t+1}] \\
    &= \frac{1}{(b_{t+1}n)^{3/2}} \sum_{j=1}^{b_{t+1}n} \E[Y_{j} - X_{j}|Q_{n,t+1}] \\
    &= -\frac{1}{b_{t+1}n} Q_{n,t+1}
  \end{aligned}
\]
by independence of $Y_j$ and $Q_{n,t+1}$. This pair satisfies the first condition of 
Theorem~\ref{theorem:meckes} with $\lambda = 1/(b_{t+1}n)$.

It also satisfies the second condition of Theorem~\ref{theorem:meckes} with:
\begin{align*}
  E_{a,a'} 
& = \frac{1}{(b_{t+1}n)^{2}} \sum_{j=1}^{b_{t+1}n} X_{j,a}X_{j,a'} - (\Sigma_{n,t+1})_{a,a'}
\end{align*}
by independence of $X_{i}$ and $Y_{i}$.
Thus, we have that
\[
  \begin{aligned}
    \E E_{a,a'}^{2} &= \frac{1}{(b_{t+1}n)^{4}} \sum_{j=1}^{b_{t+1}n} \E \left( \E[X_{j,a}X_{j,a'} - (\Sigma_{n,t+1})_{a,a'}|Q_{n,t+1}] \right)^{2} \\
    &\leq  \frac{1}{(b_{t+1}n)^{4}} \sum_{j=1}^{b_{t+1}n} \E\left[ \left( X_{j,a}X_{j,a'} - (\Sigma_{n,t+1})_{a,a'} \right)^{2} \right]  \\
    &= \frac{1}{(b_{t+1}n)^{3}} \E\left[ \left( X_{1,a}X_{1,a'}  \right)^{2} - (\Sigma_{n,t+1})_{a,a'}^{2}\right]  \\
  \end{aligned}
\]
This gives us a bound on the Hilbert-Schmidt norm:
\[
  \begin{aligned}
    \E \norm{E}_{H.S.} &\leq \sqrt{\sum_{a,a'} \E E_{a,a'}^{2}} 
    \leq \frac{1}{(b_{t+1}n)^{3/2}} \sqrt{\sum_{a,a'}\E \left( X_{1,a}X_{1,a'}  \right)^{2} - (\Sigma_{n,t+1})_{a,a'}^{2}}
    \leq \frac{1}{(b_{t+1}n)^{3/2}} \sqrt{\E\norm{X_{1}}_{2}^{4}}
  \end{aligned}
\]
We can further bound $\E\norm{X_{1}}_{2}^{4}$ by
a constant $C_{1}$ that is polynomial in $h, b_{t+1}^{-1}, s^{2}, \E\norm{\epsilon}_{2}^{3}$, and $\E\norm{\epsilon}_{2}^{4}$.
Finally, we can
bound $\E \norm{Q'_{n,t+1} - Q_{n,t+1}}_{2}^{3}$ as follows:
\[
  \begin{aligned}
    \E \norm{Q'_{n,t+1} - Q_{n,t+1}}_{2}^{3}
    =  \frac{1}{(b_{t+1}n)^{3/2}} \E \norm{Y_{I} - X_{I}}_{2}^{3}
    \leq \frac{1}{(b_{t+1}n)^{3/2}} 8\E \norm{X_{1}}_{2}^{3}
  \end{aligned}
\]
where the final inequality uses independence of $Y_{j}$ and $X_{j}$ as well as Holder's inequality.
We can bound $\E \norm{X_{1}}_{2}^{3}$ by another constant $C_{2}$ that is $h, b_{t+1}^{-1}, s^{2}$, and $\E\norm{\epsilon}_{2}^{4}$.
Plugging these bounds into the statement of Theorem~\ref{theorem:meckes}, we obtain the stated result.

\subsubsection{Proof of Lemma~\ref{lemma:kl-cov}}
\label{section:proof-kl-cov}

We use the expression for the KL-divergence between two multivariate normal distributions
\[
  \begin{aligned}
	& \dkl{N(0,  \Sigma_{n,t+1}(\pstate_{0:t}))}{N(0,\Sigma_{t+1}(\pstate_{0:t}))} \\
	&=\text{tr}\left(( \Sigma_{t+1}(\pstate_{0:t}))^{-1}  \Sigma_{n,t+1}(\pstate_{0:t})\right)
   - m +\log\frac{\det \Sigma_{t+1}(\pstate_{0:t})}{\det \Sigma_{n,t+1}(\pstate_{0:t})} \\
  &=\sum_{a} \left( s_{t+1,a}^{2}\pi_{t+1,a}(\pstate_{0:t}) \right)^{-1} 
  \left(s_{t+1,a}^{2}\pi_{t+1,a}(\pstate_{0:t}) + \pi_{t+1,a}(\pstate_{0:t}) (1 - \pi_{t+1,a}(\pstate_{0:t})) \frac{h_{a}^{2}}{b_{t+1}n} \right)-m \\
  &\qquad +\log\det  \Sigma_{t+1}(\pstate_{0:t}) -\log\det  \Sigma_{n,t+1}(\pstate_{0:t}) \\
	&\leq \sum_{a} \frac{1}{b_{t+1}n} \frac{\left(1-\pi_{t+1,a}(\pstate_{0:t})\right)h_{a}^{2}}{s_{t+1, a}^{2}}
  +\sum_{a}\log\left( s_{t+1,a}^{2}\pi_{t+1,a}(\pstate_{0:t}) \right) \\
  & \qquad-\sum_{a}\log\left(s_{t+1,a}^{2}\pi_{t+1,a}(\pstate_{0:t})+\frac{1}{b_{t+1}n}h_{a}^{2}\pi_{t+1,a}(\pstate_{0:t})\right) 
  + \left(\frac{1}{b_{t+1}n}\frac{\max_{a} h_{a}^{2}/s_{t+1,a}^{2}}{1 - \max_{a} h_{a}^{2}/(s_{t+1,a}^{2}+h_{a}^{2})} \right) \\
	&\leq\frac{1}{b_{t+1}n}\left[\sum_{a}\left(\frac{h_{a}^{2}}{s_{t+1,a}^{2}}\right) + \frac{\max_{a} h_{a}^{2}/s_{t+1,a}^{2}}{1 - \max_{a} h_{a}^{2}/(s_{a}^{2}+h_{a}^{2})}\right]
\end{aligned}
\]

We obtain $\det{\Sigma_{n,t+1}(\pstate_{0:t})^{\top}}$ using the fact that $\Sigma_{n,t+1}(\pstate_{0:t})$ is a rank-one perturbation of a diagonal matrix
\[
  \begin{aligned}
  &\log\det \Sigma_{n,t+1}	\\
  &=\log\det\left( \left( \Sigma_{t+1}+\frac{1}{b_{t+1}n}\text{diag}\left(\pi_{t+1,a}h_{a}^{2}\right)\right)\right) \\
  &\qquad+\log\det\Bigl(I + \left(  \left(\Sigma_{\epsilon,t}+\frac{1}{b_{t+1}n}\text{diag}\left(\pi_{t+1,a}h_{a}^{2}\right)\right)^{\top}\right)^{-1}
   \left(-\frac{1}{b_{t+1}n}(\pi_{t+1}h)(\pi_{t+1}h)^{\top}\right)\Bigr) \\
	&=\sum_{a} \log\left(s_{t+1,a}^{2}\pi_{t+1,a}+\frac{1}{b_{t+1}n}h_{a}^{2}\pi_{t+1,a}\right)  \\
  &\qquad+\log\det\left(I-\frac{1}{b_{t+1}n}( \pi_{t+1}h)^{\top}\left( \left(\Sigma_{t+1}
  +\frac{1}{b_{t+1}n}\text{diag}\left(\pi_{t+1,a}h_{a}^{2}\right)\right)^{\top}\right)^{-1}(\pi_{t+1}h)\right) \\
	&=\sum_{a} \log\left(\pi_{t+1,a}s_{t+1,a}^{2}
  +\frac{1}{b_{t+1}n}\pi_{t+1,a}h_{a}^{2}\right)
  +\log\left(1-\frac{1}{b_{t+1}n} \sum_{a} \pi_{t+1,a}\frac{h_{a}^{2}}{s_{t+1,a}^{2}+(1/b_{t+1}n)h_{a}^{2}}\right) \\
  &\geq \sum_{a} \log\left(\pi_{t+1,a}s_{t+1,a}^{2}+\frac{1}{b_{t+1}n}\pi_{t+1,a}h_{a}^{2}\right)
   + \log\left(1- \frac{1}{b_{t+1}n}\max_{a} \frac{h_{a}^{2}}{s_{t+1,a}^{2}+(1/b_{t+1}n)h_{a}^{2}}\right)
  \end{aligned}
\]

Finally, we use the bound $\log(1-x)\geq \frac{-x}{1-x}$ to obtain
\[
  \begin{aligned}
    \log\det\Sigma_{n,t+1}&\geq \sum_{a}
        \log\left(\pi_{t+1,a}s_{t+1,a}^{2}+\frac{1}{b_{t+1}n}\pi_{t+1,a}h_{a}^{2}\right)
    -  \frac{1}{b_{t+1}n}\frac{\max_{a} \frac{h_{a}^{2}}{s_{t+1,a}^{2}+(1/b_{t+1}n)h_{a}^{2}}}{1 - \max_{a} \frac{h_{a}^{2}}{b_{t+1}ns_{t+1,a}^{2}+h_{a}^{2}}} \\
  &\geq \sum_{a} \log\left(\pi_{t+1,a}s_{t+1,a}^{2}+\frac{1}{b_{t+1}n}\pi_{t+1,a}h_{a}^{2}\right)
   -  \frac{1}{b_{t+1}n}\frac{\max_{a} h_{a}^{2}/s_{t+1,a}^{2}}{1 - \max_{a} h_{a}^{2}/(s_{a}^{2}+h_{a}^{2})} \\
  \end{aligned}
\]

\subsection{Proof of Corollary~\ref{cor:proof-limit-rate}}
\label{section:proof-limit-rate}

We show that for any horizon $T$
\[
  d_{\text{BL}} \left( (\Pstate_{0},\ldots,\Pstate_{T-1} ),(G_{0},\ldots,G_{T-1} ) \right) \leq C \frac{M^{T} - 1}{M - 1} n^{-1/6},
\]
where $M = \left(1 + \bar{L} ( s_*^2/b_* \sqrt{K} + \max_{a} |h_{a}|) \right)$, 
$\bar{L}$ is the bound on the Lipschitz constant for $\pi^{1/2}$ and $\pi$, and $s_{*}^{2} = \max_{a} s_{a}^{2}$
and $b_{*} = \min_{0\leq t\leq T-1} b_{t}$.

We prove the statement through induction. For the base case, we have from Proposition~\ref{prop:clt-sample-mean},
Lemma~\ref{lemma:convolution}, and Lemma~\ref{lemma:kl-cov} that $d_{BL}(\Pstate_{0}, N(\pi_{0} h, \Sigma_{0}))$ is bounded by 
$C_{0}n^{-1/6}$, for some constant $C_{0}$ that depends on $\numarm$, $b_{0}$, $h$, $s^{2}$,
and higher moments of $\norm{\epsilon}_{2}$.
Suppose now that this holds up to $t$
\[
  d_{\text{BL}} \left( (\Pstate_{0},\ldots,\Pstate_{t} ),(G_{0},\ldots,G_{t} ) \right) \leq C \sum_{s=0}^{t} M^{s} n^{-1/6}.
\]
We proceed to show that this inequality then holds for $t+1$.

Recall the decomposition~\eqref{eqn:start} for any bounded Lipschitz function
$f$ and corresponding $g, g_n$ defined as in Eq.~\eqref{eqn:gs}.
From Proposition~\ref{prop:clt-sample-mean}, Lemma~\ref{lemma:convolution}, and Lemma~\ref{lemma:kl-cov}
\[
  |\mathbb{E}[g_{n}(\sqrt{n}\bar{R}_{0}^{n},...,\sqrt{n}\bar{R}_{t}^{n})]
    -\mathbb{E}[g(\sqrt{n}\bar{R}_{0}^{n},...,\sqrt{n}\bar{R}_{t}^{n})]| \leq C_{t+1} n^{-1/6} + D_{t+1}n^{-1/2}
\]
Thus, there exists some $C'_{t+1}$ such that the above is bounded by $C'_{t+1}n^{-1/6}$. Note that the dependence
of $C'_{t}$ on $t$ is only through a polynomial on $1/b_{t+1}$. Thus we can bound this by a constant $C$ which
will depends on $T$ only through polynomial factors of $b_{*} = \min_{0\leq t\leq T-1} b_{t}$.

To bound the second term in the decomposition~\eqref{eqn:start}, we use the
fact that
\[
  g(r_{0:t}) = \E\left[f\left(r_{0:t}, \Sigma_{t+1}^{1/2}(r_{0:t})Z +
      h\pi_{t+1}(r_{0:t})\right)\right]
\]
is a bounded Lipschitz function.  We seek to identify the Lipschitz
constant for $g$, as we can then utilize the rate for bounded Lipschitz
functions in the induction hypothesis.  For any two $r_{0:t}$, $r'_{0:t}$,
\[
\begin{aligned}  
  |g(r_{0:t}) - g(r'_{0:t})| & \leq \E \ltwo{(r_{0:t}, \Sigma_{t+1}^{1/2}(r_{0:t})Z + h\pi_{t+1}(r_{0:t})) - (r'_{0:t}, \Sigma_{t+1}^{1/2}(r'_{0:t})Z + h\pi_{t+1}(r'_{0:t}))}\\
  &\leq \E \left(\ltwo{r_{0:t} - r'_{0:t}}^{2} 
    + \ltwo{\Sigma_{t+1}^{1/2}(r_{0:t})Z - \Sigma_{t+1}^{1/2}(r'_{0:t})Z}^{2}+ \ltwo{h\pi_{t+1}(r_{0:t}) - h\pi_{t+1}(r'_{0:t})}^2\right)^{\half} \\
  &\leq \ltwo{r_{0:t} - r'_{0:t}} 
  + \E\ltwo{\Sigma_{t+1}^{1/2}(r_{0:t})Z - \Sigma_{t+1}^{1/2}(r'_{0:t})Z}
  + \ltwo{h\pi_{t+1}(r_{0:t}) - h\pi_{t+1}(r'_{0:t})} \\
  &\leq \left(1 + \bar{L} ( s_*^2/b_* \E \ltwo{Z} + \max_{a} |h_{a}|) \right)\ltwo{r_{0:t} - r'_{0:t}}
\end{aligned}
\]
using $\sqrt{a+b} \le \sqrt{a} + \sqrt{b}$ and $a - b = (\sqrt{a} + \sqrt{b})(\sqrt{a}-\sqrt{b})$.
This implies
\[
  \begin{aligned}
  & |\mathbb{E}[f(\sqrt{n}\bar{R}_{0}^{n},...,\sqrt{n}\bar{R}_{t+1}^{n})]
     - \mathbb{E}[f(G_{0},...,G_{t+1})]| \\
  & \leq |\mathbb{E}[g_{n}(\sqrt{n}\bar{R}_{0}^{n},...,\sqrt{n}\bar{R}_{t}^{n})]
  -\mathbb{E}[g(\sqrt{n}\bar{R}_{0}^{n},...,\sqrt{n}\bar{R}_{t}^{n})]|
  +|\E[g(\sqrt{n}\bar{R}_{0}^{n},...,\sqrt{n}\bar{R}_{t}^{n})]-\mathbb{E}[g(G_{0},...,G_{t})]| \\
  & \leq C n^{-1/6}
  +  M \cdot C \sum_{s=0}^{t} M^{s} n^{-1/6} 
  = C \sum_{s=0}^{t+1} M^{s} n^{-1/6}
  \end{aligned}
\]
Noting that $\sum_{s=0}^{T-1} M^{s} = \frac{M^{T} - 1}{M - 1}$,
we have the desired inequality.

\subsection{Proof of Corollary~\ref{cor:root-n-rate}}
\label{section:proof-root-n-rate}

The proof proceeds as in the proof of Corollary~\ref{cor:proof-limit-rate}. We show that for any horizon $T$
\[
  d_{\text{BL}} \left( (\Pstate_{0},\ldots,\Pstate_{T-1} ),(G_{0},\ldots,G_{T-1} ) \right) \leq \tilde{C} \frac{M^{T} - 1}{M - 1} n^{-1/2},
\]
where $M = \left(1 + \bar{L} ( s_*^2/b_* K^{1/2} + \max_{a} |h_{a}|) \right)$, 
$\bar{L}$ is the bound on the Lipschitz constant for $\pi^{1/2}$ and $\pi$, and $s_{*}^{2} = \max_{a} s_{a}^{2}$
and $b_{*} = \min_{0\leq t\leq T-1} b_{t}$.

For the base case, note that by Assumption~\ref{assumption:overlap},
$ \max_{a} \pi_{0,a}^{-3}(\mu_{0}, \sigma_{0}) \leq C_{o}$. Applying standard
results for rates of weak convergence of the central limit theorem (see
Lemma~\ref{lemma:clt-rate}), we have that
$d_{BL}(\Pstate_{0}, N(\pi_{0} h, \Sigma_{0}))$ is bounded by $C_{0}n^{-1/2}$,
for some constant $C_{0}$ that depends on $\numarm$ and polynomially on
$C_{o}$, $b_{0}$, $h$, and $s^{2}$.  For the induction step, assume that up to
$t$, the following inequality holds
\[
  d_{\text{BL}} \left( (\Pstate_{0},\ldots,\Pstate_{t} ),(G_{0},\ldots,G_{t} ) \right) \leq \tilde{C} \sum_{s=0}^{t} M^{s} n^{-1/2}.
\]
for $M = \left(1 + \bar{L} ( s_*^2/b_* K^{1/2} + \max_{a} |h_{a}|) \right)$ and a constant $\tilde{C}$ that depends on $\numarm$, $s_{a}^{2}$, $b_{t+1}$, and $h_{a}$.
We proceed to show that this inequality then holds for $t+1$.

Under Assumption~\ref{assumption:overlap}, we can obtain a $n^{-1/2}$ rate for
the central limit theorem specialized for bounded Lipschitz distance.  First,
we establish some facts about the covariance matrix $\Sigma_{n,t+1}(r_{0:t})$
defined in Eq.~\eqref{eqn:cov-n} for the sample average $Q_{n, t+1}(r_{0:t})$.
We omit the dependence on $r_{0:t}$ to ease notation.
\begin{lemma} \label{lemma:min-max-eig} Let $\lambda_{\max}(\cdot)$ and
  $\lambda_{\min}(\cdot)$ denote the maximum and minimum eigenvalues
  respectively. For any realization of $\pstate_{0:t}$, we have 
  \[
    \begin{aligned}
    \lambda_{\max} (\Sigma_{n,t+1} ) &\leq \lambda_{\max}
    \left( \Sigma_{t+1} 
    +\frac{1}{b_{t+1}n}  \diag \left( \pi_{t+1,a}h_{a}^{2} \right) \right)
    \leq \max_{a}\left( \frac{s_{a}^{2} + h_{a}^{2}}{b_{t+1}} \right) \\
    \lambda_{\min} ( \Sigma_{n,t+1} ) 
    & \geq \lambda_{\min}(  \Sigma_{t+1} )
    +\frac{1}{b_{t+1}n}  \lambda_{\min}\left( \diag(\pi_{t+1,a}h_{a}^{2})
    -  (\pi_{t+1}h)(\pi_{t+1}h)^{\top}  \right)\\
    &\geq  \left( \min_{a}\pi_{t+1,a}(\pstate_{0:t}) \right) \left( \min_{a}s_{t+1,a}^{2}  \right).
    \end{aligned}
  \]
\end{lemma}

\begin{proof-of-lemma}
  The first inequality follows since
  $\lambda_{\max}(A + B) \le \lambda_{\max}(A) + \lambda_{\max}(B)$ for
  symmetric matrices $A,B \in \R^{m\times m}$. The second inequality uses
  $\lambda_{\min}(A + B) \geq \lambda_{\min}(A) + \lambda_{\min}(B)$.
  Moreover, for any $b\in \R^{m}$
  \[
  \begin{split}
  & b^{\top}  \left( \text{diag}(\pi_{t+1,a}(\pstate_{0:t})h_{a}^{2}) - (\pi_{t+1}(\pstate_{0:t})h)(\pi_{t+1}(\pstate_{0:t})h)^{\top} \right) b \\
  &= \sum_{a} \pi_{t+1,a}(\pstate_{0:t})h_{a}^{2}b_{a}^{2} - \left( \sum_{a} \pi_{t+1,a}(\pstate_{0:t})h_{a}b_{a} \right)^{2} \geq 0
  \end{split}
  \]
  as by Cauchy-Schwartz:
  \[
   \left( \sum_{a} \pi_{t+1,a}(\pstate_{0:t})h_{a}b_{a} \right)^{2} \leq \left(\sum_{a} \pi_{t+1,a}(\pstate_{0:t}) \right)
   \left(\sum_{a} \pi_{t+1,a}(\pstate_{0:t})h_{a}^{2}b_{a}^{2} \right) \leq \left(\sum_{a} \pi_{t+1,a}(\pstate_{0:t})h_{a}^{2}b_{a}^{2} \right)
  \]
so the matrix is positive semi-definite
\end{proof-of-lemma}

Next, we show convergence of $Q_{n, t+1}(r_{0:t})$ to a standard normal
random vector $Z \sim N(0,I_{m})$ via a classical result, which is a corollary
of~\citet[Theorem 1]{Bhattacharya70}.  The following result explicitly
quantifies how the rate of convergence depends on the sampling probabilities,
allowing us to guarantee uniform convergence.
\begin{lemma}[{\citet[Theorem 1]{Bhattacharya70}}]. 
  \label{lemma:clt-rate}
  Let $\{X_i\}_{i =1}^n$ be a sequence of $n$ independent random vectors
  taking values in $\R^{\numarm}$ with covariance matrix $\Sigma_n$. Assume
  that for some $\delta > 0$
  \begin{equation*}
    M \defeq \sup_{n} \frac{1}{n} \sum_{i=1}^{n}
    \E \ltwo{\Sigma_{n}^{-1/2} X_{i}}^{3+\delta} < \infty.
  \end{equation*}
  Then, the normalized partial sum
  $\bar{Z}_{n} = \frac{1}{\sqrt{n}} \sum_{i=1}^{n} \Sigma_{n}^{-1/2} X_{i}$
  satisfies
  \begin{equation*}
    d_{\text{BL}}(\bar{Z}_n, N(0, I))
    \leq A_{K,\delta} M^{\frac{3+3\delta}{3+\delta}}n^{-1/2} + B_{K,\delta} M^{\frac{3}{3+\delta}}n^{-1/2}
  \end{equation*}
  where $A_{K,\delta}, B_{K,\delta}$ are constants that only depends on $K$ and
  $\delta$. 
\end{lemma}

\noindent  In order to guarantee uniform convergence over $r_{0:t}$ using
Lemma~\ref{lemma:clt-rate}, we use
Assumptions~\ref{assumption:reward}\ref{item:moment}
to control the moment term $M$. Since we assume the existence of a fourth moment,
we let $\delta = 1$.
Using $C_{\numarm}$ to denote a constant that only depends on
$\numarm$ and polynomially on $s^{2}$, $b_{t+1}$, $h$ (that may differ line by line), use lower bound on the
minimum eigenvalue of $\Sigma_{n,t+1}$ given in Lemma~\ref{lemma:min-max-eig}
to get
\begin{equation*}
  \begin{split}
  M(\pstate_{0:t})
  & \defeq \sup_{n} \frac{1}{b_{t+1}n}\sum_{j=1}^{b_{t+1}n} \mathbb{E}
    \ltwo{\left(  \Sigma_{n,t+1}(\pstate_{0:t})  \right)^{-1/2} 
     \left(
      \xi_{j}^{t+1}R_{j}^{t+1} -\frac{\pi_{t+1,a}(\pstate_{0:t})h_{a}}{\sqrt{n}}
    \right)}^{4} \\
  & \leq \left( \min_{a}\pi_{t+1,a}(\pstate_{0:t}) \min_{a} s^{2}_{t+1,a} \right)^{-2} 
  \sup_{n} \frac{1}{b_{t+1}n}\sum_{j=1}^{b_{t+1}n} \mathbb{E}
  \ltwo{ \left( \xi_{j}^{t+1}R_{j}^{t+1} -\frac{\pi_{t+1}(\pstate_{0:t})h}{\sqrt{n}}\right)}^{4} \\
      &\leq C_{K} \left( \max_{a}\pi_{t+1,a}^{-2}(\pstate_{0:t}) \right)
      \max_{a} \E \left|\xi_{a,j}^{t+1}R_{a,j} - \frac{\pi_{t+1,a}(\pstate_{0:t})h_{a}}{\sqrt{n}}\right|^{4} < \infty\\
  \end{split}
\end{equation*}
Using Assumption~\ref{assumption:reward}\ref{item:moment}, conclude
\[
  M(\pstate_{0:t})\leq C_{K} \left( \max_{a}\pi_{t+1,a}^{-2}(\pstate_{0:t}) \right)
  \left( C + \left(\frac{\max_{a} |h_{a}|}{\sqrt{n}} \right)^{4} \right) < \infty 
\]
where $C$ is the moment bound in $\epsilon_{a}$ given in Assumption~\ref{assumption:reward}\ref{item:moment}.

Apply Lemma~\ref{lemma:clt-rate} to get
\[
  \begin{split}
  &\E[\E_{\pstate_{0:t}} \left[ d_{\text{BL}} ( Q_{n, t+1}, \Sigma^{1/2}_{n,t+1}(\pstate_{0:t})Z )\right]] \\
  & \leq \max_{a} \left(\frac{s_{a}^{2} + h_{a}^{2}}{b_{t+1}}\right)^{1/2} 
  \E[\E_{\pstate_{0:t}} \left[ d_{\text{BL}} ( \Sigma_{n,t+1}^{-1/2}Q_{n, t+1}, Z )\right]] \\
  & \leq  \max_{a}\left(\frac{s_{a}^{2} + h_{a}^{2}}{b_{t+1}}\right)^{1/2} 
  \E[\E_{\pstate_{0:t}} 
  \left[ A_{K,\delta} M(\pstate_{0:t})^{\frac{3}{2}} (b_{t+1}n)^{-1/2}
  + B_{K,\delta} M(\pstate_{0:t})^{\frac{3}{4}} (b_{t+1}n)^{-1/2} \right]].
  \end{split}
\]
Taking expectation over $\pstate_{0:t}$ and using the preceding bound on $M(\pstate_{0:t})$, we have that
\[
  \begin{split}
  & \E[\E_{\pstate_{0:t}} \left[ d_{\text{BL}} (Q_{n, t+1}, \Sigma^{1/2}_{n,t+1}(\pstate_{0:t})Z )\right]] \\
  & \leq \max_{a} \left(\frac{s_{a}^{2} + h_{a}^{2}}{b_{t+1}}\right)^{1/2} \cdot 
  \E[\E_{\pstate_{0:t}} 
  \left[ A_{K,\delta} M(\pstate_{0:t})^{\frac{3}{2}} (b_{t+1}n)^{-1/2}
  + B_{K,\delta} M(\pstate_{0:t})^{\frac{3}{4}} (b_{t+1}n)^{-1/2} \right]] \\
  & \leq n^{-1/2} C_{K} \E \left[ \max_{a}\pi_{t+1,a}^{-3}(\Pstate_{0:t}) \right]
  \left( C^{\frac{3}{2}} + \frac{\max_{a} |h_{a}|^{6} }{n^{3}} \right) \\
  &\qquad +n^{-1/2} C_{K} \E \left[\max_{a}\pi_{t+1,a}^{-3/2}(\Pstate_{0:t}) \right]
  \left( C^{\frac{3}{4}} + \frac{\max_{a} |h_{a}|^{3} }{n^{3/2}} \right).
  \end{split}
\]
Due to Assumption~\ref{assumption:overlap}, the moment $\E[\max_{a}\pi_{t+1,a}^{-3}(\Pstate_{0:t})]$ is bounded by $C_{o}$,
which gives us a $n^{-1/2}$ upper bound on $\E[\E_{\pstate_{0:t}} \left[ d_{\text{BL}} (Q_{n}^{t+1}, \Sigma^{1/2}_{n,t+1}(\pstate_{0:t})Z )\right]]$.
Finally, combining with the result in Lemma~\ref{lemma:kl-cov}, there exists a constant $\tilde{C}_{t+1}$ depending on
$\numarm$ and polynomially on $C_{o}$, $s^{2}$, $b_{t+1}$, and $h$
\[
  \E[\E_{\pstate_{0:t}} \left[ d_{\text{BL}} (\Pstate_{n, t+1}, \Sigma^{1/2}_{t+1}(\pstate_{0:t})Z + h\pi_{t+1}(\pstate_{0:t}))\right]] \leq \tilde{C}_{t+1} n^{-1/2}
\]
uniformly over all $\pstate_{0:t}$.

Recall the decomposition~\eqref{eqn:start} for any bounded Lipschitz function
$f$ and corresponding $g, g_n$ defined as in Eq.~\eqref{eqn:gs}. The above result implies that
\[
  |\E[g_{n}(\sqrt{n}\bar{R}_{0}^{n},...,\sqrt{n}\bar{R}_{t}^{n})] - \E[g(\sqrt{n}\bar{R}_{0}^{n},...,\sqrt{n}\bar{R}_{t}^{n})] | \leq \tilde{C} n^{-1/2}
\]
for some constant $\tilde{C} =\max_{t} \tilde{C}_{t}$.

It remains to show convergence of $\E[g(\sqrt{n}\bar{R}_{0}^{n},...,\sqrt{n}\bar{R}_{t}^{n})]$ to $\E[g(G_{1}, \ldots, G_{t})]$.
Recall that
\[
  g(r_{0:t}) = \E\left[f\left(r_{0:t}, \Sigma_{t+1}^{1/2}(r_{0:t})Z +
      h\pi_{t+1}(r_{0:t})\right)\right]
\]
is a bounded Lipschitz function. From the proof of Corollary~\ref{section:proof-limit-rate}
we have that the Lipschitz constant for $g$ is bounded by $M = \left(1 + \bar{L} ( s_*^2/b_* \sqrt{K} + \max_{a} |h_{a}|) \right)$
where $\bar{L}$ is the upper bound of the Lipschitz constants for $\pi^{1/2}$ and $\pi$, $s_{*} := \max_{a} s_{a}$, and $b_{*} := \min_{t} b_{t}$.
We can then use the induction hypothesis, which gives us
\[
  |\E[g(\sqrt{n}\bar{R}_{0}^{n},...,\sqrt{n}\bar{R}_{t}^{n})] - \E[g(G_{0},...,G_{t})]| \leq M \cdot \tilde{C} \sum_{s=0}^{t} M^{s} n^{-1/2}.
\]
This implies
\[
  \begin{aligned}
  & |\mathbb{E}[f(\sqrt{n}\bar{R}_{0}^{n},...,\sqrt{n}\bar{R}_{t+1}^{n})]
     - \mathbb{E}[f(G_{0},...,G_{t+1})]| \\
  & \leq |\mathbb{E}[g_{n}(\sqrt{n}\bar{R}_{0}^{n},...,\sqrt{n}\bar{R}_{t}^{n})]
  -\mathbb{E}[g(\sqrt{n}\bar{R}_{0}^{n},...,\sqrt{n}\bar{R}_{t}^{n})]|
  +|\E[g(\sqrt{n}\bar{R}_{0}^{n},...,\sqrt{n}\bar{R}_{t}^{n})]-\mathbb{E}[g(G_{0},...,G_{t})]| \\
  & \leq \tilde{C} n^{-1/2}
  +  M \cdot \tilde{C} \sum_{s=0}^{t} M^{s} n^{-1/2} 
  = \tilde{C} \sum_{s=0}^{t+1} M^{s} n^{-1/2}
  \end{aligned}
\]
Noting that $\sum_{s=0}^{T-1} M^{s} = \frac{M^{T} - 1}{M - 1}$,
we have the desired inequality.

\subsection{Proof of Corollary~\ref{cor:bsr-limit}}
\label{section:proof-bsr-limit}

From the proof of Theorem~\ref{theorem:limit}, the convergence of the Bayes
simple regret immediately follows from the fact that the weak
convergence~\eqref{eqn:weak-convergence} is uniform in $h$ under the
hypothesized moment condition on $h$.  The discontinuity points of $\pi_T$ are
of measure zero with respect to $(G_0,\dots,G_{T-1})$, the
convergence~\eqref{eqn:weak-convergence} and the Portmanteau theorem implies
that for any fixed $h$, as $n \to \infty$
\begin{equation*}
  \E[\pi_{T}(\sqrt{n}\bar{R}_{0}^{n},\ldots,\sqrt{n}\bar{R}_{T-1}^{n})|h] \to
  \E[\pi_{T}(G_{0},\ldots,G_{T-1})|h]
\end{equation*}
For any prior distribution $\nu$ over $h$ such that $\E_{\nu}||h|| < \infty$,
dominated convergence gives 
\begin{align*}
  \bsr_{T}(\pi, \nu, \sqrt{n}\bar{R}^{n})
  &= \E_{h\sim\nu}[h^{\top}\pi_{T}(\sqrt{n}\bar{R}_{0}^{n},\ldots,\sqrt{n}\bar{R}_{t+1}^{n})] 
  = \E_{h\sim\nu}[h^{\top}\E[\pi_{T}(\sqrt{n}\bar{R}_{0}^{n},\ldots,\sqrt{n}\bar{R}_{t+1}^{n})|h]] \\
  &\to \E_{h\sim\mu}[h^{\top}\E[\pi_{T}(G_{0},\ldots,G_{T-1})|h]] 
  = \bsr_{T}(\pi, \nu, G).
\end{align*}


\section{Derivations for Bayesian adaptive experiment}
\label{section:proof-bae}
\subsection{Proof of Lemma~\ref{lemma:mdp}}
\label{section:proof-mdp}

In the posterior updates~\eqref{eqn:posterior}, the
result~\eqref{eqn:dynamics} follows by noting
\begin{equation*}
  G_{t, a} \mid \mu_t, \sigma_t \eqd
  \E[G_{t, a} \mid \mu_t, \sigma_t] + \sqrt{\var(G_{t, a} \mid \mu_t, \sigma_t)} Z_{t, a},
\end{equation*}
and plugging in expressions for the conditional mean and variance of
$G_{t, a}$. Instead of this change of variables argument, we give a more
evocative derivation below that shows
\begin{align}
  \label{eqn:marginal-normality}
   \mu_{t+1} \mid \mu_t, \sigma_t
  \sim N\left(\mu_t, \diag(\sigma_t^2 - \sigma_{t+1}^2)\right).
\end{align}

Gaussianity of $ \mu_{t+1} \mid \mu_t, \sigma_t$ follows since $\mu_{t+1}$ is
a linear function of $G_{t, a}$ given $\mu_t, \sigma_t$
\begin{align*}
  \mu_{t+1, a}
  & = \sigma_{t+1, a}^{2}
  \left( \sigma_{t, a}^{-2} \mu_{t, a} + s_a^{-2}b_t G_{t, a} \right) \\
  & = \left(\sigma_{t, a}^{-2} + s_a^{-2} b_t \pi_{t, a}(\mu_t, \sigma_t)\right)^{-1}
  \left( \sigma_{t, a}^{-2} \mu_{t, a} + s_a^{-2}b_t G_{t, a} \right).
\end{align*}
We now derive expressions for the conditional mean and variance of
$\mu_{t+1}$. Noting that
\begin{equation*}
  \E[G_{t, a} \mid \mu_t, \sigma_t]
  = \E[ \E[G_{t, a} \mid h, \mu_t, \sigma_t] \mid \mu_t, \sigma_t]
  = \E[\pi_{t, a}(\mu_t, \sigma_t)h_a \mid \mu_t, \sigma_t] = 
  \pi_{t, a}(\mu_t, \sigma_t)\mu_{t, a},
\end{equation*}
we conclude
\begin{equation*}
  \E[\mu_{t+1, a} \mid \mu_t,\sigma_t] = \sigma_{t+1, a}^2
  \left(\sigma_{t, a}^{-2} \mu_{t, a} + s_a^{-2} b_t \pi_{t, a}(\mu_t, \sigma_t) \mu_{t, a}\right)
  = \mu_{t, a}.
\end{equation*}

Next, use the law of total variance to derive
\begin{align*}
  \var(\mu_{t+1, a} \mid \mu_t, \sigma_t)
  & = \left( \sigma_{t+1, a}^2 s_{a}^{-2}b_{t}\right)^2
  \var(G_{t, a} \mid \mu_t, \sigma_t) \\
  & = \left( \sigma_{t+1, a}^2 s_{a}^{-2}b_{t} \right)^2
  \Big(
  \E[\var(G_{t, a} \mid h, \mu_t, \sigma_t) \mid \mu_t, \sigma_t]
  + \var(\E[G_{t, a} \mid h, \mu_t, \sigma_t] \mid \mu_t, \sigma_t)
  \Big) \\
  & = \left( \sigma_{t+1, a}^2 s_{a}^{-2}b_{t} \right)^2
    \left( \frac{\pi_{t, a}(\mu_t, \sigma_t)s_a^2}{b_t} +\pi_{t, a}(\mu_t, \sigma_t)^{2} \sigma_{t, a}^2\right) 
   = \frac{\sigma_{t, a}^4 b_t \pi_{t, a}(\mu_t, \sigma_t) }
    {s_a^2 + \sigma_{t, a}^2 b_t \pi_{t, a}(\mu_t, \sigma_t)}.
\end{align*}
We arrive at the desired result~\eqref{eqn:marginal-normality} since
\begin{equation}
  \label{eqn:var-decrease}
  \sigma_{t, a}^2 - \sigma_{t+1,a}^2 = \sigma_{t, a}^2 - \left(\sigma_{t,
      a}^{-2} + s_a^{-2} b_t \pi_{t, a}(\mu_t, \sigma_t)\right)^{-1}
    = \frac{\sigma_{t, a}^4 b_t \pi_{t, a}(\mu_t, \sigma_t) } {s_a^2 + \sigma_{t,
      a}^2 b_t \pi_{t, a}(\mu_t, \sigma_t)}.
\end{equation}

\subsection{Proof of Corollary~\ref{cor:bayes-limit}}
\label{section:proof-bayes-limit}

The posterior states $(\mu_{n,t}, \sigma_{n,t})$ can be expressed as a
function of the sample mean estimators and the propensity scores
\[
\begin{aligned}
  \sigma_{n,t+1}^{-2} &= \sigma_{0}^{-2} + \sum_{v=0}^{t} b_{v}\pi_{v}(\mu_{n,v}, \sigma_{n,v})s^{-2} \\
  \mu_{n,t+1} &= \sigma_{n,t+1}^{2} \left( \sigma_{0}^{-2}\mu_{0} + \sum_{v=0}^{t} b_{v}s^{-2} \Pstate_{v} \right).
\end{aligned}
\]
where the operations are vector-wise.

Since the allocations $\pi_v$ are assumed to be continuous in the posterior
state $(\mu_{n,v}, \sigma_{n,v})$, the states
$(\mu_{n, t+1}, \sigma_{n, t+1})$ are continuous functions of the sample mean
estimators $\Pstate_v$. By the continuous mapping theorem, we conclude
\[
  (\mu_{n,0}, \sigma_{n,0}, \ldots, \mu_{n,T-1},\sigma_{n,T-1}) \cd (\mu_{0}, \sigma_{0}, \ldots, \mu_{T-1},\sigma_{T-1}).
\]


\subsection{Proof of Lemma~\ref{lemma:rho-reduction}}
\label{section:proof-rho-reduction}

At each fixed epoch $t \in [T]$, we consider imagined \emph{counterfactual
  state transitions} if one follows future allocations
$\bar{\rho}_{t, a},\ldots,\bar{\rho}_{T-1,a}$ that only depend on currently
available information $(\mu_t, \sigma_t)$. Using the same notation for
counterfactual states to simplify the exposition, Lemma~\ref{lemma:mdp} shows
that subsequent counterfactual states will be governed by
\begin{subequations}
  \label{eqn:counterfactual-dynamics}
  \begin{align}
    \sigma_{v+1, a}^{-2}
    & \defeq \sigma_{v, a}^{-2} + s_a^{-2} b_v \bar{\rho}_{v, a}
      \label{eqn:counterfactual-dynamics-var} \\
    \mu_{v+1, a}
    & \defeq   \mu_{v, a} + \left( \sigma_{v, a}^2 - \sigma_{v+1, a}^2\right)^{\half} Z_{v, a}
      = \mu_{v, a} + \sigma_{v, a}\sqrt{\frac{b_v \bar{\rho}_{v, a}(\mu_t, \sigma_t)
      \sigma_{v, a}^{2}}{s_a^{2}+b_v\bar{\rho}_{v, a}(\mu_t, \sigma_t)\sigma_{v, a}^{2}}} Z_{v, a}.
      \label{eqn:counterfactual-dynamics-mean}
  \end{align}
\end{subequations}

First, we rewrite the objective~\eqref{eqn:rho} as a function of a single
Gaussian variable $\bar{Z}$. Recalling the identity~\eqref{eqn:var-decrease}
\begin{align*}
 \left( \sigma_{v, a}^2 - \sigma_{v+1, a}^2\right)^{\half} Z_{v, a}
  = \sigma_{v, a}\sqrt{\frac{b_v \bar{\rho}_{v, a}(\mu_t, \sigma_t)
  \sigma_{v, a}^{2}}{s_a^{2}+b_v\bar{\rho}_{v, a}(\mu_t, \sigma_t)\sigma_{v, a}^{2}}} Z_{v, a},
\end{align*}
use the recursive relation~\eqref{eqn:counterfactual-dynamics-var} to write
\begin{align*}
  \var\left( \sum_{v=t}^{T-1} \sigma_{v,a}\sqrt{\frac{b_v \bar{\rho}_{v, a}(\mu_t, \sigma_t)\sigma_{v, a}^{2}}
      {s_a^{2}+b_v\bar{\rho}_{v, a}(\mu_t, \sigma_t)\sigma_{v, a}^{2}}}Z_{v, a}\right)
  & = \sum_{v=t}^{T-1} \sigma_{v, a}^2 - \sigma_{v+1, a}^2
    = \sigma_{t, a}^2 - \sigma_{T, a}^2 \\
  & = \sigma_{t, a}^2 -  \left( \sigma_{t,a}^{-2}
    + \sum_{v=t}^{T-1} \frac{b_v \bar{\rho}_{v, a}(\mu_t, \sigma_t)}{s_a^2}  \right)^{-1} \\
  & =  \frac{\sigma_{t,a}^4 \sum_{v=t}^{T-1} b_{v}\bar{\rho}_{v, a}(\mu_t, \sigma_t)}
  {s_a^2 + \sigma_{t,a}^2 \sum_{v=t}^{T-1} b_{v}\bar{\rho}_{v, a}(\mu_t, \sigma_t)}.
\end{align*}
Thus, we conclude
\begin{align*}
  & \E \left[ \max_{a} \left\{ \mu_{t,a}
  + \sum_{v=t}^{T-1} \sigma_{v,a}\sqrt{\frac{\sigma_{v, a}^{2} b_v \bar{\rho}_{v, a}(\mu_t, \sigma_t)}
  {s_a^{2}+\sigma_{v, a}^{2} b_v \bar{\rho}_{v, a}(\mu_t, \sigma_t)}}Z_{v, a}
  \right\}~\Bigg|~ \mu_{t},\sigma_{t} \right] \\
  & = \E \left[ \max_{a} \left\{ \mu_{t,a}
  + \sqrt{\frac{\sigma_{t, a}^4 \sum_{v=t}^{T-1} \bar{\rho}_{v, a}(\mu_t, \sigma_t) b_{v}}
  {s_a^2 + \sigma_{t, a}^2 \sum_{v=t}^{T-1} \bar{\rho}_{v, a}(\mu_t, \sigma_t) b_{v}}} \bar{Z}_{a}
  \right\} ~\Bigg|~ \mu_{t},\sigma_{t} \right].
\end{align*}
Abusing notation, we replace $\bar{Z}$ with $Z_t$ in the final expectation.

Note that for any sequence of future allocations
$\bar{\rho}_{t},\ldots,\bar{\rho}_{T-1}$ that only depend on
$(\mu_t,\sigma_t)$, the change of variables
$\bar{\rho}_{a} = \frac{\sum_{v=t}^{T-1} \bar{\rho}_{v,a} b_{v}
}{\sum_{v=t}^{T-1} b_{v}}$ give
\[
  \begin{split}
  V^{\bar{\rho}_{t:T-1}}_{t}(\mu_{t},\sigma_{t})
  &= \E \left[ \max_{a} \left\{ \mu_{t,a}
  + \sqrt{\frac{\sigma_{t, a}^4 \sum_{v=t}^{T-1} \bar{\rho}_{v, a} b_{v}}
  {s_a^2 + \sigma_{t, a}^2 \sum_{v=t}^{T-1} \bar{\rho}_{v, a} b_{v}}} \bar{Z}_{a}
  \right\} ~\Bigg|~ \mu_{t},\sigma_{t} \right] \\
  &= \E \left[ \max_{a} \left\{ \mu_{t,a}
  + \sqrt{\frac{\sigma_{t, a}^4 \bar{\rho}_{a} \bar{b}_{t}}
  {s_a^2 + \sigma_{t, a}^2 \bar{\rho}_{a} \bar{b}_{t}}} \bar{Z}_{a}
  \right\} ~\Bigg|~ \mu_{t},\sigma_{t} \right] \\
  &= V^{\bar{\rho}}_{t}(\mu_{t},\sigma_{t})
  \end{split}
\]
where $\bar{b}_{t} = \sum_{v=t}^{T-1} b_{v}$ and
$V^{\bar{\rho}}_{t}(\mu_{t},\sigma_{t})$ is the value function of the constant
allocation $\bar{\rho}$.  Thus, for any sequence of future allocations that
only depend on $(\mu_t,\sigma_t)$, there exists a constant allocation
$\bar{\rho}(\mu_t, \sigma_t)$ that achieves the same performance.

\subsection{Proof of Proposition~\ref{prop:rho-vs-static}}
\label{section:proof-rho-vs-static}

By Lemma~\ref{lemma:rho-reduction}, for any future allocation
$\bar{\pi}_{t:T} = (\bar{\pi}_t(\mu_{t}, \sigma_{t}), \ldots,
\bar{\pi}_{T-1}(\mu_{t}, \sigma_{t}))$ that only depends on currently
available information $(\mu_{t},\sigma_{t}))$, there is a constant
allocation that matches the same performance.  Thus, it is sufficient to show
that the value function for $\rho$ dominates the value function of any constant
allocation
$\bar{\pi}_{v}(\mu_{t}, \sigma_{t}) \equiv \bar{\rho}(\mu_{t}, \sigma_{t})$ for $v \ge t$.  Proceeding by induction, observe for the base case that
when $t = T-1$, the definition~\eqref{eqn:rho} of the policy $\rho_t$ implies
\begin{equation*}
  V^{\rho}_{T-1}(\mu_{T-1},\sigma_{T-1})
  = \max_{\bar{\rho} \in \Delta_{\numarm}} V^{\bar{\rho}}_{T-1}(\mu_{T-1},\sigma_{T-1})
\end{equation*}
for all $(\mu_{T-1}, \sigma_{T-1})$.

Next, as an inductive hypothesis, suppose
\begin{equation*}
V^{\rho}_{t+1}(\mu_{t+1},\sigma_{t+1}) \geq \max_{\bar{\rho} \in
  \Delta_{\numarm}} V^{\bar{\rho}}_{t+1}(\mu_{t+1},\sigma_{t+1})
~~~\mbox{for all}~~(\mu_{t+1}, \sigma_{t+1}).
\end{equation*}
 Then, for any
$(\mu_{t}, \sigma_{t})$
\begin{align}
  V^{\rho}_{t}(\mu_{t}, \sigma_{t}) 
  = \E_{t} \left[V^{\rho}_{t+1}\left(\mu_{t+1}, \sigma_{t+1}\right)\right] 
  & \geq \E_{t} \left[ \max_{\bar{\rho} \in \Delta_{\numarm}} V^{\bar{\rho}}_{t+1}
  \left(\mu_{t+1}, \sigma_{t+1}\right)\right] \nonumber \\
  & \ge \max_{\bar{\rho} \in \Delta_{\numarm}} \E_{t} \left[  V^{\bar{\rho}}_{t+1}
    \left(\mu_{t+1}, \sigma_{t+1}\right)\right]
      \label{eqn:exp-max-interchange}
\end{align}
where we abuse notation to denote by $(\mu_{t+1}, \sigma_{t+1})$ the state
transition under the policy $\rho_t$~\eqref{eqn:dynamics}
\begin{subequations}
  \label{eqn:dynamics-rho-t}
  \begin{align}
    \sigma_{t+1, a}^{-2}
    & \defeq \sigma_{t, a}^{-2} + s_a^{-2} b_t \rho_{t, a}(\mu_t, \sigma_t)
    \label{eqn:dynamics-rho-t-var} \\
    \mu_{t+1, a}  
    & \defeq   \mu_{t, a} + \left( \sigma_{t, a}^2 - \sigma_{t+1, a}^2\right)^{\half} Z_{t, a}
      = \mu_{t, a} + \sigma_{t, a}\sqrt{\frac{b_t \rho_{t, a}(\mu_t, \sigma_t)
      \sigma_{t, a}^{2}}{s_a^{2}+b_t\rho_{t, a}(\mu_t, \sigma_t)\sigma_{t, a}^{2}}} Z_{t, a}.
              \label{eqn:dynamics-rho-t-mean}
  \end{align}
\end{subequations}

We now derive a more explicit representation for
$\E_{t} \left[ V^{\bar{\rho}}_{t+1} \left(\mu_{t+1},
    \sigma_{t+1}\right)\right]$, which is the value function for a policy that
follows $\rho_t$ at time $t$ and the constant policy $\bar{\rho}$ onwards.  By
the expression~\eqref{eqn:rho} in Lemma~\ref{lemma:rho-reduction}, we have
\begin{align*}
  \E_{t} \left[ V^{\bar{\rho}}_{t+1} \left(\mu_{t+1}, \sigma_{t+1}\right)\right]
  =  \E_{t} \left[ \max_{a} \left\{ \mu_{t+1,a}
          + \sqrt{\frac{\sigma_{t+1, a}^4 \bar{\rho}_{a} \bar{b}_{t+1}}
            {s_a^2 + \sigma_{t+1, a}^2 \bar{\rho}_{a} \bar{b}_{t+1}}} Z_{t+1, a}
        \right\} \right].
\end{align*}
Plugging in the state transition~\eqref{eqn:dynamics-rho-t-mean}, we have
\begin{align*}
  \E_{t} \left[ V^{\bar{\rho}}_{t+1} \left(\mu_{t+1}, \sigma_{t+1}\right)\right]
  =  \E_{t} \left[ \max_{a} \left\{ \mu_{t,a} 
  + \left( \sigma_{t, a}^2 - \sigma_{t+1, a}^2\right)^{\half} Z_{t, a}
          + \sqrt{\frac{\sigma_{t+1, a}^4 \bar{\rho}_{a} \bar{b}_{t+1}}
            {s_a^2 + \sigma_{t+1, a}^2 \bar{\rho}_{a} \bar{b}_{t+1}}} Z_{t+1, a}
        \right\} \right].
\end{align*}
Noting that
\begin{align*}
  \var_t \left(
  \left( \sigma_{t, a}^2 - \sigma_{t+1, a}^2\right)^{\half} Z_{t, a}
          + \sqrt{\frac{\sigma_{t+1, a}^4 \bar{\rho}_{a} \bar{b}_{t+1}}
            {s_a^2 + \sigma_{t+1, a}^2 \bar{\rho}_{a} \bar{b}_{t+1}}} Z_{t+1, a}
  \right) =
  \frac{\sigma_{t, a}^4 \left( \bar{\rho}_{a} \bar{b}_{t+1}
  + \rho_{t, a}(\mu_{t}, \sigma_{t}) b_{t} \right)}
  {s_a^2 + \sigma_{t, a}^2 \left(
  \bar{\rho}_{a} \bar{b}_{t+1} +\rho_{t, a}(\mu_{t}, \sigma_{t}) b_{t} \right)},
\end{align*}
we arrive at the identity
\begin{align*}
  \E_{t} \left[ V^{\bar{\rho}}_{t+1} \left(\mu_{t+1}, \sigma_{t+1}\right)\right]
  =  \E_{t} \left[ \max_{a} \left\{ \mu_{t,a} 
  + \sqrt{\frac{\sigma_{t, a}^4 \left( \bar{\rho}_{a} \bar{b}_{t+1}
  + \rho_{t}(\mu_{t}, \sigma_{t}) b_{t} \right)}
  {s_a^2 + \sigma_{t, a}^2 \left(
  \bar{\rho}_{a} \bar{b}_{t+1} +\rho_{t}(\mu_{t}, \sigma_{t}) b_{t} \right)}}
  Z_{t, a} \right\} \right].
\end{align*}

Recalling the bound~\eqref{eqn:exp-max-interchange}, use
$\bar{\rho} = \rho_{t}(\mu_{t}, \sigma_{t})$ to conclude
\begin{align*}
  V^{\rho}_{t}(\mu_{t}, \sigma_{t})
  & \ge \max_{\bar{\rho} \in \Delta_{\numarm}}
  \E_{t} \left[ \max_{a} \left\{ \mu_{t,a} 
  + \sqrt{\frac{\sigma_{t, a}^4 \left( \bar{\rho}_{a} \bar{b}_{t+1}
  + \rho_{t}(\mu_{t}, \sigma_{t}) b_{t} \right)}
  {s_a^2 + \sigma_{t, a}^2 \left(
  \bar{\rho}_{a} \bar{b}_{t+1} +\rho_{t}(\mu_{t}, \sigma_{t}) b_{t} \right)}}
  Z_{t, a} \right\} \right] \\
  & \ge \E_{t} \left[ \max_{a} \left\{ \mu_{t,a} 
  + \sqrt{\frac{\sigma_{t, a}^4 \rho_{t, a}(\mu_{t}, \sigma_{t}) \bar{b}_{t} }
  {s_a^2 + \sigma_{t, a}^2  \rho_{t, a}(\mu_{t}, \sigma_{t}) \bar{b}_{t}  }}
    Z_{t, a} \right\} \right]  \\
  & = \max_{\bar{\rho} \in \Delta^\numarm} V_t^{\bar{\rho}} (\mu_t, \sigma_t),
\end{align*}
where we used the definition of $\rho_t$ in expression~\eqref{eqn:rho} in
the final line.

\subsection{Distilled Residual Horizon Optimization}
\label{section:distilled}

Instead of computing
$\rho_t(\mu_t, \sigma_t)$ as the experiment progresses, it can sometimes be
convenient to pre-compute the mapping $\rho_t(\cdot, \cdot)$ so that it can be
readily applied for any observed current state.  By reformulating the
optimization problem~\eqref{eqn:rho} as a stochastic optimization problem over
\emph{functions of $(\mu_t, \sigma_t)$}, we propose a \emph{distilled} variant
of the $\algo$ policy that can learn the mapping $\rho_t(\cdot,\cdot)$ in a
fully offline manner.
\begin{proposition}
  \label{prop:loss-min}
  The solution to the $\algofull$ problem~\eqref{eqn:rho} coincides with the
  maximizer of the following reward maximization problem over measurable
  functions $\rho(\mu_t, \sigma_t) \in \Delta^{\numarm}$
  \begin{equation}
    \label{eqn:loss-min}
    \maximize_{\rho(\cdot)~\mbox{\scriptsize measurable}}
    ~\E_{(\mu_t, \sigma_t) \sim P_t} \left[ \max_{a} \left\{ \mu_{t,a}
        + \sqrt{\frac{\sigma_{t, a}^4 \rho(\mu_t, \sigma_t) \bar{b}_t}
        {s_a^2 + \sigma_{t, a}^2 \rho(\mu_t, \sigma_t) \bar{b}_t}} Z_{t, a}
      \right\} \right],
  \end{equation}
  where $P_t$ is any probability measure with support
  $\R^\numarm \times (0, \sigma_0^2]^\numarm$ and
  $\bar{b}_t \defeq \sum_{v=t}^{T-1} b_v$.
\end{proposition}

Proposition~\ref{prop:loss-min} allows us to use standard ML best practices to
(approximately) solve the reward maximization problem~\eqref{eqn:rho}. We
parameterize the policy using black-box ML models
\begin{equation*}
  \left\{\rho_{\theta,t}(\mu_t, \sigma_t): \theta \in \Theta_t\right\},
\end{equation*}
and use stochastic gradient-based optimization algorithms to solve for the
optimal parameterization (Algorithm~\ref{alg:rho}). In particular, the model
training problem~\eqref{eqn:loss-min} can be approximated by neural networks
and optimized through standard auto-differentiation
frameworks~\cite{TensorFlow15, PyTorch19} for computing stochastic
(sub)gradients. 
By the envelope theorem, the stochastic (sub)gradient at the
sample $(\mu_t, \sigma_t) \sim P_t$ is given by
\begin{equation*}
  Z_{t, a\opt} ~\partial_{\theta} \left\{
    \sqrt{\frac{\sigma_{t, a\opt}^4  \rho_{\theta,t} (\mu_t, \sigma_t) \bar{b}_t}
    {s_{a\opt}^2 + \sigma_{t, a\opt}^2 \rho_{\theta,t}(\mu_t, \sigma_t) \bar{b}_t}}
  \right\},
\end{equation*}
where
$a\opt \in \argmax_{a} \left\{ \mu_{t,a} + \sqrt{\frac{\sigma_{t, a}^4
      \rho_{\theta,t}(\mu_t, \sigma_t) \bar{b}_t} {s_a^2 + \sigma_{t, a}^2
      \rho_{\theta,t}(\mu_t, \sigma_t) \bar{b}_t}} Z_{t, a} \right\}$.


\begin{proof}
It remains to show that minimizing over all measurable functions
$\rho(\mu_t, \sigma_t) \in \Delta^{\numarm}$ gives pointwise suprema for each
$(\mu_t, \sigma_t)$ almost surely
\begin{align*}
  & \sup_{\rho(\cdot)~\mbox{\scriptsize measurable}} ~\E_{(\mu_t, \sigma_t) \sim P_t} \left[ \max_{a} \left\{ \mu_{t,a}
        + \sqrt{\frac{\sigma_{t, a}^4 \rho(\mu_t, \sigma_t) \bar{b}_t}
        {s_a^2 + \sigma_{t, a}^2 \rho(\mu_t, \sigma_t) \bar{b}_t}} Z_{t, a}
    \right\} \right] \\
  & = \E_{(\mu_t, \sigma_t) \sim P_t} \left[ \sup_{\rho}
    \E\left[ \max_{a} \left\{ \mu_{t,a}
        + \sqrt{\frac{\sigma_{t, a}^4 \rho \bar{b}_t}
        {s_a^2 + \sigma_{t, a}^2 \rho \bar{b}_t}} Z_{t, a}
      \right\} ~\Bigg|~ \mu_t, \sigma_t \right] \right].
\end{align*}
We use normal integrand theory~\cite[Section 14.D]{RockafellarWe98} to
interchange the integral and supremum over measurable mappings. Recall that a
map
$f: \Delta^{\numarm} \times (\R^{\numarm} \times (0, \sigma_0]^\numarm) \to
\bar{\R}$ is a normal integrand if its epigraphical mapping---viewed as a
set-valued
mapping---$S_f: \R^{\numarm} \times (0, \sigma_0]^\numarm \to \R \times
\R,~(\mu_t, \sigma_t) \mapsto \epi f(\cdot; \mu_t, \sigma_t) = \{ (\rho,
\alpha) \in \Delta^{\numarm} \times \R: f(\rho; \mu_t, \sigma_t) \le \alpha\}$ is
closed-valued. That is, $S_f(\mu_t, \sigma_t)$ is closed for all
$(\mu_t, \sigma_t) \in \R^{\numarm} \times (0, \sigma_0]^\numarm$ and
measurable (i.e., for any open set
$O \subseteq \R^{\numarm} \times (0, \sigma_0]^\numarm$,
$S_f^{-1}(O) \defeq \cup_{o \in O}~ S_f^{-1}(o)$ is measurable).
\begin{lemma}[{\citet[Theorem 14.60]{RockafellarWe98}}]
  \label{lemma:inf-int-interchange}
  If
  $f: \Delta^{\numarm} \times (\R^{\numarm} \times (0, \sigma_0]^\numarm) \to
  \bar{\R}$ is a normal integrand, and
  $ \int_{\R^{\numarm} \times (0, \sigma_0]^\numarm} f(\rho(\mu_t, \sigma_t);
  \mu_t, \sigma_t) ~ dP_t(\mu_t, \sigma_t) < \infty$ for some measurable
  $\rho(\cdot)$, then
  \begin{align*}
    & \inf_{\rho(\cdot)} \left\{
      \int_{\R^{\numarm} \times (0, \sigma_0]^\numarm} f(\rho(\mu_t, \sigma_t); \mu_t, \sigma_t)
      ~d P_t(\mu_t, \sigma_t) ~\Big|~
      \rho: \R^{\numarm} \times (0, \sigma_0]^\numarm \to \R~\mbox{measurable}
      \right\} \\
    & = \int_{\R^{\numarm} \times (0, \sigma_0]^\numarm}
      \inf_{\rho \in \Delta^\numarm} f(\rho; \mu_t, \sigma_t) ~ dP_t(\mu_t, \sigma_t).
  \end{align*}
  If this common value is not $-\infty$, a measurable function
  $ \rho^*: \R^{\numarm} \times (0, \sigma_0]^\numarm \to \R$ attains the
  minimum of the left-hand side iff
  $ \rho^*(\mu_t, \sigma_t) \in \argmin_{\rho \in \Delta^\numarm} f(\rho; \mu_t,
  \sigma_t)$ for $P_t$-almost every
  $(\mu_t, \sigma_t) \in \R^{\numarm} \times (0, \sigma_0]^\numarm$.
\end{lemma}
\noindent Define the function
\begin{equation*}
  f(\rho; \mu_t, \sigma_t) \defeq - \E_{Z_t \sim N(0,I)}\left[ \max_{a} \left\{ \mu_{t,a} +
    \sqrt{\frac{\sigma_{t, a}^4 \rho_a \bar{b}_t} {s_a^2 + \sigma_{t, a}^2 \rho_a
        \bar{b}_t}} Z_{t, a} \right\} \right].
\end{equation*}
Since $f$ is continuous in $(\rho, \mu_t, \sigma_t)$ by dominated convergence,
$f$ is a normal integrand~\cite[Examples 14.31]{RockafellarWe98}.  Applying
Lemma~\ref{lemma:inf-int-interchange}, we obtain the desired result.
\end{proof}


\section{Proofs of large horizon results}
\label{section:proof-asymptotic}

\subsection{Proof of Theorem~\ref{theorem:asymptotic-rho}}
\label{section:proof-asymptotic-rho}

The proof of Theorem~\ref{theorem:asymptotic-rho} first relies on a
characterization of the gradient of the objective as $T - t \to \infty$.

\begin{proposition}
  \label{prop:grad-hess}
  Consider any fixed epoch $t$ and state $(\mu, \sigma)$.
  Let $\theta_{a} := \mu_a + \sigma_{a}Z_{a}$ be the random variable
  representing the posterior belief of the average reward for arm $a$,
  and let $\theta_{a}^{*} := \max_{a' \neq a} \theta_{a'}$.
  Suppose that as $T-t$ grows large, 
  the residual sample budget $\bar{b}_{t} = \sum_{s=t}^{T-1} b_{s}\to \infty$.
  Then gradient and Hessian of the planning objective $V^{\bar{\rho}}_{t}(\mu,\sigma)$
  with respect to the allocation $\bar{\rho}$ converge as follows:
    \begin{align*}
    \lim_{T \to \infty} \bar{b}_{t} \nabla V^{\bar{\rho}}_{t}(\mu,\sigma)
     &= \textup{diag} \left( \frac{s_{a}^{2}}{2\bar{\rho}_{a}^{2}}
     \E \left[\frac{1}{\sigma_{a}} \phi \left( \frac{\theta_{a}^{*}-\mu_{a}}{\sigma_{a}} \right) \right] \right) \\
     \lim_{T \to \infty} \bar{b}_{t} \nabla^{2} V^{\bar{\rho}}_{t}(\mu,\sigma)
     &= - \textup{diag} \left(\frac{s_{a}^{2}}{2\bar{\rho}_{a}^{3}}
     \E \left[\frac{1}{\sigma_{a}} \phi \left( \frac{\theta_{a}^{*}-\mu_{a}}{\sigma_{a}} \right) \right] \right)
    \end{align*}
  where $\phi(\cdot)$ is the standard normal density.
\end{proposition}

\begin{proof}
We characterize the asymptotic behavior of the gradient and hessian of
$V_{t}^{\bar{\rho}}(\mu,\sigma) =
\mathbb{E}_{t}\left[\max_{a}\mu_{T,a}\right]$ with respect to $\bar{\rho}$ as
the residual horizon $T-t$ grows large.  The residual horizon affects the
planning problem through the residual batch size $\bar{b}_{t}$, which is
assumed to converge to infinity as $T-t\to \infty$.  Rewrite the planning
objective $V_{t}^{\bar{\rho}}(\mu,\sigma)$ with a rescaled measurement
variance $s_{a}^{2}/\bar{b}_{t}$
\begin{align*}
  V_{t}^{\bar{\rho}}(\mu,\sigma) =\E_{t} \left[ \max_{a} \left\{ \mu_{t,a}
  + \sqrt{\frac{\sigma_{t, a}^4 \bar{\rho}_{a}}
  {s_a^2/\bar{b}_{t} + \sigma_{t, a}^2 \bar{\rho}_{a} }} Z_{t, a}  \right\} \right],
\end{align*}
It is sufficient to understand the behavior of
$V_{t}^{\bar{\rho}}(\mu,\sigma)$ as the measurement variance goes to zero.
For ease of notation we define $\alpha := 1/\bar{b}_{t}$, and consider
$\alpha$ as a scaling parameter of the measurement variance that converges to
zero.

Observe $V_{t}^{\bar{\rho}}(\mu,\sigma)$ can be expressed as the composition
$V_{t}^{\bar{\rho}}(\mu,\sigma)=g\left(\varphi\left(\bar{\rho},\sigma^{2},\alpha
    s^{2} \right)\right)$ where 
\[
\begin{aligned}
  g(v) &\defeq \E[\max_{a}\mu_{a}+v_{a}Z_{a}] \\
  \varphi_{a}(\bar{\rho}_{a},\sigma_{a}^{2},s_{a}^{2})&\defeq\sigma_{a}\sqrt{\frac{\bar{\rho}_{a}\sigma_{a}^{2}}{s_{a}^{2}+\bar{\rho}_{a}\sigma_{a}^{2}}}
\end{aligned}
\]
and $\varphi, \sigma^{2}, s^{2}$ refer to the vectorized versions. The
function $g(v)$ describes the behavior of the max of Gaussian random variables
under standard deviations $v$, and $\varphi$ characterizes how the standard
deviation of the update changes with the sampling probabilities $\bar{\rho}$.

Through the chain rule, it is sufficient to characterize the gradient and
hessian of $g,\varphi$.  \citet{FrazierPo10} derived expressions for the
gradient and hessian of $g(v)$.
\begin{lemma}[{\citet{FrazierPo10}}]
  \label{lemma:frazier} Letting
  $W_{a}\defeq \max_{a'\neq a}\mu_{a'}+v_{a'}Z_{a'}$,
  \begin{align*}
    \frac{\partial}{\partial v_{a}}g(v)
    & = \mathbb{E}\left[Z_{a}\indic{a=\arg\max_{a'}\mu_{a'}+v_{a'}Z_{a'}}\right]
      = \mathbb{E}\left[\phi\left(\frac{W_{a}-\mu_{a}}{v_{a}}\right)\right] \\
    \frac{\partial^{2}}{\partial v_{a}^{2}}g(v)
    & = \mathbb{E}\left[\left(\frac{W_{a}-\mu_{a}}{v_{a}^{3}}\right)
      \phi\left(\frac{W_{a}-\mu_{a}}{v_{a}}\right)\right] \\
    \frac{\partial^{2}}{\partial v_{a}\partial v_{a'}}g(v)
    & = \mathbb{E}\left[-\left(\frac{W_{a}-\mu_{a}}{v_{a}}\right)
      \phi\left(\frac{W_{a}-\mu_{a}}{v_{a}}\right)Z_{a'}
      \indic{a'=\argmax_{\hat{a}\neq a}
      \mu_{\hat{a}}+v_{\hat{a}}Z_{\hat{a}}}\right]
  \end{align*}
\end{lemma}
\noindent The first and second derivatives of $\varphi$ are given by
\[ \frac{\partial}{\partial\bar{\rho}_{a}}\varphi_{a}(\bar{\rho}_{a},\sigma_{a}^{2},\alpha s_{a}^{2}) 
  = \frac{\alpha s_{a}^{2} \sigma_{a}^{2}}{2\bar{\rho}_{a}^{1/2}(\alpha s_{a}^{2} + \bar{\rho}_{a}\sigma_{a}^{2})^{3/2}},
  \qquad \frac{\partial^{2}}{\partial\bar{\rho}_{a}^{2}}\varphi_{a}(\bar{\rho}_{a},\sigma_{a}^{2},\alpha s_{a}^{2}) = -\frac{\sigma_{a}^{2}(\alpha^2 s_{a}^{4} + 4\bar{\rho}_{a} \alpha s_{a}^{2}\sigma_{a}^{2})}{4\bar{\rho}_{a}^{3/2}(\alpha s_{a}^{2} + \bar{\rho}_{a}\sigma_{a}^{2})^{5/2}}.
\]

We first study the behavior of $\varphi$ and its derivatives as
$\alpha\to0$
\[
\varphi_{a}(\bar{\rho}_{a},\sigma_{a}^{2},\alpha s_{a}^{2}) \to \sigma_{a},
\qquad \frac{\partial}{\partial\bar{\rho}_{a}}\varphi_{a}(\bar{\rho}_{a},\sigma_{a}^{2},\alpha s_{a}^{2})\sim\frac{\alpha s_{a}^{2}}{2\bar{\rho}_{a}^{2}\sigma_{a}},
\qquad \frac{\partial^{2}}{\partial\bar{\rho}_{a}^{2}}\varphi_{a}(\bar{\rho}_{a},\sigma_{a}^{2}, \alpha s_{a}^{2})\sim-\frac{\alpha s_{a}^{2}}{\bar{\rho}_{a}^{3}\sigma_{a}}.
\]
To analyze $g(\varphi (\bar{\rho},\sigma^{2},\alpha s^{2}))$, recall that
$\theta_{a} = \mu_{a} + \sigma_{a}Z_{a}$ is the random variable representing
the posterior belief of the average reward of arm $a$.  For any realization of
$Z\sim N(0,I)$, we have
\[
  W_{a} = \max_{a'\neq a}\left\{ \mu_{a'}+\varphi_{a'}(\bar{\rho}_{a'},\sigma_{a'}^{2},\alpha s_{a'}^{2})Z_{a'} \right\} \to \theta_{a}^{*},
\]
so dominated convergence implies
\begin{align*}
  \lim_{\alpha \to 0}\frac{\partial}{\partial v_{a}}
  g(\varphi(\bar{\rho}, \sigma^{2}, \alpha s^2))
  & = \E \left[
    \phi\left(\frac{\theta_{a}^{*} - \mu_{a}}{\sigma_{a}}\right)
    \right] \\
  \lim_{\alpha \to 0}
  \frac{\partial^{2}}{\partial v_{a}^{2}}
  g(\varphi(\bar{\rho}, \sigma^{2}, \alpha s^2))
  & = \E \left[
    \frac{\theta_{a}^{*} - \mu_{a}}{\sigma_{a}^{3}} 
    \cdot \phi\left(\frac{\theta_{a}^{*}-\mu_{a}}{\sigma_{a}}\right)
    \right] \\
  \lim_{\alpha \to 0}
  \frac{\partial^{2}}{\partial v_{a}\partial v_{a'}}
  g(\varphi(\bar{\rho}, \sigma^{2}, \alpha s^2))
  & = \E \left[
    - \frac{\theta_{a}^{*}-\mu_{a}}{\sigma_{a}} \cdot
    \phi\left(\frac{\theta_{a}^{*}-\mu_{a}}{\sigma_{a}}\right)
    Z_{a'} \indic{a'=\arg\max_{\hat{a}\neq a}\theta_{\hat{a}}}
    \right]
\end{align*}

By the chain rule,
$\nabla_{\bar{\rho}} V_{t}^{\bar{\rho}}(\mu, \sigma) =
\nabla_{v}g(\varphi(\mu, \sigma^{2}, \alpha s^{2})) \nabla_{\bar{\rho}}
\varphi (\bar{\rho},\sigma^{2},\alpha s^{2}) $ where the Jacobian
$\nabla_{\bar{\rho}} \varphi$ is
\[
  \nabla_{\bar{\rho}} \varphi (\bar{\rho},\sigma^{2},\alpha s^{2}) := \text{diag} \left( \frac{\partial}{\partial\bar{\rho}_{a}}\varphi_{a}(\bar{\rho}_{a},\sigma_{a}^{2},\alpha s_{a}^{2}) \right).
\]
Rescaling the planning objective $V_{t}^{\bar{\rho}}$ by $1/\alpha$, we obtain
the following limit
\[
  \lim_{\alpha \to \infty}  \alpha^{-1}
  \nabla_{\bar{\rho}} V^{\bar{\rho}}_{t}(\mu, \sigma)
  = \frac{s_{a}^{2}}{2\bar{\rho}_{a}^{2}\sigma_{a}}
  \E \left[\phi\left(\frac{\theta_{a}^{*} - \mu_{a}}{\sigma_{a}}\right)\right]
\]
As for the Hessian, the chain rule implies
\begin{align*}
  \frac{\partial^{2} V_{t}^{\bar{\rho}}}{\partial \bar{\rho}_{a}\partial \bar{\rho}_{a'}}
  &= \sum_{i} \frac{\partial g(\varphi(\bar{\rho}))}{\partial v_{i}}
    \frac{\partial^{2} \varphi_{i}}{\partial \bar{\rho}_{a} \partial \bar{\rho}_{a'}}
    + \sum_{i,j} \frac{\partial g(\varphi(\bar{\rho})) }{\partial v_{i}\partial v_{j}} 
    \frac{\partial \varphi_{i}}{\partial \bar{\rho}_{a}}
    \frac{\partial \varphi_{j}}{\partial \bar{\rho}_{a'}} \\
  &= \frac{\partial g(\varphi(\bar{\rho}))}{\partial v_{a}}
    \frac{\partial^{2} \varphi_{a}}{\partial \bar{\rho}_{a}^{2}} 1\{ a = a' \}
    + \frac{\partial g(\varphi(\bar{\rho})) }{\partial v_{a}\partial v_{a'}} 
    \frac{\partial \varphi_{a}}{\partial \bar{\rho}_{a}}
    \frac{\partial \varphi_{a'}}{\partial \bar{\rho}_{a'}},
\end{align*}
where we suppress the dependence of $\varphi$ on other arguments for
succinctness. Rewriting this using
$\nabla_{\bar{\rho}}^{2}\varphi := \diag \left( \frac{\partial^{2}
    \varphi_{a}}{\partial \bar{\rho}_{a}^{2}} \right)$, we get
\[  
  \nabla_{\bar{\rho}}^{2}V^{\bar{\rho}}_{t}(\mu,\sigma)
 = \text{diag}(\nabla_{v} g(\varphi(\bar{\rho}))) \nabla_{\bar{\rho}}^{2}\varphi + \nabla_{\bar{\rho}}\varphi \nabla_{v}^{2}g(\varphi(\bar{\rho}))\nabla_{\bar{\rho}}\varphi.
\]
Collecting the above results, we have
\begin{align*}
  \lim_{\alpha \to 0} \alpha^{-1} 
  \text{diag}(\nabla_{v} g(\varphi(\bar{\rho}))) \nabla_{\bar{\rho}}^{2}\varphi
  &= -\text{diag} \left( \frac{ s_{a}^{2}}{\bar{\rho}_{a}^{3}\sigma_{a}} 
      \E \left[\phi\left(\frac{\theta_{a}^{*} - \mu_{a}}{\sigma_{a}}\right)\right] \right)\\
  \lim_{\alpha \to 0} \alpha^{-1}  \nabla_{\bar{\rho}}\varphi \nabla_{v}^{2}g(\varphi(\bar{\rho}))\nabla_{\bar{\rho}}\varphi
  &= \lim_{\alpha \to 0} \alpha^{-1}
  \text{diag} \left(  \frac{\alpha s_{a}^{2}}{2 \bar{\rho_{a}}^{2} \sigma_{a}} \right)
  \nabla_{v}^{2} g(\varphi(\bar{\rho}))
  \text{diag} \left( \frac{ \alpha s_{a}^{2}}{2 \bar{\rho_{a}}^{2} \sigma_{a}} \right)
  = 0,
\end{align*}
which gives the final result.

\end{proof}

This leads us to the following corollary, which uses the characterization of the Hessian to show that 
the objective becomes strongly concave as $T-t \to \infty$.

\begin{corollary}
  \label{cor:strong-concavity}
  Let $\Delta_K^\epsilon = \Delta_K \cap \{p: p \ge \epsilon\}$ be the
  truncated simplex. For a fixed epoch $t$ and state $(\mu, \sigma)$, there
  exists $T_0(\epsilon) > 0$ such that $\forall T - t > T_0$,
  $\bar{\rho} \mapsto \bar{b}_{t}V^{\bar{\rho}}_{t}(\mu,\sigma)$ is strongly
  concave on $\Delta_K^\epsilon$.
\end{corollary}


\begin{proof}
In the proof of Proposition~\ref{prop:grad-hess}, we showed that as the
(scaled) measurement variance $\alpha s_{a}^{2}$ converges to zero, the
Hessian $\alpha^{-1} \nabla_{\bar{\rho}}^{2}V^{\bar{\rho}}_{t}(\mu,\sigma)$
can be expressed as a negative-definite diagonal matrix plus a matrix that
converges to zero. As long as the sampling probabilities are bounded below, the
convergence is uniform. We have the following bound for the largest eigenvalue
of the Hessian
\begin{align*}
  \lambda_{\max}\left(
  \nabla_{\bar{\rho}}^{2} \alpha^{-1}
  V^{\bar{\rho}}_{t}(\mu,\sigma) \right)
  & \leq \lambda_{\max}\left( \alpha^{-1}
    \text{diag}(\nabla_{v} g(\varphi(\bar{\rho})))
    \nabla_{\bar{\rho}}^{2}\varphi \right)
    +\lambda_{\max}\left(\alpha^{-1} \nabla_{\bar{\rho}}\varphi
    \nabla_{v}^{2}g(\varphi(\bar{\rho}))\nabla_{\bar{\rho}}\varphi
    \right) \\
  & \leq -\min_{a} \frac{s_{a}^{2}}{2\sigma_{a}\epsilon^3}
    \E \left[\phi\left(\frac{\theta_{a}^{*} - \mu_{a}}{\sigma_{a}}\right)\right]
    + \alpha C/\epsilon^{4},
\end{align*}
where $C$ is a constant that depends only on $\mu,\sigma,s$. So as
$\alpha \to 0$, the Hessian will be negative-definite for all
$\bar{\rho}\in \Delta_{K}^{\epsilon}$. Thus, for a large enough residual batch
size $\bar{b}_{t}$ the planning objective will be strongly concave on the
truncated simplex $\Delta_{K}^{\epsilon}$ for some threshold $\epsilon$.
\end{proof}




Finally, by strong concavity, the KKT conditions of the planning problem have a unique solution.
By the implicit function theorem, we can characterize the limit of this solution.

\begin{proposition}
  \label{prop:dts}
  Consider a fixed epoch $t$ and state $(\mu, \sigma)$. Suppose there exists
  $T_1$ such that for $\forall T - t > T_1$, the $\algo$ allocation satisfies
  $\rho_{t,a}(\mu,\sigma) > \epsilon$. Then as $T-t \to \infty$,
  \begin{equation*}
    \rho_{t,a}(\mu,\sigma) \to \pi_{a}^{\textup{DTS}}(\mu, \sigma),~\mbox{where}~
    \pi^{\rm DTS}_a(\mu, \sigma) \propto s_a \left[ \frac{\partial}{\partial
        \mu_a}\pi_{a}^{\textup{TS}}(\mu,\sigma) \right]^{1/2}
  \end{equation*}
\end{proposition}

\begin{proof}
As in the proof of Proposition~\ref{prop:grad-hess}, we let
$\alpha = \bar{b}_{t}^{-1}$. Let $\alpha_{0}$ to be the threshold such that
for $\alpha < \alpha_{0}$ the objective is strongly concave (such a threshold
exists from Corollary~\ref{cor:strong-concavity}), and let
$\alpha_{1} < \alpha_{0}$ to be the threshold such that for
$\alpha < \alpha_{1}$ the unique solution is in the interior of
$\Delta_{K}^{\epsilon}$ (such a threshold exists by hypothesis).

For $\alpha$ small enough, the (scaled) KKT conditions give
\begin{equation}
  \label{eqn:opt_foc}
  \alpha^{-1} \nabla_{\bar{\rho}} V^{\bar{\rho}}_{t}(\mu,\sigma)
  - \lambda \onevec = 0, \quad \onevec^\top \rho_t = 1,
\end{equation}
where $\lambda$ is the optimal dual variable for the equality constraint.  The
Jacobian of this equation with respect to $(\bar{\rho}, \lambda)$ is the
following block matrix
\[
\begin{bmatrix}
  \nabla^{2}_{\bar{\rho}} \alpha^{-1}V^{\bar{\rho}}_{t}(\mu,\sigma) & -\onevec\\
  \onevec^{\top} & 0
\end{bmatrix}
\]
Since $\alpha < \alpha_{0}$, the Hessian of
$\alpha^{-1}V^{\bar{\rho}}_{t}(\mu,\sigma)$ is negative definite and
invertible in a neighborhood of $\alpha = 0$.  The Jacobian is also invertible
since the Schur complement
\begin{equation*}
  0 - \onevec^\top \alpha^{-1} \nabla^{2}_{\bar{\rho}}\E\left[\max_{a}\mu_{T,a}\right] (-\onevec) < 0
\end{equation*}
is invertible.  Moreover,
$\alpha^{-1} \nabla_{\bar{\rho}} V^{\bar{\rho}}_{t}(\mu,\sigma)$ is
continuously differentiable in $\alpha$ at $\alpha = 0$.  So by the Implicit
Function Theorem, there exists a $\alpha_2$ such that for all
$\alpha < \alpha_2$, the solution satisfying the KKT
conditions~\eqref{eqn:opt_foc}, $\bar{\rho}(\alpha)$ and $\lambda(\alpha))$,
is continuous in $\alpha$.  Thus, as $\alpha \to 0$, the unique maximizer of
$\alpha^{-1}V^{\bar{\rho}}_{t}(\mu,\sigma)$ converges to $\bar{\rho}\opt$
satisfying the limiting version of the KKT conditions~\eqref{eqn:opt_foc}.
Using the explicit expression for
$\alpha^{-1} \nabla_{\bar{\rho}} V^{\bar{\rho}}_{t}(\mu,\sigma)$ obtained in
Proposition~\ref{prop:grad-hess}, we conclude
\begin{equation}
  \label{eqn:kkt_dts}
  \bar{\rho}_{a}\opt \propto s_{a} \E \left[\frac{1}{\sigma_{a}}
    \phi \left( \frac{\theta_{a}\opt-\mu_{a}}{\sigma_{a}} \right)
  \right]^{1/2}
\end{equation}


To show the equivalence with the partial derivatives of the Thompson sampling probabilities, observe
\begin{align*}
  \pi_{a}^{\text{TS}}(\mu,\sigma)
  & =\mathbb{P}(\theta_{a}>\theta_{a}^{*})
    =\mathbb{P}\left(Z_{a}>\frac{\theta_{a}^{*}-\mu_{a}}{\sigma_{a}}\right) \\
  & = \E\left[\mathbb{P}\left(
    Z_{a}>\frac{\theta_{a}^{*}-\mu_{a}}{\sigma_{a}}
    \mid \theta_{a}^{*} \right) \right] 
   = \E\left[1-\Phi\left(
    \frac{\theta_{a}^{*}-\mu_{a}}{\sigma_{a}}\right)\right],
\end{align*}
where the final equality is a result of the independence prior and posterior
distributions across the treatment arms. Interchanging the derivative
and expectation, 
\[ \frac{\partial}{\partial\mu_{a}}
  \pi_{a}^{\text{TS}}(\mu,\sigma)
  = \frac{\partial}{\partial\mu_{a}}
  \E\left[
    1-\Phi\left(\frac{\theta_{a}^{*}-\mu_{a}}{\sigma_{a}}\right)
  \right]
  = \E\left[\frac{1}{\sigma_{a}}
    \phi\left(\frac{\theta_{a}^{*}-\mu_{a}}{\sigma_{a}}\right)\right].
\]
\end{proof}

\subsection{More details on Density Thompson Sampling}
\label{section:proof-log-dts}

Similar to TS, DTS samples from the posterior distribution; while TS
asymptotically assigns all sampling effort to the best arm, DTS does not
oversample the best arm, as it is based on the gap between the second best
arm. Concretely, the relative sampling proportions
$\pi^{\text{DTS}}_{a}/\pi^{\text{DTS}}_{a'}$ can be expressed as ratios of the
`index'
$\E \left[\frac{1}{\sigma_{a}} \phi \left(
    \frac{\theta_{a}^{*}-\mu_{a}}{\sigma_{a}} \right) \right]^{1/2}$; the
higher the index, the more sampling effort allocated to that arm during the
epoch.  This index decreases at an exponential rate for \emph{all} arms as the
experiment progresses and the posterior uncertainty shrinks
$\sigma_{a} \to 0$.  
\begin{proposition}
  \label{prop:log-dts}
  Consider a fixed state $(\mu, \sigma)$ and suppose arms are sorted in decreasing order of their means $\mu_{1} > \mu_{2} > \ldots > \mu_{K}$.
  Then we have that:
  \[
    \begin{cases}
      -\frac{(\mu_{1}-\mu_{2})^{2}}{2(\sigma_{1}^{2}+\sigma_{2}^{2})}
      & ~~\mbox{top two arms}~~a\in{1,2} \\
      -\frac{(\mu_{1}-\mu_{a})^{2}}{2\sigma_{a}^{2}}
      & ~~\mbox{otherwise} \\
    \end{cases}
  \lesssim \log \E \left[\frac{1}{\sigma_{a}} \phi \left( \frac{\theta_{a}^{*}-\mu_{a}}{\sigma_{a}} \right) \right]
  \lesssim -\min_{a'\neq a}\frac{(\mu_{a'}-\mu_{a})^{2}}{2(\sigma_{a}^{2}+\sigma_{a'}^{2})}
  \]
  where $\lesssim$ contains additive terms that are logarithmic in $\min_{a} \sigma_{a}$.
\end{proposition}
If an arm is excessively under-sampled, its index will eventually be larger than others and
so sampling effort will be spread out across arms rather than concentrating on
the best one.  This makes DTS better suited for best-arm identification and
closer to Top-Two Thompson Sampling~\citep{Russo20}.

\begin{proof}
Suppose that the arms are sorted so that $\mu_{1}>...>\mu_{K}$. Since
$\theta_{a}^{*}$ is the max of independent random variables, $\mathbb{P}(\theta_{a}^{*}\leq x)=\prod_{a'\neq a}\mathbb{P}(\theta_{a'}\leq x)=\prod_{a'\neq a}\Phi\left(\frac{x-\mu_{a'}}{\sigma_{a'}}\right)$.
This means that the density of $\theta_{a}^{*}$ , denoted as $f_{a}^{*}(x)$,
can be expressed as (by the Leibniz product rule)
\begin{align*}
  f_{a}^{*}(x)
  & =\frac{d}{dx}\prod_{a'\neq a}\Phi\left(\frac{x-\mu_{a'}}{\sigma_{a'}}\right)\\
  & =\sum_{a'\neq a}
    \left[\frac{d}{dx}\Phi\left(\frac{x-\mu_{a'}}{\sigma_{a'}}\right)\right]
    \prod_{\hat{a}\notin\{a,a'\}}\Phi\left(\frac{x-\mu_{\hat{a}}}{\sigma_{\hat{a}}}\right)\\
  & =\sum_{a'\neq a}\left[\frac{1}{\sigma_{a'}}
    \phi\left(\frac{x-\mu_{a'}}{\sigma_{a'}}\right)\right]
    \prod_{\hat{a}\notin\{a,a'\}}\Phi\left(\frac{x-\mu_{\hat{a}}}{\sigma_{\hat{a}}}\right).
\end{align*}
Thus, we can write the expectation as
\begin{align*}
  \mathbb{E}\left[\frac{1}{\sigma_{a}}\phi\left(\frac{\theta_{a}^{*}-\mu_{a}}{\sigma_{a}}\right)\right]
  & =\int_{-\infty}^{\infty}\frac{1}{\sigma_{a}}
    \phi\left(\frac{x-\mu_{a}}{\sigma_{a}}\right)f_{a}^{*}(x)dx\\
  & =\int_{-\infty}^{\infty}\frac{1}{\sigma_{a}}
    \phi\left(\frac{x-\mu_{a}}{\sigma_{a}}\right)\sum_{a'\neq a}
    \left[\frac{1}{\sigma_{a'}}
    \phi\left(\frac{x-\mu_{a'}}{\sigma_{a'}}\right)\prod_{\hat{a}\notin\{a,a'\}}
    \Phi\left(\frac{x-\mu_{\hat{a}}}{\sigma_{\hat{a}}}\right)\right]dx\\
  & =\sum_{a'\neq a}\int_{-\infty}^{\infty}\frac{1}{\sigma_{a}}
    \phi\left(\frac{x-\mu_{a}}{\sigma_{a}}\right)
    \frac{1}{\sigma_{a'}}\phi\left(\frac{x-\mu_{a'}}{\sigma_{a'}}\right)
    \prod_{\hat{a}\notin\{a,a'\}}\Phi\left(\frac{x-\mu_{\hat{a}}}{\sigma_{\hat{a}}}\right)dx.
\end{align*}

For any arm $a$, we can obtain an upper bound for each term in the
sum by bounding $\Phi\left(\frac{x-\mu_{\hat{a}}}{\sigma_{\hat{a}}}\right)\leq1$
and directly integrating
\begin{align*}
  & \int_{-\infty}^{\infty}\frac{1}{\sigma_{a}}
    \phi\left(\frac{x-\mu_{a}}{\sigma_{a}}\right)
    \frac{1}{\sigma_{a'}}\phi\left(\frac{x-\mu_{a'}}{\sigma_{a'}}\right)
    \prod_{\hat{a}\notin\{a,a'\}}\Phi\left(\frac{x-\mu_{\hat{a}}}{\sigma_{\hat{a}}}\right)dx\\
  & \leq\int_{-\infty}^{\infty}\frac{1}{\sigma_{a}}
    \phi\left(\frac{x-\mu_{a}}{\sigma_{a}}\right)
    \frac{1}{\sigma_{a'}}\phi\left(\frac{x-\mu_{a'}}{\sigma_{a'}}\right)dx\\
  & =\frac{1}{\sqrt{2\pi(\sigma_{a'}^{2}+\sigma_{a}^{2})}}
    e^{-\frac{(\mu_{a'}-\mu_{a})^{2}}{2(\sigma_{a'}^{2}+\sigma_{a}^{2})}}.
\end{align*}
Finally, taking a max and summing up the terms gives
\[
  \mathbb{E}\left[\frac{1}{\sigma_{a}}
    \phi\left(\frac{\theta_{a}^{*}-\mu_{a}}{\sigma_{a}}\right)\right]
  \leq(K-1)\max_{a'\neq a}\frac{1}{\sqrt{2\pi(\sigma_{a'}^{2}+\sigma_{a}^{2})}}
  e^{-\frac{(\mu_{a'}-\mu_{a})^{2}}{2(\sigma_{a'}^{2}+\sigma_{a}^{2})}}
\]

To get the lower bound, it suffices to lower bound a single $a'$ term in the
sum, since all terms are positive. We select the term involving $a'=1$ ($a'=2$
if $a=1$). Since the integrand is positive, we obtain a further lower bound by
truncating the integration from $\mu_{1}$ to $\infty$
\begin{align*}
  \mathbb{E}\left[\frac{1}{\sigma_{a}}\phi\left(\frac{\theta_{a}^{*}-\mu_{a}}{\sigma_{a}}\right)\right]
  & \geq\int_{-\infty}^{\infty}
    \frac{1}{\sigma_{a}}\phi\left(\frac{x-\mu_{a}}{\sigma_{a}}\right)
    \frac{1}{\sigma_{1}}\phi\left(\frac{x-\mu_{1}}{\sigma_{1}}\right)
    \prod_{\hat{a}\notin\{a,1\}}\Phi\left(\frac{x-\mu_{\hat{a}}}{\sigma_{\hat{a}}}\right)dx \\
  & \geq\int_{\mu_{1}}^{\infty}\frac{1}{\sigma_{a}}
    \phi\left(\frac{x-\mu_{a}}{\sigma_{a}}\right)
    \frac{1}{\sigma_{1}}\phi\left(\frac{x-\mu_{1}}{\sigma_{1}}\right)
    \prod_{\hat{a}\notin\{a,1\}}\Phi\left(\frac{x-\mu_{\hat{a}}}{\sigma_{\hat{a}}}\right)dx.
\end{align*}
On the domain of integration, $x\geq\mu_{1}$, so
$\Phi\left(\frac{x-\mu_{\hat{a}}}{\sigma_{\hat{a}}}\right)\geq\frac{1}{2}$ for
all $\hat{a}\notin\{a,1\}$. Conclude
\begin{align*}
  & \int_{\mu_{1}}^{\infty}\frac{1}{\sigma_{a}}
    \phi\left(\frac{x-\mu_{a}}{\sigma_{a}}\right)
    \frac{1}{\sigma_{1}}\phi\left(\frac{x-\mu_{1}}{\sigma_{1}}\right)
    \prod_{\hat{a}\notin\{a,1\}}\Phi\left(\frac{x-\mu_{\hat{a}}}{\sigma_{\hat{a}}}\right)dx\\
  & \geq\left(\frac{1}{2}\right)^{K-2}
    \int_{\mu_{1}}^{\infty}\frac{1}{\sigma_{a}}
    \phi\left(\frac{x-\mu_{a}}{\sigma_{a}}\right)
    \frac{1}{\sigma_{1}}\phi\left(\frac{x-\mu_{1}}{\sigma_{1}}\right)dx.
\end{align*}
Evaluating the integral explicitly gives
\begin{align*}
   \int_{\mu_{1}}^{\infty}\frac{1}{\sigma_{a}}
    \phi\left(\frac{x-\mu_{a}}{\sigma_{a}}\right)
    \frac{1}{\sigma_{1}}\phi\left(\frac{x-\mu_{1}}{\sigma_{1}}\right)dx
  = \frac{e^{-\frac{(\mu_{a}-\mu_{1})^{2}}{2(\sigma_{a}^{2}+\sigma_{1}^{2})}}}{\sqrt{2\pi (\sigma_{a}^{2}+ \sigma_{1}^{2})}}
    \bar{\Phi} \left( \frac{\sigma_{1}}{\sigma_{a}}\frac{\mu_{1} - \mu_{a}}{\sqrt{\sigma_{a}^{2}+ \sigma_{1}^{2}}} \right)
\end{align*}
Using the Gaussian tail bound $\bar{\Phi}(x)\geq\frac{\phi(x)}{2x}$ for
$x > 0$, 
\begin{align*}
  \mathbb{E}\left[\frac{1}{\sigma_{a}}
  \phi\left(\frac{\theta_{a}^{*}-\mu_{a}}{\sigma_{a}}\right)\right]
  & \geq\left(\frac{1}{2}\right)^{K-2}
    \int_{\mu_{1}}^{\infty}\frac{1}{\sigma_{a}}
    \phi\left(\frac{x-\mu_{a}}{\sigma_{a}}\right)
    \frac{1}{\sigma_{1}}\phi\left(\frac{x-\mu_{1}}{\sigma_{1}}\right)dx\\
  & \geq\frac{1}{\sqrt{2\pi(\sigma_{a}^{2}+\sigma_{1}^{2})}}
    e^{-\frac{(\mu_{a}-\mu_{1})^{2}}{2(\sigma_{a}^{2}+\sigma_{1}^{2})}}
    \frac{1}{2\left(\frac{\sigma_{1}}{\sigma_{a}}\frac{(\mu_{1}-\mu_{a})}{\sqrt{\sigma_{1}^{2}+\sigma_{a}^{2}}}\right)}
    e^{-\frac{(\mu_{1}-\mu_{a})^{2}}{2\sigma_{a}^{4}(\sigma_{1}^{-2}+\sigma_{a}^{-2})}}\\
  & =\frac{\sigma_{a}}{2\sigma_{1}\sqrt{2\pi}(\mu_{1}-\mu_{a})}
    e^{-\frac{1}{2}\frac{(\mu_{a}-\mu_{1})^{2}}{\sigma_{a}^{2}}}.
\end{align*}

We can get tighter lower bounds for the top two arms $a\in\{1,2\}$.  As in the
above lower bound, we focus on lower bounding the term corresponding with
$a'=1$ if $a=2$ or $a'=2$ if $a=1$ (for both we obtain the same lower bound).
But instead of integrating from $\mu_{1}$, we integrate from $\mu_{2}$.  For
illustration, we focus on $a=2$
\begin{align*}
  \mathbb{E}\left[\frac{1}{\sigma_{2}}
  \phi\left(\frac{\theta_{2}^{*}-\mu_{2}}{\sigma_{2}}\right)\right]
  & \geq\int_{-\infty}^{\infty}\frac{1}{\sigma_{2}}
    \phi\left(\frac{x-\mu_{2}}{\sigma_{2}}\right)
    \frac{1}{\sigma_{1}}\phi\left(\frac{x-\mu_{1}}{\sigma_{1}}\right)
    \prod_{\hat{a}\notin\{a,1\}}
    \Phi\left(\frac{x-\mu_{\hat{a}}}{\sigma_{\hat{a}}}\right)dx\\
  & \geq\int_{\mu_{2}}^{\infty}\frac{1}{\sigma_{2}}
    \phi\left(\frac{x-\mu_{2}}{\sigma_{2}}\right)
    \frac{1}{\sigma_{1}}\phi\left(\frac{x-\mu_{1}}{\sigma_{1}}\right)
    \prod_{\hat{a}\notin\{a,1\}}
    \Phi\left(\frac{x-\mu_{\hat{a}}}{\sigma_{\hat{a}}}\right)dx.
\end{align*}
On this range, we have $x\geq\mu_{2}$. This implies that for every
$\hat{a}\neq\{1,2\}$, $x-\mu_{\hat{a}}\geq\mu_{2}-\mu_{\hat{a}}\geq0$. In
turn, this means
$\Phi\left(\frac{x-\mu_{\hat{a}}}{\sigma_{\hat{a}}}\right)\geq\frac{1}{2}$ as
we had before
\[
  \mathbb{E}\left[\frac{1}{\sigma_{a}}
    \phi\left(\frac{\theta_{a}^{*}-\mu_{a}}{\sigma_{a}}\right)\right]
  \geq\left(\frac{1}{2}\right)^{K-2}
  \int_{\mu_{2}}^{\infty}\frac{1}{\sigma_{2}}
  \phi\left(\frac{x-\mu_{2}}{\sigma_{2}}\right)
  \frac{1}{\sigma_{1}}\phi\left(\frac{x-\mu_{1}}{\sigma_{1}}\right)dx
\]
For $a=2$, we can evaluate this integral explicitly
\begin{align*}
  & \int_{\mu_{2}}^{\infty}\frac{1}{\sigma_{2}}
    \phi\left(\frac{x-\mu_{2}}{\sigma_{2}}\right)\frac{1}{\sigma_{1}}
    \phi\left(\frac{x-\mu_{1}}{\sigma_{1}}\right)dx\\
  & =\frac{1}{2\sqrt{2\pi(\sigma_{1}^{2}+\sigma_{2}^{2})}}
    e^{-\frac{(\mu_{1}-\mu_{2})^{2}}{2(\sigma_{1}^{2}+\sigma_{2}^{2})}}
    \left(1+\text{erf}\left(\frac{(\mu_{1}-\mu_{2})\sigma_{2}}
    {\sqrt{2}\sigma_{1}\sqrt{\sigma_{1}^{2}+\sigma_{2}^{2}}}
    \right)\right).
\end{align*}
Note that for positive values of the error function it is bounded below by zero. This gives 
\[
  \mathbb{E}\left[\frac{1}{\sigma_{a}}
    \phi\left(\frac{\theta_{a}^{*}-\mu_{a}}{\sigma_{a}}\right)\right]
  \geq\left(\frac{1}{2}\right)^{K-2}
  \frac{1}{2\sqrt{2\pi(\sigma_{1}^{2}+\sigma_{2}^{2})}}
  e^{-\frac{(\mu_{1}-\mu_{2})^{2}}{2(\sigma_{1}^{2}+\sigma_{2}^{2})}}.
\]
Repeating these steps for $a=1$ (taking $a'=2$) results in the same lower
bound.
\end{proof}

\section{Implementation Details}
\label{section:implementation}

We first discuss the implementation details for solving the $\algofull$ planning problem \eqref{eqn:rho}.
Recall that the problem involves maximizing the expectation of future posterior means
over constant sampling allocations in the simplex $\Delta_{\numarm}$:
\[
\maximize_{\bar{\rho} \in \Delta_\numarm}
\left\{ V^{\bar{\rho}}_{t}(\mu_{t},\sigma_{t})
  =  \E_{t} \left[ \max_{a} \left\{ \mu_{t,a}
      + \sqrt{\frac{\sigma_{t, a}^4 \bar{\rho}_{a} \bar{b}_{t}}
        {s_a^2 + \sigma_{t, a}^2 \bar{\rho}_{a} \bar{b}_{t}}} Z_{t, a}
    \right\} \right] \right\}.
\]
We approximate the expectation in the objective function by a sample average approximation
with $N$ standard normal random vectors.
\[
    \maximize_{\bar{\rho} \in \Delta_\numarm}
    \left\{
    \frac{1}{N}\sum_{j=1}^{N}
   \max_{a} \left\{ \mu_{t,a}
      + \sqrt{\frac{\sigma_{t, a}^4 \bar{\rho}_{a} \bar{b}_{t}}
        {s_a^2 + \sigma_{t, a}^2 \bar{\rho}_{a} \bar{b}_{t}}} Z_{t, a, j}
    \right\}
    \right\}.
\]
where $Z_{t,a,1},...,Z_{t,a,N}$ are iid draws of $N(0,1)$ random variables. 
We use quasi-Monte Carlo methods for variance reduction and draw
the normal random variables from a Sobol sequence,
which is widely used in practice in Bayesian Optimizaton \cite{BalandatEtAl20}.

There are many methods for solving this constrained optimization problem (e.g. projected gradient descent).
We use a softmax parameterization of the 
simplex $\bar{\rho}_{a} \propto e^{v_{a}}$,
and use unconstrained stochastic gradient methods to optimize over $v$.
We observe that vanilla stochastic gradient descent gets stuck at sub-optimal allocations
that allocate all sampling effort to one treatment arm. We obtain much better performance
from approximate second-order methods such as Adam \cite{KingmaBa15} or L-BFGS \cite{LiuNo89},
and use Adam for the experimental evaluation.

For the policy gradient method, we parameterize the policy as a feed-forward neural network with 
2 hidden layers with 512 units in each layer. The network uses the rectified non-linearity \cite{GlorotBoBe11} for all hidden layers.
We pass the posterior means $\mu \in \R^{\numarm}$, the posterior variances $\sigma^{2} \in \R^{\numarm}$, the current epoch $t \in \N$,
and the measurement variance $s^{2} \in \R^{\numarm}$ as inputs. The output of the network is passed through a softmax layer
and so the final output is a sampling allocation $\bar{\rho}\in \Delta^{\numarm}$. We train the network with the 
Adam optimizer with learning rate $5.0 \times 10^{-6}$ and $(\beta_{1},\beta_{2}) = (0.9, 0.999)$ in minibatches of size $50$.
Minibatches are drawn by randomly generating priors $(\mu_{0}, \sigma_{0})$.



\end{document}
